%% file: 2026_NeurIPS_flowce_main.tex
\documentclass{article}

\usepackage[main,final]{neurips_2026}

\usepackage[utf8]{inputenc} 
\usepackage[T1]{fontenc}    
\usepackage{hyperref}       
\usepackage{url}            
\usepackage{booktabs}       
\usepackage{amsfonts}       
\usepackage{nicefrac}       
\usepackage{microtype}      
\usepackage[table]{xcolor}  
\usepackage{enumitem}

\usepackage{algorithm}
\usepackage{algorithmic}
\usepackage{multirow}
\usepackage{array}
\usepackage{caption} 
\usepackage{wrapfig}

  \usepackage{amsmath}

  \usepackage{graphicx}
  \usepackage{bbm}
  \usepackage{booktabs}        
  \usepackage[table]{xcolor}
  \usepackage{colortbl}

  \usepackage{booktabs}
  \usepackage[table]{xcolor}
  \usepackage{siunitx}
\definecolor{newcolor}{RGB}{255, 102, 0}

\definecolor{urlred}{rgb}{0.9, 0.29, 0.0}

\hypersetup{
  colorlinks=true,
  linkcolor=black,
  citecolor=black,
  urlcolor=urlred
}

\usepackage{wrapfig}
  \usepackage{float}

  \usepackage{pifont}
  \newcommand{\cmark}{\ding{51}} 
  \newcommand{\xmark}{\ding{55}} 

  \DeclareMathOperator*{\argmin}{argmin}

  \newcommand{\bfX}{\mathbf{X}}
  
  \newcommand{\bfmu}{\boldsymbol{\mu}}
  \newcommand{\bfSigma}{\boldsymbol{\Sigma}}
  \newcommand{\bfL}{\mathbf{L}}

  \newcommand{\indic}{{\mathbbm{1}}}
  \newcommand{\Var}{{\rm Var}\,}
  \newcommand{\Exp}{{\rm I\hspace{-0.8mm}E}\,}
  \newcommand{\Prob}{{\mathbb{P}}\,}
  \newcommand{\calN}{{\mathcal{N}}\,}

  \newcommand{\x}{{\mathbf{x}}}
  \newcommand{\bfx}{{\mathbf{x}}}
  \newcommand{\bfw}{{\mathbf{w}}}
  \newcommand{\X}{{\mathbf{X}}}

  \newcommand{\bfS}{\mathbf{S}}

  \usepackage{amsthm}
  \newcommand{\bftheta}{\boldsymbol{\theta}}

  \usepackage{amsfonts}

      \newtheorem{theorem}{Theorem}
      
      \newtheorem{proposition}{Proposition}

\title{Robust Importance Sampling for Rare Events    via Constrained Gaussian Mixtures}

\author{%
\textbf{Paweł Lorek}\textnormal{\textsuperscript{1,2}}\thanks{Corresponding author: \texttt{pawel.lorek@math.uni.wroc.pl}}\quad
\textbf{Rafał Nowak}\textnormal{\textsuperscript{1,2}}\quad
\textbf{Rafał Topolnicki}\textnormal{\textsuperscript{1,2,3}}\\
\textbf{Tomasz Trzciński}\textnormal{\textsuperscript{2,4,5}}\quad
\textbf{Maciej Zięba}\textnormal{\textsuperscript{2,6}}\\[0.35em]
{\normalfont\small \textsuperscript{1}University of Wrocław\quad
\textsuperscript{2}Tooploox\quad
\textsuperscript{3}TRAILS University of Warsaw}\\
{\normalfont\small \textsuperscript{4}Warsaw University of Technology\quad
\textsuperscript{5}IDEAS Research Institute}\\
{\normalfont\small \textsuperscript{6}Wrocław University of Science and Technology}
}

\begin{document}

\maketitle

\begin{abstract}
We study estimating rare-event probabilities $I = \mathbb{P}(g(\mathbf{X}) > \gamma)$ with $\mathbf{X} \sim \mathcal{N}(\bfmu, \bfSigma)$ and general $g : \mathbb{R}^d \to \mathbb{R}$.
We address this problem through importance sampling,
and propose a framework that substantially improves efficiency and robustness over baselines such as crude Monte Carlo, adaptive cross-entropy, variational-inference-based methods (including reverse- and forward-KL approaches), as well as Safe-ICE, Subset Simulation, and Sequential Monte Carlo,
\textit{drawing}
on ideas from both rare-event estimation and cross-entropy optimization.
 The key contribution has two parts: first, we separate the problem into \emph{coverage}, to overcome the cold-start barrier, and \emph{fitting}, to refine proposals once a meaningful signal is available;  second,  we constrain the final GMM proposal so that it has finite importance-sampling variance (since coverage alone is not sufficient -- without safeguards, importance sampling may still suffer from infinite variance).
Together, these ingredients yield expressive proposals; finite variance does
not by itself guarantee practical stability at a fixed sampling budget. Extensive experiments demonstrate substantial variance reduction, strong robustness across diverse benchmarks, and favorable cost--efficiency trade-offs, with the proposed approach often outperforming these baselines, particularly in high-dimensional and multimodal settings where competing methods frequently become unstable or fail.
 Our code is available at \url{https://github.com/lorek/robust-cfi-is}.
\end{abstract}

 \section{Introduction}
  We consider events
  $\{g(\bfX)>\gamma\}$ with $\bfX\!\sim\!\calN(\bfmu,\bfSigma)$ and possibly nonconvex $g:\mathbb{R}^d\!\to\!\mathbb{R}$; let $f$ denote the density of $\bfX$.
  Our target is:
  \begin{equation}\label{eq:general_problem}
  I \;=\; \Prob\!\big(g(\bfX)>\gamma\big),
  \end{equation}
  in regimes where $I$ is tiny (e.g., $10^{-6}$–$10^{-3}$), rendering crude Monte Carlo impractical.

  A long line of work has sought to overcome this difficulty with importance sampling, cross-entropy (CE) methods, and variational inference (VI).
  While in some cases these approaches can reduce variance dramatically, in others they fail completely, since they face two recurring obstacles in rare-event regimes.
  First, \emph{cold start}: when the event probability is extremely small, the learning signal obtained from naive sampling is nearly absent, leading to unstable or stalled optimization.
  Second, \emph{lack of variance control}: minimizing KL divergence alone does not prevent proposals that yield infinite-variance estimators, even when optimization appears successful.
  These issues are well documented in reliability engineering, rare-event simulation, and, more recently, in machine learning.
  For example  Owen \cite{owen_var_is} noted: ``\textsl{Importance sampling can bring enormous gains, making an otherwise infeasible problem amenable to Monte Carlo. It can also backfire, yielding an estimate
with infinite variance when simple Monte Carlo would have had a finite variance.
It is the hardest variance reduction method to use well.}''

  We address these challenges by introducing a pipeline that cleanly separates coverage from fitting and builds robustness into the proposal family. The key idea is to first construct a proposal that is designed to provide  overlap with the rare set, ensuring that subsequent training receives a meaningful signal.
  We then refine this proposal within a constrained Gaussian mixture family whose covariance structure precludes the infinite-variance pathology often encountered in importance sampling.
  Finally, we optionally add a normalizing-flow layer to capture additional complexity. The finite-variance guarantee applies when the final proposal remains a constrained GMM and does not automatically extend to the flow-refined variant (although our experiments show that flow refinement often improves variance reduction while remaining stable across the tested benchmarks).

  In summary, our contributions are as follows: (i) We propose a two-stage, coverage-first framework for rare-event estimation that overcomes cold-start issues.  (ii) We extend a sufficient finite-variance condition from Gaussian to Gaussian mixture proposals and make it enforceable during proposal learning. (iii) We demonstrate strong variance reduction and robustness through extensive experiments across diverse benchmarks.

\section{Background}\label{sec:background}
We consider the problem of estimating rare probabilities, given by eq. \eqref{eq:general_problem}. To tackle the problem, we utilize the importance sampling (IS) technique:
\begin{equation}
  I=\Exp_{\X\sim h}[k_1(\X)],
  \quad
  k_1(\x)=\indic(g(\x)>\gamma)\,\frac{f(\x)}{h(\x)},
\end{equation}
where $h(\x)$ is the proposal distribution. Please note that zero-variance choice is $h^{\mathrm{opt}}(\x)=\indic(g(\x)>\gamma)f(\x)/I$, but to sample from $h^{\mathrm{opt}}$, knowledge\footnote{MCMC may (in some cases) sample from unnormalized targets, but breaks i.i.d.\ draws and standard unbiased IS.} of $I$ itself is required. We therefore restrict to a parametric family $h\equiv p_{\bftheta}$ and fit $p_{\bftheta}$ to approximate $h^{\mathrm{opt}}$.

 Baseline approaches tackle the problem by minimizing a KL divergence, either $\mathrm{KL}(p_{\bftheta}\|h^{\text{opt}})$
(the so-called Reverse KL (\textbf{RKL})), or $\mathrm{KL} (h^{\text{opt}}\|p_{\bftheta})$  (the so-called  forward KL (\textbf{FKL})).
More explicitly, minimizing the FKL is equivalent to maximizing
\begin{equation}
\label{eq:fkl_background}
\max_{\bftheta}\;
\Exp_{\X\sim f}
\!\left[
\indic(g(\X)>\gamma)\,\log p_{\bftheta}(\X)
\right],
\end{equation}
which corresponds to the classical cross-entropy (CE) objective used in rare-event
simulation.
In contrast, directly minimizing the reverse KL divergence
$\mathrm{KL}(p_{\bftheta}\|h^{\rm opt})$
is ill-posed due to indicator -- which in practice we replace by a  smooth gate
$s_\tau(g(\x),\gamma)=\sigma((g(\x)-\gamma)/\tau)\in(0,1),$ where
$\sigma(z)=({1+e^{-z}})^{-1}$.
This leads to the surrogate reverse-KL objective
\begin{equation}
\label{eq:rkl_background}
\max_{\bftheta}\;
\Exp_{\X\sim p_{\bftheta}}
\!\left[
\log s_\tau\!\bigl(g(\X),\gamma\bigr)-
\log\frac{p_{\bftheta}(\X)}{f(\X)}
\right],
\end{equation}
which is commonly used in variational formulations of rare-event estimation
(see Appendix~\ref{app:baselines})
\color{black}

Although these techniques are widely used to find surrogates for  $h^{\rm opt}$  , they are prone to two major issues while estimating rare probability regions. First, they are sensitive to cold start issues, where a significant number of samples is required to satisfy $(g(\x)>\gamma)$.
Second,
even when the reverse or forward KL problem is solved exactly,
 the resulting  $p_{\bftheta}$ may yield
  an IS estimator with high -- even
\emph{infinite} -- variance.
 Consequently,
 KL minimization alone (reverse/forward) does not ensure a safe proposal. In order to eliminate those limitations, we propose a novel, two-stage (optionally refined further by a normalizing flow model)
 technique that learns a proposal $p_{\bftheta}$ robust to these issues.

\section{Related Work}
Estimating rare-event probabilities has a long history across simulation,
reliability engineering, and machine learning. Au and Beck introduced Subset
Simulation, which uses a sequence of intermediate events to estimate a rare
probability progressively~\cite{au2001estimation}. Sequential importance
sampling methods adapt proposals over stages to concentrate samples in the
rare-event region~\cite{papaioannou2016sequential}; related particle methods
include adaptive multilevel splitting and SMC-based rare-event
estimators~\cite{cerou2012smc}. Rubinstein formalized the cross-entropy (CE)
method~\cite{rubinstein1999cross}, which Rubinstein and Kroese developed
further~\cite{rubinstein2004cross}. Subsequent work introduced adaptive CE updates and mixture proposals
for rare-event simulation~\cite{tito2002rare,deboer2005tutorial,wang2016cross}.

Safe and defensive importance sampling control unstable likelihood ratios by
mixing an adaptive proposal with a fixed component. Owen and Zhou
developed safe mixture sampling~\cite{owen2000safe}, while Hesterberg's
defensive mixture includes the nominal density as a fixed component and thus
bounds the corresponding likelihood ratio~\cite{hesterberg1995weighted}.
Delyon and Portier studied adaptive mixtures with a heavy-tailed safe component
and established asymptotic guarantees~\cite{delyon2021safe}. Our safeguard
uses a different mechanism: we constrain every covariance matrix in a learned
Gaussian mixture so that all components can move toward the rare-event region
while retaining sufficiently heavy Gaussian tails. The resulting guarantee is
a finite-variance statement for the final constrained GMM proposal.
The half-variance threshold has classical Gaussian precedents: Owen's
Examples~9.1 and~9.3 derive it for a scalar Gaussian moment problem and a
$d$-dimensional standard Gaussian with isotropic proposal covariance $\sigma^2\mathbf I$,
respectively~\cite{owen_var_is}. The corresponding full-covariance condition
for a single Gaussian proposal   was established
in~\cite{lorek2025lectures}. Here, we extend this condition to
Gaussian mixture proposals and incorporate it as an enforceable covariance
safeguard during proposal learning. Safe-ICE
instead modifies CE through a problem-dependent safe construction
~\cite{gao2025safeice}; we include it in the empirical comparison.

Adaptive mixture importance sampling also updates the weights and parameters
of a proposal mixture online. Douc et al. established convergence results for
such schemes and studied Rao--Blackwellized updates
~\cite{douc2007convergence}. These methods address the validity and efficiency
of repeated adaptation. Our covariance constraint addresses a complementary
failure mode: a learned Gaussian proposal can fit its target while producing
an infinite-variance importance-sampling estimator.

Variational inference (VI) offers a flexible approach to approximate inference.
Classical black-box VI uses REINFORCE-style gradient
estimators~\cite{williams1992simple,ranganath2014black}, which are unbiased but
often have high variance. Work on variance reduction combines VI with
overdispersed proposals~\cite{ruiz2016overdispersed}, adaptive schemes
~\cite{li2018variance}, and importance-weighted objectives
~\cite{agrawal2020advances}. Jerfel et al.~\cite{jerfel2021variational} and Su
and Chen~\cite{su2021variational} studied forward-KL proposal refinement for
importance sampling, linking variational objectives to sampling efficiency.
We focus on KL objectives because they yield direct Monte Carlo estimators for
the rare-event target; Appendix~\ref{app:divergence_measures} explains why
Wasserstein and KSD objectives are less convenient here.

Generative models provide more flexible proposal families. Rezende and
Mohamed introduced normalizing flows~\cite{rezende2015variational}, followed by
tractable architectures including NICE~\cite{DinhKB14}, RealNVP
~\cite{DinhSB17}, and MAF~\cite{PapamakariosMP17}. Gao et
al.~\cite{gao2023rare} and Gibson et al.~\cite{gibson2023flow} applied flows to
rare-event simulation. Large-deviation methods offer a different route: Tong
and Stadler combine large-deviation theory with adaptive importance sampling,
using it to initialize a Gaussian biasing distribution and identify an
informative low-dimensional subspace in high-dimensional problems
~\cite{tong2023large}. Our method instead uses progressive CE-style discovery, without requiring
large-deviation rate-function optimization or a problem-specific
dominating-point calculation, and then fits an explicit constrained mixture
with optional flow refinement. The next section describes this procedure. Appendix
~\ref{app:alternative_sampling_methods} discusses further large-deviation,
neural SMC, and adjoint-based samplers.

\section{Our Approach}
We first construct a coverage-first proposal, then refine it by
constrained KL fitting, optionally followed by normalizing-flow refinement.
The finite-variance guarantee applies when the final proposal is a
constrained GMM satisfying Theorem~\ref{thm:gmm_finite_var}.

\subsection{Coverage-first initializer (CFI)}
We first construct a proposal distribution $p_{\rm CFI}$ that allocates a
prescribed fraction $1-\rho$ of its mass inside or near
$\{g(\x)>\gamma\}$. This provides overlap with the rare-event region and a
nontrivial training signal for subsequent RKL and FKL refinement, even in
ultra-rare regimes.
Like adaptive CE
~\cite{rubinstein2004cross,deboer2005tutorial,wang2016cross}, CFI draws samples from the current
proposal and progresses through a sequence of intermediate levels. However,
CFI performs an unweighted fit to the elite samples and does not evaluate the
likelihood-ratio weights $f/p_{\bftheta_{t-1}}$ used by full adaptive CE.
\begin{wrapfigure}{r}{0.52\textwidth}
\vspace{-0.1em}
\centering
\begin{minipage}{0.50\textwidth}
\captionof{algorithm}{CFI: Coverage-First Initialization}
\label{alg:PACE}
\begin{algorithmic}[1]
\REQUIRE Family $\{p_{\bftheta}\}$ with $\bftheta_0$, function $g$, level $\gamma$, quantile $\rho$, batch $M$, smoothing $\alpha\in[0,1)$
\STATE $t=0$
\REPEAT
  \STATE $t\gets t+1$; \ sample $\bfX_{1:M} \sim p_{\bftheta_{t-1}}$
  \STATE Set $\eta_i=g(\bfX_i)$, sort: $\eta_{(1)} \leq \ldots\leq \eta_{(M)}$
  \STATE $(1-\rho)$-quantile: $\gamma_t \gets \eta_{(\lceil(1-\rho)M\rceil)}$
  \STATE \textbf{CE update:}
  \[
  {\small
  \bftheta'_t
  =
  \arg\max_{\bftheta}
  \frac{1}{M}
  \sum_{i=1}^M
  \indic\!\big(g(\bfX_i)\ge \gamma_t\big)
  \log p_{\bftheta}(\bfX_i)
  }
  \]
  \STATE $\bftheta_t \gets \alpha\,\bftheta'_t + (1-\alpha)\,\bftheta_{t-1}$
  \label{alg:cfi_smoothing}
\UNTIL{$\gamma_t \ge \gamma$}
\STATE \textbf{Return} $p_{\rm CFI}\equiv p_{\bftheta_t}$
\end{algorithmic}
\end{minipage}
\vspace{-1em}
\end{wrapfigure}CFI
is therefore used as a coverage-first initialization rather than as the final
importance-sampling adaptation procedure. The full procedure is given in
Alg.~\ref{alg:PACE}.

At each iteration, CFI samples from the current proposal, updates the
intermediate level using the empirical $(1-\rho)$-quantile, and refits
to elite samples with parameter smoothing.

\paragraph{CFI vs gradient-based mode seeking.}
A natural question is whether the coverage stage of CFI can be replaced by gradient-based optimization of $g(\x)$, e.g., via multiple gradient ascent restarts. In Appendix~\ref{sec:grad_appendix} we show that this approach is unreliable, as it can systematically miss relevant modes or diverge, failing to provide adequate coverage even in simple settings.

\subsection{Distribution proposals}

A crucial aspect of CFI is the selection of the proposal distribution family. In our approach, we propose using Gaussian Mixture Models (GMMs). Since our method relies on gradient-based optimization, we first introduce an adjusted parameterization of the model, which we refer to as the unconstrained GMM.
However,  the unconstrained GMM can lead to infinite variance in importance sampling (IS),
see e.g., a toy example in Section \ref{sec:motivating_example}. To address this issue, we propose a parameterization of a constrained GMM variant.

\paragraph{Unconstrained GMM.}
  The standard proposal is  a  Gaussian mixture
 $
  p_{\bftheta}(\x)=\sum_{k=1}^K \pi_k\,\calN(\x\mid \bfmu_k,\bfSigma_k),
  $
  with parameters $\bftheta=\{\bfmu_k,\bfL_k,\ell_k\}_{k=1}^K$ optimized by gradient methods.
  For the sake of gradient updates, weights and covariances are parameterized as:

   \begin{equation}
  \pi_k=\frac{\exp(\ell_k)}{\sum_{j=1}^K \exp(\ell_j)},
  \qquad \bfSigma_k=\bfL_k\bfL_k^\top.
  \end{equation}
 
Unconstrained Gaussian proposals can lead to unstable or even invalid
importance-sampling estimators. The one-dimensional motivating example in
Section~\ref{sec:motivating_example} shows that, when both the mean and variance
are fitted by KL minimization, the optimizer may choose a variance below the
critical threshold required for finite IS variance. This is not merely a
one-dimensional artifact, but reflects a more general covariance-collapse
mechanism.
\begin{proposition}[KL-induced covariance collapse]
\label{propp:kl_collapse}
Let $\X \sim \mathcal{N}(\bfmu_0,\bfSigma_0)$ and consider
the $\mathrm{FKL}$ objective~(\ref{eq:fkl_background})
with an unconstrained Gaussian proposal
$p_{\bftheta}=\mathcal{N}(\bfmu,\bfSigma)$.
Then the optimizer satisfies
$\bfmu_{\rm opt}
=
\Exp[\X \mid g(\X)>\gamma]$
and
$\bfSigma_{\rm opt}
=
\operatorname{Cov}(\X \mid g(\X)>\gamma)$.
Moreover, if
$\lambda_{\min}(
\operatorname{Cov}(\X \mid g(\X)>\gamma))
\to 0$
as $\gamma\to\infty$,
then
$\lambda_{\min}(\bfSigma_{\rm opt})\to0$.
\end{proposition}

The proof is provided in Appendix~\ref{app_sec:proofs}.
Proposition~\ref{propp:kl_collapse} shows that forward-KL fitting may inherit
the concentration of the rare-event conditional distribution, causing the
proposal covariance to collapse. This is problematic for importance sampling,
since the second moment contains the term
$\int_{\{g(\x)>\gamma\}} f(\x)^2/p(\x)\,d\x$,
which may diverge if the proposal decays too quickly in relevant tail
directions. Thus, a well-optimized KL objective does not necessarily yield a
safe proposal.

Fig.~\ref{fig:g0} illustrates the same variance pathology for
$\Prob(|X|>4)$ with $X\sim\mathcal{N}(0,1)$; see also
Section~\ref{sec:motivating_example}. This motivates imposing an explicit
covariance floor. Theorem~\ref{thm:gmm_finite_var} shows that requiring
$\bfSigma_k \succ \tfrac12\,\bfSigma$ for every mixture component guarantees
finite variance of the resulting importance-sampling estimator.

\begin{theorem}\label{thm:gmm_finite_var}[Finite Variance for Constrained GMMs]
Let $f(\bfx)=\calN(\bfx\mid\bfmu,\bfSigma)$
be the nominal density, and let $p(\x)=\sum_k \pi_k\,\calN(\x\mid \bfmu_k,\bfSigma_k)$
be a GMM proposal. For i.i.d.\ $\bfX_i\sim p$ and the IS estimator
\[
\hat{I}^{\rm IS}_R = \frac{1}{R}\sum_{i=1}^R \indic(g(\bfX_i)>\gamma)\,\frac{f(\bfX_i)}{p(\bfX_i)},
\]
if $\bfSigma_k \succ \tfrac{1}{2}\,\bfSigma$ for all $k$, then $\Var(\hat{I}^{\rm IS}_R)$ is finite.
\end{theorem}
The proof is provided in Appendix~\ref{app_sec:proofs}.
The theorem guarantees only finiteness of the variance; it does not imply that
the variance is practically small, or guarantee good contribution ESS,
representative finite-budget estimates, or accurate and stable behavior at a
fixed budget $R$. Accordingly, the suffix -\textbf{c} denotes the covariance
constraint, not a promise of good finite-budget performance.
The factor $1/2$ matches the Gaussian thresholds in Owen's
Examples~9.1 and~9.3~\cite{owen_var_is}. These examples treat a scalar moment
problem and a $d$-dimensional standard Gaussian with an isotropic proposal,
respectively. The full-covariance version for a
single Gaussian proposal was established
in~\cite{lorek2025lectures}.
Theorem~\ref{thm:gmm_finite_var} extends this condition to full-covariance Gaussian mixture proposals.

The componentwise condition in Theorem~\ref{thm:gmm_finite_var} is sufficient
but not necessary for a general GMM. In particular, one component with positive
mixture weight and sufficiently broad covariance can control the tails of the
full mixture even when other components have smaller covariances. Necessity
also depends on the geometry of the rare-event set, specifically on whether it
is unbounded in directions in which the proposal is too light-tailed. We impose
the stronger componentwise condition because it is independent of the geometry
of $g$, directly checkable, and enforceable during optimization.

As a consequence of this theorem, we propose a constrained variant of GMM.

\paragraph{Constrained GMM.}
The constrained proposal has the same mixture form, with covariance
parameterization
$\displaystyle \bfSigma_k=\bfL_k\bfL_k^\top+\beta\,\bfSigma.$
Theorem~\ref{thm:gmm_finite_var} applies to the actual final proposal whenever
$\bfSigma_k\succ\tfrac12\,\bfSigma$ for every component. For
$\beta>\tfrac12$, this strict inequality follows from the parameterization;
for $\beta=\tfrac12$, it additionally requires
$\bfL_k\bfL_k^\top\succ0$ for every component. We use $\beta=0.5$ except
during the CFI-c annealing phase described below.
At $\beta=0.5$, the strict condition must be verified for the final
proposal; a numerically nonzero $\bfL_k\bfL_k^\top$ is insufficient.

Throughout, suffixes -\textbf{c} and -\textbf{u} denote \textbf{c}onstrained and \textbf{u}nconstrained models, respectively.

  \paragraph{Adaptive $\beta$ annealing in CFI-c.}
  During the CE iterations, the proposal must remain sufficiently flexible to localize rare-event regions.
  A fixed $\beta=0.5$ can be overly conservative at this stage, leading to premature saturation of the level sequence.
  Therefore, in CFI-c only, we adapt $\beta$ during training:
  when progress of the level sequence stagnates, $\beta$ is temporarily reduced (from initial 0.5) to relax the covariance floor;
  when progress resumes, $\beta$ is gradually increased.
  This adaptive annealing improves coverage and prevents early plateaus.
    Once the target level $\gamma$ is reached, we set $\beta=0.5$ and terminate the
  CE iterations. A final refitting step at the target level is performed with
  $\beta=0.5$; the finite-variance guarantee applies when the resulting final
  proposal satisfies $\bfSigma_k\succ\tfrac12\,\bfSigma$ for all $k$.

  \paragraph{Parameter smoothing.}
In line~\ref{alg:cfi_smoothing} of Alg.~\ref{alg:PACE}, we exponentially smooth the means,
mixture weights, and covariance matrices using $\alpha$.
For constrained GMMs, we smooth only the free covariance component
$\bfS_k=\bfL_k\bfL_k^\top$, preserving the floor $\beta\bfSigma$.
Details are given in Appendix~\ref{app:hyperparameters}.

\subsection{Constrained-KL fitting}

Starting from the coverage-first proposal  -- denote it by
$p_{\rm CFI}$ -- we refine the proposal within the constrained Gaussian mixture
family, either by forward--KL (FKL-c) or reverse--KL (RKL-c).
Both variants enforce the covariance floor $\bfSigma_k\succ\tfrac12\,\bfSigma$,
which guarantees finite variance of the resulting importance sampling estimator.

\paragraph{FKL-c (constrained forward KL).}
As shown in Section~\ref{sec:background}, minimizing
$\mathrm{KL}(h^{\rm opt}\|p_{\bftheta})$ is equivalent to maximizing the objective
in~\eqref{eq:fkl_background}.
Since the corresponding expectation under $f$ cannot be evaluated directly in the
rare-event regime, we estimate it by importance sampling using the CFI proposal
$p_{\rm CFI}$ as the sampling distribution.
This yields the constrained forward-KL fitting problem

\begin{equation}\label{eq:fkl-c_fitting}
\max_{\bftheta}\;
\Exp_{\X\sim p_{\rm CFI}}
\!\left[
\indic\bigl(g(\X)>\gamma\bigr)\,
\frac{f(\X)}{p_{\rm CFI}(\X)}\,
\log p_{\bftheta}(\X)
\right],
\end{equation}
which is optimized over $\bftheta$ within the constrained family.

\paragraph{RKL-c (constrained reverse KL).}
Similarly, the reverse-KL surrogate introduced in~\eqref{eq:rkl_background}
defines an optimization problem over $\bftheta$.
To evaluate the corresponding expectation in a stable manner, we again rely on
the CFI proposal $p_{\rm CFI}$ as the sampling distribution, yielding
\begin{equation}\label{eq:rkl-c_fitting}
\max_{\bftheta}\;
\Exp_{\bfX\sim p_{\rm CFI}}
\!\left[
\frac{p_{\bftheta}(\bfX)}{p_{\rm CFI}(\bfX)}
\Bigl(
\log s_\tau\!\bigl(g(\bfX),\gamma\bigr)
-
\log\frac{p_{\bftheta}(\bfX)}{f(\bfX)}
\Bigr)
\right].
\end{equation}
Both FKL-c and RKL-c operate within the same constrained (-\textbf{c}) family, are initialized from $p_{\rm CFI}$ to ensure sufficient coverage,
and optimize the corresponding KL objectives by Monte Carlo approximation of the
expectations and gradient-based updates under the covariance constraint.

\paragraph{Adaptive selection of the smoothing parameter $\tau$.}
RKL-based approaches rely on the smoothed gate $s_\tau(g(\x),\gamma)$ and are therefore
  sensitive to the choice of the smoothing parameter~$\tau$. To avoid manual tuning, we set
$\tau_t=c\,\mathrm{Std}(g(\bfX)-\gamma)$,
where the standard deviation is estimated from samples drawn from the
current proposal and $c>0$ is shared across experiments.
We clip $\tau_t$ to a fixed interval $[\tau_{\min},\tau_{\max}]$ to avoid excessively sharp
or flat gates; see Appendix~\ref{app:hyperparameters} for details.

\subsection{Refinement with normalizing flows}
As a final stage, we refine constrained model $p^{\rm KL}$ obtained via RKL or FKL,
using   a more expressive proposal $    p_{\bftheta}^{\rm flow}$.
To be more specific: we train the flow model to minimize $\mathrm{KL}(p_{\bftheta}^{\rm flow}||h^{\rm opt})$
or $\mathrm{KL}(h^{\rm opt}||p_{\bftheta}^{\rm flow}$), which yields
optimizing \eqref{eq:rkl_background} or  \eqref{eq:fkl_background}  with
$p_{\bftheta}\equiv p^{\rm flow}_{\bftheta}$.  We estimate the expectations
by  importance sampling using $p^{\rm KL}$ proposals, yielding
maximization problem for FKL-c (cf.~\eqref{eq:fkl-c_fitting})
  \[
   \Exp_{\X\sim p^{\rm KL}}\!\left[
  \indic\big(g(\X)>\gamma\big)\,\frac{f(\X)}{p^{\rm KL}(\X)}\,\log p^{\rm flow}_{\bftheta}(\X)\right].
  \]
and maximization problem for RKL-c (cf.~\eqref{eq:rkl-c_fitting})
\[
\Exp_{\bfX\sim p^{\rm KL}}
\!\left[
\frac{p^{\rm flow}_{\bftheta}(\bfX)}{p^{\rm KL}(\bfX)}
\Bigl(
\log s_\tau\!\bigl(g(\bfX),\gamma\bigr)
-
\log\frac{p^{\rm flow}_{\bftheta}(\bfX)}{f(\bfX)}
\Bigr)
\right].
\]
In other words, the flow is trained to better capture the geometry of the rare-event set
while remaining anchored to the stable constrained-GMM backbone $p^{\rm KL}$.
This improves robustness empirically (note, however, that because the final proposal is a flow, the finite-variance guarantee of Theorem~\ref{thm:gmm_finite_var} does not formally apply).

  A normalizing flow~\cite{rezende2015variational} defines $p^{\rm flow}_{\bftheta}$ via an invertible
  transformation $T_{\bftheta}$:
  $\displaystyle
  p^{\rm flow}_{\bftheta}(\bfX) = p_{\mathbf{Z}}(T_{\bftheta}(\bfX)) \cdot \left| \det J_{\bftheta}(\bfX) \right|,
  $
  where $p_\mathbf{Z}$ is a simple base density ($p^{\rm KL}$, i.e., the GMM, in our case) and $J_{\bftheta}$ is the Jacobian of $h_{\bftheta}$.   To ensure tractability of $\det J_{\bftheta}$ we use RealNVP~\cite{DinhSB17}, which provides sufficient expressiveness at low computational cost.

\subsection{Motivating toy example: $\Prob(X>4), X\sim\calN(0,1)$}
\label{sec:motivating_example}
\begin{figure}[htbp]
\centering
\begin{minipage}[t]{0.40\textwidth}
\vspace{0pt}
\centering
\includegraphics[width=0.98\linewidth]{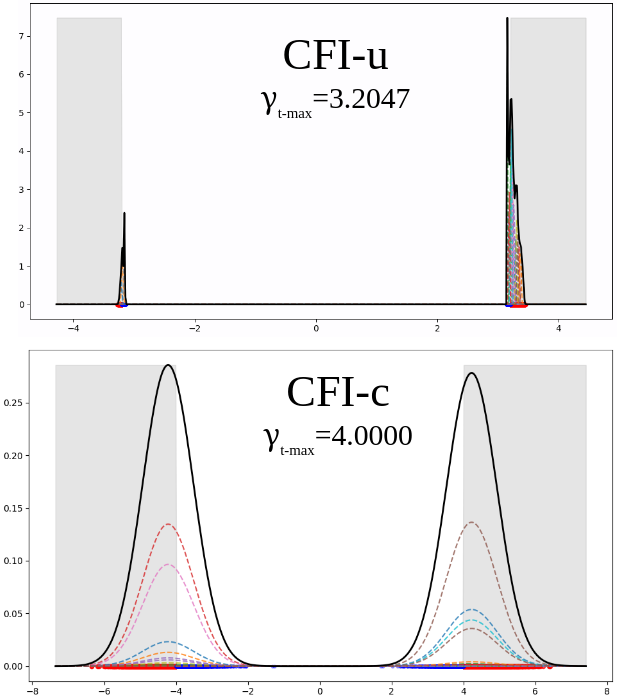}
\caption{CFI-u vs CFI-c for estimating $\Prob(|X|>4)$,
$X\sim\mathcal{N}(0,1)$. GMM with $K=20$ components. Red: samples with
$g(x_i)>\gamma_{\mathrm{t_{max}}}$; blue: samples with
$g(x_i)\leq\gamma_{\mathrm{t_{max}}}$; gray area: $|x|>4$.}
\label{fig:g0}
\end{minipage}\hfill
\begin{minipage}[t]{0.57\textwidth}
\vspace{0pt}
\input{tables/table_new_k20_g0_neurips_v3}
\end{minipage}
\end{figure}
We use the one-dimensional problem $I=\Prob(X>4)$, $X\sim\calN(0,1)$, to
illustrate why unconstrained KL fitting can be unsafe for importance sampling;
full derivations are given in Appendix~\ref{app:motivating}.

For fixed-variance Gaussian proposals $p_{\bftheta}=\calN(\mu,1)$, both FKL and
RKL yield stable proposals, with optima
$\mu^*_{\rm FKL}=4.2256$ and $\mu^*_{\rm RKL}=4.2036$ (for $\tau=0.1$).
However, once the variance is also optimized,
$p_{\bftheta}=\calN(\mu,\sigma^2)$, both objectives choose overly concentrated
proposals:
$\sigma^2_{\rm FKL}=0.0466$ and $\sigma^2_{\rm RKL}=0.4306$, below the critical
threshold $\sigma^2=\tfrac12$ required for finite IS variance. Thus, optimizing
the KL objective can still produce an infinite-variance estimator.

This pathology extends to Gaussian mixtures: unconstrained KL fitting may
collapse individual covariance matrices. This is the covariance-collapse
mechanism formalized in Proposition~\ref{propp:kl_collapse}. To prevent it,
Theorem~\ref{thm:gmm_finite_var} imposes the multivariate analogue
$\bfSigma_k \succ \tfrac12\bfSigma$, which is sufficient for finite variance.
On the related problem $\Prob(|X|>4)$ with a $K=20$ GMM proposal,
Fig.~\ref{fig:g0} and Table~\ref{tab:res_k20_g0} show that CFI-u suffers severe
variance pathologies (\xmark\ means it failed), whereas the constrained CFI-c proposal remains stable and
accurate.
Among the evaluated methods,
CFI-c+FKL-c+F achieves the best empirical variance reduction, log-error, and
CoV, while the theorem-covered CFI-c+FKL-c variant ranks second on all three
metrics.
See the definition of evaluation metrics in Section \ref{sec:exp_setup}.

\paragraph{Takeaway.}
KL fitting alone does not guarantee reliable importance sampling: proposals can
appear well optimized while losing the tail overlap needed for finite variance.
This motivates the covariance constraint and coverage-first initialization used
below.

\section{Experimental results}

\subsection{Experimental setup}\label{sec:exp_setup}
We report results for the following methods.
CFI-u denotes unconstrained coverage-first initialization and already exhibits
numerical instabilities and estimator failures across most considered problems;
as a consequence, further refinements built on CFI-u are omitted from the main
comparison.
CFI-c refers to constrained coverage-first initialization.
Starting from CFI-c, we consider constrained KL refinement using either reverse
or forward KL (CFI-c+RKL-c and CFI-c+FKL-c).
Finally, these constrained methods can be further refined using normalizing
flows (CFI-c+RKL-c+F and CFI-c+FKL-c+F).

\paragraph{External baselines.}
We additionally compare against three external rare-event baselines: Subset
Simulation (SS), Adaptive Multilevel Splitting (AMS), and
Safe-ICE~\cite{gao2025safeice,safeice_github}. SS is implemented using
\texttt{UQpy}, following the original method of Au and Beck~\cite{au2001estimation}.
AMS is implemented following the SMC framework of
Cérou et al.~\cite{cerou2012smc}. For SS and AMS, we use a matched computational
budget measured by the number of evaluations of $g$. Safe-ICE is discussed in
Appendix~\ref{app:safeice}; because the available implementation is
computationally expensive, this comparison is restricted to representative
low-, moderate-, and high-dimensional benchmarks.
We also consider a Hesterberg-style defensive counterpart of fRKL~\cite{hesterberg1995weighted}, denoted
fRKL-d (d stands for \textsl{defensive}). It uses the full defensive proposal
\[
p_{\bftheta}^{\mathrm{def}}(\bfx)
=
\lambda f(\bfx)
+
(1-\lambda)p_{\mathrm{GMM},\bftheta}(\bfx),
\]
which ensures
$f(\bfx)/p_{\bftheta}^{\mathrm{def}}(\bfx)\leq 1/\lambda$.
Unlike our covariance constraint, this construction bounds the importance
weights through a fixed nominal component.

\paragraph{Additional KL baselines.}
For completeness, Appendix~\ref{app:full_results} additionally reports results
for the classical unconstrained reverse- and forward-KL baselines (fRKL-u and
fFKL-u, where $f$ stands for \textbf{f}ull/original baseline). These methods are
representative of classical KL-based importance-sampling approaches and serve
primarily to illustrate the instability addressed by our constrained framework:
although they can occasionally perform well, they frequently fail catastrophically
due to infinite or extremely large variance.

\paragraph{IS budget.}
The sample size required for importance sampling can grow exponentially
with proposal--target mismatch~\cite{chatterjee2018sample}, motivating
our coverage-first initialization and KL refinement.
We use $R=10^5$ samples per evaluation.

\paragraph{Evaluation metrics.}
For each method, we report the rare-event probability estimate $\hat I_R$
together with complementary measures of accuracy, variability, and efficiency.
We compare $\hat{I}_R$ against $I_{\rm true}=\Prob(g(\bfX)>\gamma)$, denoting
either an analytically available value or a numerical reference obtained from
a large-scale CMC run with $5\times 10^7$ samples. For the ResNet benchmark
$g_7$, the probability is not analytically accessible; following
\cite{gao2023rare}, the high-budget CMC estimate is used only as a numerical
reference for reporting log-error and is not used by FlowCE. Empirical
variance, CoV, and variance reduction remain applicable without an analytic
value; log-error provides an additional numerical sanity check.
In the main body, we report  the logarithmic error
$\mathrm{log\text{-}err}=|\log I_{\rm true}-\log \hat I_R|$
and the  coefficient of variation (CoV),
whereas variance and variance reduction (compared to CMC method) 
are reported in full tables in Appendix \ref{app:full_results}. 
For estimators with comparable means, variance scales approximately with the
squared CoV. Thus, halving the CoV corresponds to approximately fourfold lower
variance, or four times fewer samples for the same relative precision.
For CMC and FlowCE methods evaluated under this protocol, we use 20
independent evaluation runs. For each non-CMC FlowCE estimator evaluated under
this protocol, the reported estimate $\hat I_R$, mean log-error, and CoV are
computed using all 20 runs, including runs in which the corresponding CMC
estimate is zero. The only metric-specific exclusion
concerns variance reduction relative to CMC. If CMC observes no rare event in
a given run, its empirical variance is zero and the corresponding
variance-reduction ratio is not informative; we therefore average variance
reduction only over runs in which the CMC estimate is positive. Accordingly,
the ``\#seeds used'' entry reports the number of runs used for the
variance-reduction metric and does not apply to $\hat I_R$, log-error, or CoV.
For example, $4/20$ means that CMC observed at least one rare event in four of
the 20 runs, whereas all 20 corresponding FlowCE runs were retained for the
other reported metrics.

\subsection{Benchmark problems}
We benchmark the proposed pipeline consisting of CFI initialization,
followed by FKL-c or RKL-c refinement and optional flow refinement (+F),
on a collection of rare--event problems of the form
$\Prob\!\bigl(g(\bfX)>\gamma\bigr)$ with $\bfX\sim\calN(\bfmu,\bfSigma)$.
The target probabilities range from $10^{-6}$ to $10^{-3}$; full function
definitions and experimental settings are given in
Appendix~\ref{app:benchmark_settings}.
\input{tables/tab_main_chat_v2}
We consider two covariance structures:
$\bfSigma=\mathbf{I}$ for all examples besides $g_{10}$ and 
$\bfSigma_{ii}^{g_{10}}=i$ and $\bfSigma_{ij}^{g_{10}}=0.5\sqrt{ij}$ for $i\neq j$.
Two mean vectors $\bfmu$ are
$\bfmu^0=(0,\ldots,0)$ and
$\bfmu^{g_{10}}=(1-d,\,2-d,\,\ldots,\,d-d)$, respecively.

\paragraph{Low--dimensional benchmarks.}
For $g_1$, $g_2$, and $g_3$ we consider the two--dimensional case ($d=2$) with
$\gamma=0$.
These benchmarks capture complementary geometric structures, including a
two--mode configuration ($g_1$), a radial benchmark with piecewise boundary
($g_2$), and an anisotropic problem with a linear tilt ($g_3$).
Function $g_1$ corresponds to the ``Leaf'' function from \cite{gao2023rare},
while $g_2$ and $g_3$ are structural-reliability benchmarks studied using
sequential importance sampling in~\cite{papaioannou2016sequential}.

\paragraph{Higher--dimensional benchmarks.}
We additionally consider higher--dimensional problems commonly used in the
rare--event literature.
These include the $40$--dimensional Powell function $g_5$, a standard nonconvex
benchmark considered in~\cite{gao2023rare}, and a linear function in $d=100$
dimensions ($g_6$), which isolates high-dimensional scaling in a geometrically
simple setting and was studied in~\cite{papaioannou2016sequential}. For consistency with
those papers, we set  $\gamma=-403.6636$ for $g_5$ and $\gamma=3.5$.

\begin{wrapfigure}{r}{0.50\textwidth}
  \centering
  \vspace{-2 ex}
  \setlength{\tabcolsep}{2pt}
  \begin{tabular}{cc}
  \includegraphics[width=0.47\linewidth]{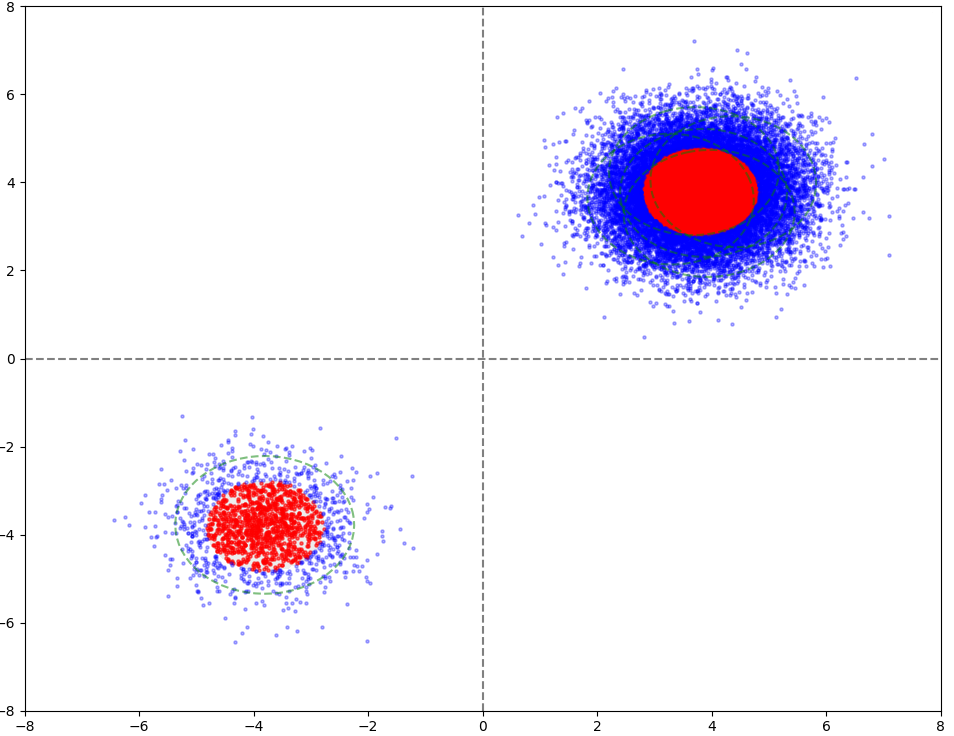} &
  \includegraphics[width=0.47\linewidth]{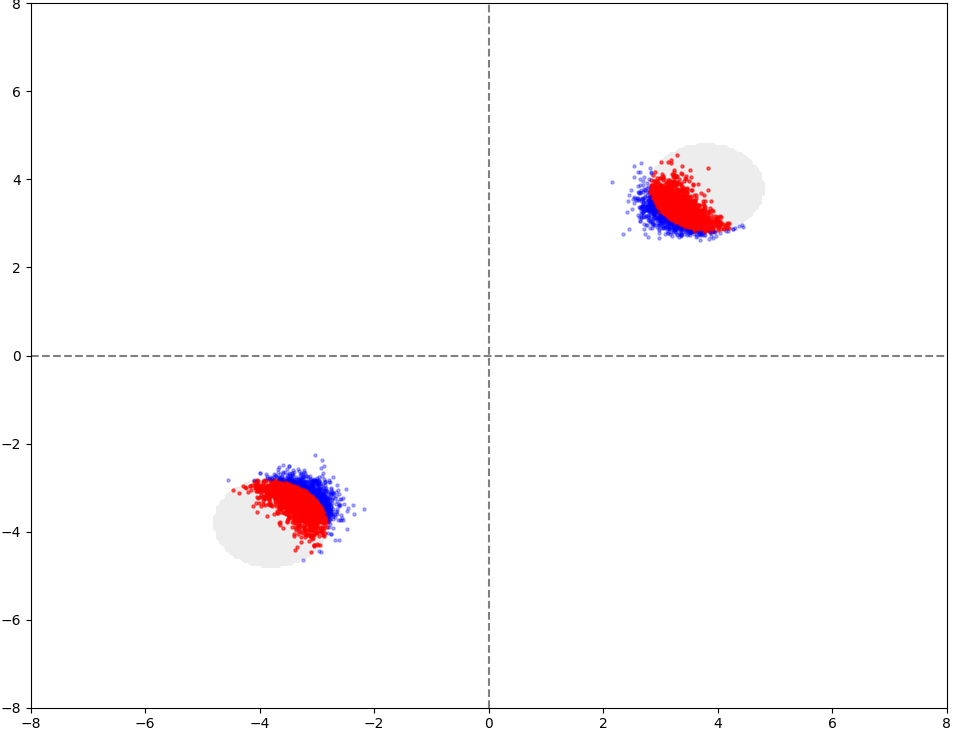} \\
  {\scriptsize CFI-c} &
  {\scriptsize CFI-c+FKL-c+F}
  \end{tabular}
  \caption{Snapshots for $\Prob\!\bigl(g_1(\bfX)>0\bigr)$ with
  $\bfX\sim\calN(\mathbf{0},\bf I)$ and $K=2$.}
  \label{fig:g1-quick}
  \vspace{-4ex}
\end{wrapfigure}

Function $g_7$ denotes the ResNet-based rare-event model introduced
in~\cite{gao2023rare}, where rare events arise from stochastic perturbations of
network parameters. We consider a ResNet-18 model subjected to a 62-dimensional
Gaussian perturbation applied in a layerwise manner, with rare events defined by
exceeding a prescribed loss threshold; implementation details are provided in
Appendix~\ref{app:benchmark_settings}. To further assess scalability, we also
report a higher-dimensional ResNet variant with $d=174$ perturbation variables
in Appendix~\ref{app:resnet174}

\paragraph{Lognormal sum benchmark.}
For $g_{10}$ we consider $d=10$ dimensional $\bfX\sim\calN(\bfmu^{g_{10}},\bfSigma^{g_{10}})$ and estimate
$\Prob\left(\sum_{i=1}^{10} \exp(X_i) > \gamma\right),$
corresponding to the tail probability of a sum of correlated lognormal random
variables.
This problem arises naturally in financial risk analysis, where $\exp(X_i)$
represents the price of the $i$th asset in a correlated Black--Scholes model.
We consider $\gamma=40{,}000$. In Appendix \ref{app:full_results} we report also values for $\gamma=20{,}000$ and $\gamma=500{,}000$ following the setup from~\cite{asmussen2011efficient}.

\paragraph{Ablations and sensitivity.}
Unless stated otherwise, all experiments use a shared PyTorch implementation
and common hyperparameters. The reverse-KL smoothing parameter $\tau$ is
selected automatically, while $\beta$-annealing allows CFI-c to reach the
target level before restoring the covariance floor; implementation details are
given in Appendix~\ref{app:hyperparameters}. The full $K=16$ comparisons in
Tables~\ref{tab:res_k16_g1g2g3}--\ref{tab:res_k16_s10} expose the pipeline
ablations: CFI provides coverage, the covariance constraint prevents collapse,
and subsequent KL or flow refinement may improve efficiency.
Tables~\ref{tab:rho_sensitivity} and~\ref{tab:alpha_sensitivity} show only mild
sensitivity to $\rho$ and $\alpha$ on $g_5$ ($d=40$). We treat $M$ separately
as a simulation budget and keep it fixed across methods and benchmarks.

\paragraph{Visualization ($d=2$, $g_1$).}
Fig.~\ref{fig:g1-quick} shows representative snapshots for
$\Prob\!\big(g_1(\bfX)>0\big)$ with $\bfX\sim\calN(\bf 0,\bf I)$
using Gaussian mixtures    with $K=2$.
The panels illustrate the CFI initialization and final CFI-c+FKL-c+F with normalizing flow. Additional visualizations, with alternative mixture sizes and RKL-c
 trajectories, are provided in Appendix~\ref{app:graph_g1}.

 \subsection{Summary of main results}
 \input{tables/tab_ratios_chat}
For $K=16$, Table~\ref{tab:headline_results} summarizes estimates,
log-errors, and CoVs, while Table~\ref{tab:failure_rates} reports
failure counts over 20 evaluations.
A run fails if its IS estimate is non-finite or zero because no positive
event contribution is observed. These counts measure empirical stability;
zero failures for flow-refined proposals do not establish finite variance.
For fFKL-c on $g_6$ and $g_7$, all 20 estimates are positive and finite
but vary substantially across seeds, especially at $K=16$, thus
theorem coverage and numerical success do not ensure small finite-budget
variance. Appendix~\ref{app:full_results} provides full results for
$K=6$ and $K=16$, including unconstrained baselines.

Unconstrained CFI-u frequently collapses. Likewise, fFKL-u and fRKL-u
can perform well but often fail, particularly in high-dimensional or
anisotropic settings. In contrast, CFI-c provides a stable starting
point for constrained KL refinement. FKL-c often achieves greater
variance reduction, notably on $g_1$ and $g_3$, while RKL-c remains
competitive.

Flow refinement (+F) often further reduces variance, especially for multimodal
and high-dimensional benchmarks
$g_5$ and $g_7$. These results support constrained CFI initialization
followed by KL refinement, with optional flows when empirically beneficial.

\paragraph{Computational cost.}
Representative timings from our reported setup illustrate the computational
cost trade-off. In these measurements, CFI is comparatively inexpensive,
whereas full adaptive KL fitting costs substantially more because it repeats
proposal optimization across intermediate levels. On $g_7$, the optional flow
stage adds approximately $38$--$42$ seconds and yields a moderate CoV
improvement. The implementations were not aggressively optimized, and
wall-clock timings are setup- and hardware-dependent; we therefore use them to
illustrate trade-offs rather than establish a universal runtime ranking.
Appendix~\ref{app:running_times} reports detailed training times for $g_1$,
$g_6$, and $g_7$.

\paragraph{Defensive importance sampling.}
We additionally evaluate fRKL-d on $g_1$ using a $K=16$ trainable GMM and the
full defensive proposal throughout adaptation. Following the guidance of
Hesterberg~\cite[Sec.~6.2]{hesterberg1995weighted} that defensive-mixture results
are not unduly sensitive when $\lambda$ is bounded away from zero and that values
between $0.1$ and $0.5$ are typically used, Table~\ref{tab:defensive_g1} reports
the representative choice $\lambda=0.1$.
\input{tables/tab_defensive_g1}
We also evaluated
$\lambda\in\{0.1,0.2,0.3,0.4,0.5\}$; the full sensitivity results are reported
in Appendix~\ref{app:full_results}. All 20 fRKL-d evaluation estimates were
nonzero. Relative to fRKL-d, fRKL-c reduces CoV by approximately $76\%$,
increases variance reduction by a factor of $17$, and reduces log-error by a
factor of $3.2$; CFI and optional flow refinement improve these metrics further.

\section*{Limitations}
The finite-variance guarantee assumes a Gaussian nominal distribution and applies only when the final proposal remains a constrained GMM; it does not extend automatically to the optional flow refinement. For non-Gaussian nominal distributions, analogous guarantees would require a separate tail analysis. CFI and a finite GMM may miss modes or approximate narrow and irregular rare-event regions poorly, while the covariance floor can limit concentration around the rare-event set. Full-covariance mixtures require \(O(Kd^2)\) parameters and enough elite
samples to estimate them; the covariance floor does not remove this sample
complexity, and low-rank or structured covariances are a natural extension. CFI reduces cold-start failures but may require many intermediate levels when the rare-event region is difficult to reach. The method is therefore most useful when cold start and covariance collapse dominate, and less attractive when crude Monte Carlo is adequate or the available budget cannot support full-covariance estimation.

\section{Conclusions}
We proposed a rare-event importance-sampling framework combining
coverage-first initialization, constrained KL refinement, and optional
normalizing flows.
Theorem~\ref{thm:gmm_finite_var} guarantees finite variance for final
constrained GMM proposals satisfying its strict covariance condition,
but not necessarily small finite-budget variance.
Across the tested benchmarks, the framework achieves favorable empirical
accuracy and stability compared with KL baselines, Subset Simulation,
AMS, and selected Safe-ICE comparisons, while unconstrained KL proposals
frequently collapse.

\begin{ack}
We gratefully acknowledge Polish high-performance computing infrastructure PLGrid (HPC Center: WCSS) for providing computer facilities and support within computational grant no. PLG/2026/019627. 
Rafał Topolnicki acknowledges the support of the project "Center for Trustworthy Artificial Intelligence for Life Sciences" implemented under the International Research Agendas programme of the Foundation for Polish Science co-financed by the European Union under the European Funds for Smart Economy 2021-2027 (FENG). 
The work of Tomasz Trzciński was supported by the National Science Centre,
Poland (NCN), under grant no.~2022/45/B/ST6/02817.

\end{ack}

  \bibliographystyle{abbrv}  
  \bibliography{flow_ce_library}

 \newpage
\appendix
\begin{center}
\hrule
\vspace{1em}
{\Large\bfseries
Robust Importance Sampling for Rare Events\\
via Constrained Gaussian Mixtures:\\[0.5em]
Appendix
}

\vspace{1em}
\hrule
\end{center}

\input{2026_NeurIPS_flowce_appendix_theory}
\input{2026_NeurIPS_flowce_appendix_proofs}
\input{2026_NeurIPS_flowce_appendix_completeresults}

\input{2026_NeurIPS_flowce_appendix_times}
\input{2026_NeurIPS_flowce_appendix_scalability}
\input{2026_NeurIPS_flowce_appendix_sensitivity}
\input{2026_NeurIPS_flowce_appendix_alternative}
\input{2026_NeurIPS_flowce_appendix_safeice}
\input{2026_NeurIPS_flowce_appendix_benchmark}
\input{2026_NeurIPS_flowce_appendix_other}

\end{document}

%% file: tables/table_new_k20_g0_neurips_v3.tex
\begingroup
\centering
\scriptsize
\captionof{table}{Toy: estimating $\Prob(|X|>4)$, $X\sim\mathcal{N}(0,1)$, $K=20$.
FG marks methods whose final proposal is a constrained GMM covered by Theorem~\ref{thm:gmm_finite_var}.
 \textbf{Bold}, \textbf{\emph{bold italic}}, and \emph{italic} mark the first, second, and third result, respectively.}
\label{tab:res_k20_g0}
\setlength{\tabcolsep}{2pt}
\rowcolors{3}{gray!11}{white}
\begin{tabular}{l c | r r r r}
\toprule
 & \multicolumn{1}{r|}{$I_{\rm true}=$} & \multicolumn{4}{c}{6.33425e-05} \\
\midrule
 \multicolumn{1}{c}{Method} & \multicolumn{1}{c|}{\textbf{FG}}
& \multicolumn{1}{c}{$\hat{I}_R$}
& \multicolumn{1}{c}{Var red.}
& \multicolumn{1}{c}{log-err}
& \multicolumn{1}{c}{CoV} \\
\midrule

\multicolumn{2}{c|}{\emph{Baselines}} & \multicolumn{4}{c}{} \\
 \hspace{0.1em}CMC &
 & 6.55e-05 & 1 & 2.65e-01 & 4.06e-01 \\
 \hspace{0.1em}fRKL-u &
 & 6.30e-05 & 2.96e+02 & 1.88e-02 & 2.37e-02 \\
 \hspace{0.1em}fFKL-u &
 & 0 & \xmark & \xmark & \xmark \\
 \hspace{0.1em}fRKL-c & \textcolor{green!60!black}{\cmark}
 & 6.32e-05 & 5.89e+02 & 1.38e-02 & 1.67e-02 \\
 \hspace{0.1em}fFKL-c & \textcolor{green!60!black}{\cmark}
 & 6.36e-05 & 3.67e+03 & 6.41e-03 & 6.65e-03 \\
 \hspace{0.1em}AMS &
 & 6.33e-05 & 4.19e-01 & 1.04e-01 & 5.98e-01 \\
 \hspace{0.1em}Subset &
 & 6.33e-05 & 6.81e-01 & 5.61e-02 & 5.23e-01 \\
 \hspace{0.1em}Safe-ICE &
 & 6.17E-05 & 5.25e-02 & 2.93e-02 & 5.45e-01 \\

\multicolumn{2}{c|}{\emph{Our methods}} & \multicolumn{4}{c}{} \\
 \hspace{0.1em}CFI-u &
 & 0 & \xmark & \xmark & \xmark \\

 \hspace{0.1em}CFI-c & \textcolor{green!60!black}{\cmark}
 & 6.34e-05 & 2.94e+03 & 5.17e-03 & 7.44e-03 \\
 \hspace{0.1em}CFI-c+RKL-c & \textcolor{green!60!black}{\cmark}
 & 6.33e-05 & 2.90e+03 & 5.72e-03 & 7.51e-03 \\
 \hspace{0.1em}CFI-c+FKL-c & \textcolor{green!60!black}{\cmark}
 & 6.33e-05 & \textbf{\emph{5.56e+03}} & \textbf{\emph{4.49e-03}} & \textbf{\emph{5.42e-03}} \\
 \hspace{0.1em}CFI-c+RKL-c+F &
 & 6.32e-05 & \emph{3.82e+03} & \emph{4.62e-03} & \emph{6.54e-03} \\
 \hspace{0.1em}CFI-c+FKL-c+F &
 & 6.34e-05 & \textbf{5.81e+03} & \textbf{4.39e-03} & \textbf{5.29e-03} \\
\bottomrule
\end{tabular}
\par
\endgroup

%% file: tables/tab_main_chat_v2.tex
\begin{table*}[t]
\centering
\footnotesize
\setlength{\tabcolsep}{2.3pt}
\caption{
Best methods by CoV at $R=10^5$ simulator evaluations: baseline
(CMC, fFKL-c, fRKL-c, CFI-u, AMS, Subset), theorem-covered proposed
method with a final constrained GMM, and proposed method overall.
Bold marks the overall lowest CoV.
}
\label{tab:headline_results}
\begin{tabular}{llr lrr lrr lrr}
\toprule
$g$
& $d$
& $I_{\rm true}$
& \multicolumn{3}{c}{Best baseline by CoV}
& \multicolumn{3}{c}{Best finite-var. method}
& \multicolumn{3}{c}{Best proposed overall} \\
\cmidrule(lr){4-6}
\cmidrule(lr){7-9}
\cmidrule(lr){10-12}
& &
& Method & log-err $\downarrow$ & CoV $\downarrow$
& Method & log-err $\downarrow$ & CoV $\downarrow$
& Method & log-err $\downarrow$ & CoV $\downarrow$ \\
\midrule
$g_1$
& 2
& 4.79e-6
& CMC & 1.43e+00 & \textbf{3.92e-03}
& CFI-c+RKL-c & 6.26e-03 & 7.24e-03
& CFI-c+FKL-c+F & 3.16e-03 & \textbf{3.92e-03} \\

$g_2$
& 2
& 2.23e-3
& CMC & 4.14e-02 & 2.46e-03
& CFI-c+FKL-c & 3.39e-03 & 4.85e-03
& CFI-c+RKL-c+F & 2.87e-03 & \textbf{1.86e-03} \\

$g_3$
& 2
& 4.21e-3
& CMC & 4.25e-02 & \textbf{3.31e-03}
& CFI-c+FKL-c & 3.25e-03 & 4.35e-03
& CFI-c+FKL-c+F & 2.80e-03 & \textbf{3.31e-03} \\

$g_5$
& 40
& 3.15e-5
& fRKL-c & 7.31e-02 & 3.19e-02
& CFI-c+RKL-c & 7.28e-02 & 4.30e-02
& CFI-c+RKL-c+F & 6.45e-02 & \textbf{2.99e-02} \\

$g_6$
& 100
& 2.33e-4
& fRKL-c & 5.90e-03 & 6.08e-03
& CFI-c+RKL-c & 7.02e-03 & 6.12e-03
& CFI-c+RKL-c+F & 5.94e-03 & \textbf{5.43e-03} \\

$g_7$
& 62
& 6.00e-5
& fRKL-c & 5.26e-02 & 8.88e-03
& CFI-c+RKL-c & 5.44e-02 & 6.80e-03
& CFI-c+RKL-c+F & 5.49e-02 & \textbf{6.24e-03} \\

$g_{10}$
& 10
& 4.68e-4
& fFKL-c & 1.24e-02 & \textbf{6.01e-03}
& CFI-c+FKL-c & 1.14e-02 & 6.80e-03
& CFI-c+FKL-c+F & 1.61e-02 & 6.22e-03 \\
\bottomrule
\end{tabular}
\end{table*}

%% file: tables/tab_ratios_chat.tex
\begin{wraptable}{r}{0.5\textwidth}
\centering
\vspace{-3ex}
\setlength{\tabcolsep}{3pt}
\caption{
Failure counts over 20 independent runs. 
This table measures reliability rather than conditional accuracy on successful runs.}
\label{tab:failure_rates}
\begin{tabular}{lrrrrrrr}
\toprule
Method
& $g_1$
& $g_2$
& $g_3$
& $g_5$
& $g_6$
& $g_7$
& $g_{10}$ \\
\midrule
CMC
& 16 & 0 & 0 & 1 & 0 & 0 & 0 \\
fRKL-c
& 0 & 0 & 0 & 0 & 0 & 0 & 0 \\
fFKL-c
& 0 & 0 & 0 & 0 & 0 & 0 & 0 \\
CFI-u
& 20 & 20 & 19 & 0 & 20 & 20 & 0 \\
\midrule
CFI-c
& 0 & 0 & 0 & 0 & 0 & 0 & 0 \\
CFI-c+RKL-c
& 0 & 0 & 0 & 0 & 0 & 0 & 0 \\
CFI-c+FKL-c
& 0 & 0 & 0 & 0 & 0 & 0 & 0 \\
CFI-c+RKL-c+F
& 0 & 0 & 0 & 0 & 0 & 0 & 0 \\
CFI-c+FKL-c+F
& 0 & 0 & 0 & 0 & 0 & 0 & 0 \\
\bottomrule
\end{tabular}
\vspace{-2ex}
\end{wraptable}

%% file: tables/tab_defensive_g1.tex
\begin{table}[H]
\centering
\small
\caption{Comparison with Hesterberg-style defensive importance sampling on
$g_1$ using $K=16$, $R=10^5$, and 20 independent evaluation runs. The
non-defensive rows are reproduced from Table~\ref{tab:res_k16_g1g2g3}.}
\label{tab:defensive_g1}
\setlength{\tabcolsep}{5pt}
\begin{tabular}{lrrrr}
\toprule
Method & Estimate & Var. reduction & log-error & CoV \\
\midrule
fRKL-d, $\lambda=0.1$
& 4.773e-06 & 1.73e+03 & 5.06e-02 & 7.53e-02 \\
fRKL-c
& \underline{4.80e-06} & \underline{2.95e+04}
& \underline{1.59e-02} & \underline{1.82e-02} \\
CFI-c+RKL-c
& \textbf{4.79e-06} & \textbf{1.87e+05}
& \textbf{6.26e-03} & \textbf{7.24e-03} \\
CFI-c+RKL-c+F
& \textbf{4.79e-06} & \textbf{7.77e+05}
& \textbf{3.55e-03} & \textbf{3.95e-03} \\
\bottomrule
\end{tabular}
\end{table}

%% file: 2026_NeurIPS_flowce_appendix_theory.tex
\section{Additional background and baseline methods}\label{app:baselines}

This appendix summarizes classical Monte Carlo and KL-based baselines for
rare-event probability estimation. These methods serve as reference points for
the constrained and coverage-first approaches introduced in the main body.

\paragraph{Crude Monte Carlo (CMC).}
The target probability can be written as
\[
I=\Exp_{\X\sim f}\big[\indic(g(\X)>\gamma)\big],
\]
and estimated by the empirical fraction of $R$ samples falling in the rare set.
The variance of the estimator equals $I(1-I)/R$, which becomes prohibitive in the
rare-event regime unless $R=\mathcal{O}(1/I)$.

\paragraph{Importance Sampling and KL fitting.}
Importance sampling introduces a proposal distribution $h$ and rewrites
\[
I=\Exp_{\X\sim h}\!\left[
\indic(g(\X)>\gamma)\,\frac{f(\X)}{h(\X)}
\right].
\]
The zero-variance proposal is
\[
h^{\rm opt}(\x)=\frac{\indic(g(\x)>\gamma)\,f(\x)}{I},
\]
which is not directly usable since it depends on the unknown $I$
(except some speciatl cases where MCMC may be employed --  but breaks i.i.d.\ draws and standard unbiased IS). We therefore restrict to a parametric family $h\equiv p_{\bftheta}$ and fit
$p_{\bftheta}$ to approximate $h^{\rm opt}$ by minimizing a KL divergence.
Due to the asymmetry of KL, this leads to two classical baselines depending on its
orientation.

\paragraph{Reverse KL (RKL -- variational inference).}
Minimizing $\mathrm{KL}(p_{\bftheta}\|h^{\rm opt})$ is ill-posed because
$h^{\rm opt}$ contains an indicator.
A standard remedy replaces the indicator with a smooth gate
\[
s_\tau(g(\x),\gamma)=\sigma\!\left(\frac{g(\x)-\gamma}{\tau}\right), \qquad
\sigma(z)=\frac{1}{1+e^{-z}},
\]
leading to the surrogate target
\[
h^{\rm opt}_{s}(\x)
=\frac{s_\tau(g(\x),\gamma)\,f(\x)}{I_s},
\qquad
I_s=\int s_\tau(g(\x),\gamma)f(\x)\,d\x .
\]
The resulting reverse-KL objective is
\begin{equation}\label{eq:vi-obj}
\Exp_{\bfX\sim p_{\bftheta}}\!\Big(
\log p_{\bftheta}(\bfX)
-\log s_\tau(g(\bfX),\gamma)
-\log f(\bfX)
\Big).
\end{equation}
Using samples $\bfX_i\sim p_{\bftheta}$ yields the empirical surrogate
\[
\hat L(\bftheta)=
\frac{1}{M}\sum_{i=1}^M
\Big(
\log p_{\bftheta}(\bfX_i)
-\log f(\bfX_i)
-\log s_\tau(g(\bfX_i),\gamma)
\Big).
\]
Gradients are estimated via the log-derivative (REINFORCE) estimator
\cite{williams1992simple}.
Writing
$\eta_{\bftheta}(\x)=\log p_{\bftheta}(\x)-\log f(\x)-\log s_\tau(g(\x),\gamma)$,
\[
\nabla_{\bftheta} L(\bftheta)
=
\Exp_{p_{\bftheta}}\!\big[
(1+\eta_{\bftheta}(\bfX)-b)\,
\nabla_{\bftheta}\log p_{\bftheta}(\bfX)
\big],
\]
where $b$ is a control variate used to reduce variance.
This yields the practical update summarized in
Alg.~\ref{alg:vi_reinforce}.

\begin{algorithm}[t]
\caption{RKL -- Reverse-KL (Variational Inference)}
\label{alg:vi_reinforce}
\begin{algorithmic}[1]
\REQUIRE Family $\{p_{\bftheta}\}$, init $\bftheta_0$, base $f$, function $g$,
level $\gamma$, batch $M$, learning rates $\{\kappa_t\}$,
schedule $\{\tau_t\}$ (or fixed $\tau$),
smoothing $\alpha$, baseline momentum $\beta$
\STATE Initialize $\bftheta'_0=\bftheta_0$, $b_0=0$, $t=0$
\REPEAT
  \STATE $t\gets t+1$; $\tau\gets\tau_t$; sample $\bfX_{1:M}\sim p_{\bftheta'_{t-1}}$
  \STATE $z_i\gets(g(\bfX_i)-\gamma)/\tau$, \ $\log s_i\gets\log\sigma(z_i)$
  \STATE $\eta_i\gets\log p_{\bftheta'_{t-1}}(\bfX_i)-\log f(\bfX_i)-\log s_i$
  \STATE Compute
  \[
  \widehat{\nabla}_{\bftheta} L_t =
  \frac{1}{M}\sum_{i=1}^M
  (1+\eta_i-b_{t-1})\,
  \nabla_{\bftheta}\log p_{\bftheta}(\bfX_i)
  \Big|_{\bftheta=\bftheta'_{t-1}}
  \]
  \STATE Update:
  $\bftheta'_t \gets \bftheta'_{t-1}-\kappa_t\,\widehat{\nabla}_{\bftheta} L_t$;
  \ $b_t \gets (1-\beta)b_{t-1}+\beta\,\frac{1}{M}\sum_i\eta_i$
  \STATE Smoothing:
  $\bftheta_t\gets\alpha\,\bftheta'_t+(1-\alpha)\,\bftheta'_{t-1}$
\UNTIL{convergence (or $t=T_{\max}$)}
\STATE \textbf{Return} $p_{\bftheta_t}$
\end{algorithmic}
\end{algorithm}

\paragraph{Forward KL (FKL -- cross-entropy).}
Minimizing $\mathrm{KL}(h^{\rm opt}\|p_{\bftheta})$ yields
\begin{equation}\label{eq:theta_opt_exp}
\bftheta^*
=
\arg\max_{\bftheta}
\Exp_{\X\sim f}\!\big[
\indic(g(\X)>\gamma)\,\log p_{\bftheta}(\X)
\big].
\end{equation}
Using samples $\bfX_j\sim f$ leads to the surrogate
\begin{equation}\label{eq:ce_pure}
\bftheta^*
\approx
\arg\max_{\bftheta}
\frac{1}{M}\sum_{j=1}^{M}
\indic(g(\bfX_j)>\gamma)\,\log p_{\bftheta}(\bfX_j),
\end{equation}
corresponding to the standard non-adaptive cross-entropy method.
This baseline suffers from severe sample starvation when $I$ is small.

\paragraph{Adaptive forward KL (fFKL -- full cross-entropy).}
The adaptive CE method iteratively increases the threshold $\gamma_t$ and updates
$p_{\bftheta}$ using importance-weighted samples, ensuring that a fixed fraction
of mass lies above the current level.
While this approach alleviates cold-start issues, it does not guarantee finite
variance.

\begin{figure}[t]
\centering
\begin{minipage}[t]{0.48\linewidth}
\begin{algorithm}[H]
\caption{FKL-u: forward KL, unconstrained (baseline)}
\label{alg:ffkl_c}
\begin{algorithmic}[1]
\REQUIRE Family (GMM), batch $M$, epochs $T$, smoothing $\alpha$
\STATE Sample $\bfX_{1:M}\sim f$
\STATE Maximize
\[
\hat J(\bftheta)=
\frac{1}{M}\sum_{i=1}^M
\indic(g(\bfX_i)\ge\gamma)\log p_{\bftheta}(\bfX_i)
\]
by gradient steps for $T$ epochs (with smoothing)
\STATE \textbf{Return} $p_{\bftheta'_T}$
\end{algorithmic}
\end{algorithm}
\end{minipage}\hfill
\begin{minipage}[t]{0.48\linewidth}

\begin{algorithm}[H]
\caption{fFKL: full adaptive CE (baseline)}
\label{alg:ffkl_baseline}
\begin{algorithmic}[1]
\REQUIRE Family $\{p_{\bftheta}\}$ with $\bftheta_0$, nominal density $f$,
function $g$, target level $\gamma$, quantile $\rho$, batch $M$, smoothing
$\alpha\in[0,1)$
\STATE $t=0$
\REPEAT
  \STATE $t\gets t+1$; sample $\bfX_{1:M}\sim p_{\bftheta_{t-1}}$
  \STATE Set $\eta_i=g(\bfX_i)$, sort:
  $\eta_{(1)}\leq\cdots\leq\eta_{(M)}$
  \STATE $(1-\rho)$ quantile:  $\gamma_t\gets\eta_{(\lceil(1-\rho)M\rceil)}$
  \STATE Compute likelihood-ratio weights\par
  $\qquad \displaystyle w_i\gets f(\bfX_i)/p_{\bftheta_{t-1}}(\bfX_i)$
  \STATE \textbf{Weighted CE update:}
  \[
  {\footnotesize
  \bftheta'_t=\arg\max_{\bftheta}\frac{1}{M}
  \sum_{i=1}^M w_i\,
  \indic\!\bigl(g(\bfX_i)\geq\gamma_t\bigr)
  \log p_{\bftheta}(\bfX_i)
  }
  \]
  \STATE $\bftheta_t\gets
  \alpha\bftheta'_t+(1-\alpha)\bftheta_{t-1}$
\UNTIL{$\gamma_t\geq\gamma$}
\STATE \textbf{Return} $p_{\bftheta_t}$
\end{algorithmic}
\end{algorithm}
\end{minipage}
\end{figure}

%% file: 2026_NeurIPS_flowce_appendix_proofs.tex
\section{Proofs}\label{app_sec:proofs}

  \subsection{Proof of Proposition  \ref{propp:kl_collapse}}
  \begin{proof}
Let $A_\gamma=\{g(\bfX)>\gamma\}$ and $I=\Prob(A_\gamma)$, where
$\bfX\sim\mathcal{N}(\bfmu,\bfSigma)$, denote its density by $f$. Let the proposal be
$p_{\bftheta}=\mathcal{N}(\bfmu',\bfSigma')$. Then
\[
\Exp_f\left[
\mathbf{1}_{A_\gamma}\log p_{\bftheta}(\bfX)
\right]
=
I\,
\Exp\left[
\log p_{\bftheta}(\bfX)\mid A_\gamma
\right].
\]
Since $I$ does not depend on $\bftheta$, maximizing the left-hand side is
equivalent to maximum-likelihood fitting a Gaussian distribution to the
conditional law $\mathcal{L}(\bfX\mid A_\gamma)$.
The Gaussian maximum-likelihood solution matches the first two moments of this
conditional law. Hence
\[
\bfmu'_{\rm opt}
=
\Exp[\bfX\mid A_\gamma],
\qquad
\bfSigma'_{\rm opt}
=
\operatorname{Cov}(\bfX\mid A_\gamma).
\]
Consequently, if
\[
\lambda_{\min}\left(
\operatorname{Cov}(\bfX\mid A_\gamma)
\right)\to 0
\quad \text{as } \gamma\to\infty,
\]
then
\[
\lambda_{\min}(\bfSigma'_{\rm opt})\to 0.
\]
Thus, unconstrained KL fitting inherits the concentration of the rare-event
conditional distribution, which explains the covariance-collapse mechanism.
\end{proof}
As mentioned earlier, such concentration is typical in rare-event regimes:
as $\gamma$ increases, the event $\{g(\bfX)>\gamma\}$ is dominated by a small
subset of the space (e.g., near a boundary or a dominating point), leading to
effectively lower-dimensional support of the conditional distribution.

  \subsection{Proof of Theorem \ref{thm:gmm_finite_var} }\label{app:proof_gmm_finite_var}

  \begin{proof}
  Let $f(\bfx) = \calN(\bfx \mid \bfmu, \bfSigma)$ denote the target density, and let
  $p(\bfx) = \sum_k \pi_k \calN(\bfx \mid \bfmu_k, \bfSigma_k)$ be the proposal GMM.

  Define the importance weight:
  \[
  w(\bfx) = \indic(g(\bfx) > \gamma) \cdot \frac{f(\bfx)}{p(\bfx)}.
  \]

Since $\Exp_p[w(\bfX)]=I<\infty$, finiteness of the variance is
equivalent to finiteness of the second moment:
\[
\Var_p[w(\bfX)]<\infty
\quad\Longleftrightarrow\quad
\Exp_p[w(\bfX)^2]<\infty.
\]
Moreover,
\[
\Exp_p[w(\bfX)^2]
=
\int
\left(\frac{f(\bfx)}{p(\bfx)}\right)^2
\indic(g(\bfx)>\gamma)\,p(\bfx)\,d\bfx
=
\int_{\{g(\bfx)>\gamma\}}
\frac{f(\bfx)^2}{p(\bfx)}\,d\bfx.
\]

Choose any mixture component $k$ with $\pi_k>0$. Then
\[
p(\bfx)
=
\sum_j \pi_j \calN(\bfx\mid\bfmu_j,\bfSigma_j)
\geq
\pi_k\calN(\bfx\mid\bfmu_k,\bfSigma_k).
\]
Therefore,
\[
\int_{\{g(\bfx)>\gamma\}}
\frac{f(\bfx)^2}{p(\bfx)}\,d\bfx
\leq
\frac{1}{\pi_k}
\int_{\mathbb R^d}
\frac{f(\bfx)^2}
{\calN(\bfx\mid\bfmu_k,\bfSigma_k)}
\,d\bfx.
\]
Up to a positive multiplicative constant independent of $\bfx$,
\[
\frac{f(\bfx)^2}
{\calN(\bfx\mid\bfmu_k,\bfSigma_k)}
\propto
\exp\!\left(
-(\bfx-\bfmu)^\top\bfSigma^{-1}(\bfx-\bfmu)
+\frac12
(\bfx-\bfmu_k)^\top
\bfSigma_k^{-1}(\bfx-\bfmu_k)
\right).
\]
If
$\bfSigma_k\succ\tfrac12\bfSigma$, then
$\bfSigma_k^{-1}\prec2\bfSigma^{-1}$, and hence
\[
\bfSigma^{-1}-\frac12\bfSigma_k^{-1}\succ0.
\]
Consequently,
\[
(\bfx-\bfmu)^\top\bfSigma^{-1}(\bfx-\bfmu)
-\frac12
(\bfx-\bfmu_k)^\top
\bfSigma_k^{-1}(\bfx-\bfmu_k)
\longrightarrow +\infty
\]
as $\|\bfx\|\to\infty$.
Thus the ratio above has Gaussian-decaying tails and is integrable over
$\mathbb R^d$.

Thus $\Exp_p[w(\bfX)^2]<\infty$, and consequently
$\Var_p[w(\bfX)]<\infty$. Since $w(\bfX_1),\ldots,w(\bfX_R)$ are
i.i.d.,
\[
\Var_p(\hat I_R^{\rm IS})
=
\frac{1}{R}\Var_p[w(\bfX)]
<\infty.
\]
  \end{proof}

\section{Motivating example: full derivation}
\label{app:motivating}

We consider the rare-event probability
\[
I=\Prob(X>4), \qquad X\sim\calN(0,1).
\]
For crude Monte Carlo (CMC), the estimator variance equals
\[
\Var(\hat I^{\rm CMC})=\frac{I(1-I)}{R},
\]
which is prohibitive in rare-event regimes unless $R$ is of order $1/I$.
We compare the classical forward-KL approach (FKL, equivalently cross-entropy)
and the reverse-KL approach (RKL, variational inference with a smoothed gate
$s_\tau$) for two Gaussian proposal families:
\[
p_{\bftheta}=\calN(\mu,1)
\qquad\text{and}\qquad
p_{\bftheta}=\calN(\mu,\sigma^2).
\]

\paragraph{Fixed variance: $p_{\bftheta}=\calN(\mu,1)$.}
\begin{itemize}
\item \textbf{FKL (CE).}
From the forward-KL objective \eqref{eq:theta_opt_exp}, the optimal mean satisfies
\begin{equation}
\mu^*_{\rm FKL}
=
\argmin_{\mu}\;
\int_{4}^{\infty} (x-\mu)^2 e^{-x^2/2}\,dx
=
4.2256.
\end{equation}
The resulting importance sampling variance equals
$4.5162\times10^{-9}$.

\item \textbf{RKL (smoothed).}
For the reverse-KL objective with smoothed gate
\[
s_\tau(x,4)=\sigma((x-4)/\tau),
\]
the optimization reduces to minimizing
\[
r(\mu;\tau)
=
\tfrac12\mu^2
-
\Exp\!\big[\log\sigma((X-4)/\tau)\big].
\]
The stationarity condition yields
\[
\mu^*_{\rm RKL}
=
\tau^{-1}\Exp\!\big[1-\sigma((X-4)/\tau)\big].
\]
For $\tau=0.1$, this gives $\mu^*_{\rm RKL}=4.2036$, with IS variance
$4.5045\times10^{-9}$.
\end{itemize}
Thus, with fixed variance, both FKL and RKL lead to well-behaved proposals and
comparable variance reduction over crude Monte Carlo.

\paragraph{Free mean and variance: $p_{\bftheta}=\calN(\mu,\sigma^2)$.}

The behavior changes qualitatively once the proposal variance is also optimized.
\begin{itemize}
\item \textbf{FKL (CE).}
Optimizing \eqref{eq:theta_opt_exp} over $(\mu,\sigma^2)$ yields
\[
(\mu^*_{\rm FKL},\sigma^{2*}_{\rm FKL})=(4.2256,\,0.046671).
\]
The second moment of the IS estimator is finite only if
\[
\Exp_p\!\left[
\mathbf{1}(X>4)\frac{f(X)^2}{p(X)^2}
\right]
=
\int_4^\infty \frac{f(x)^2}{p(x)}\,dx
<\infty.
\]
For a Gaussian proposal $p=\calN(\mu,\sigma^2)$, a direct calculation gives
\[
\int_4^\infty \frac{f(x)^2}{p(x)}\,dx
=
\frac{\sigma}{\sqrt{2\pi}}
\exp\!\Big(\frac{\mu^2}{2\sigma^2}\Big)
\int_{4}^{\infty}
\exp\!\Big(\Big(\tfrac{1}{2\sigma^2}-1\Big)x^2-\mu x\Big)\,dx.
\]
This integral diverges whenever $\sigma^2<\tfrac12$
(and also for $\sigma^2=\tfrac12$ if $\mu\le0$).
Thus, the threshold $1/2$ is sharp in this one-dimensional example, up to the
boundary case $\sigma^2=1/2$, whose behavior depends on the proposal mean and
the geometry of the event. For multivariate GMM proposals, however, the
componentwise condition in Theorem~\ref{thm:gmm_finite_var} is deliberately
conservative rather than necessary.
Since $\sigma^{2*}_{\rm FKL}=0.046671<\tfrac12$, the resulting IS estimator has
infinite variance, despite optimizing the forward-KL objective.

\item \textbf{RKL (smoothed).}
For the reverse-KL objective, we minimize
\[
  r(\mu,\sigma^2;\tau)=
  \underbrace{\tfrac12\Big(\mu^2+\sigma^2-\log\sigma^2-1\Big)}_{\mathrm{KL}(\calN(\mu,\sigma^2)\,\|\,\calN(0,1))}
  -\Exp_{X\sim\calN(\mu,\sigma^2)}\!\big[\log\sigma((X-4)/\tau)\big].
\]
The first-order conditions are
\begin{align*}
\partial_\mu r
&=
\mu
-
\tau^{-1}
\Exp\!\Big[1-\sigma\!\Big(\tfrac{\mu+\sigma\varepsilon-4}{\tau}\Big)\Big]
=0,\\
\partial_{\sigma^2} r
&=
\tfrac12\Big(1-\tfrac{1}{\sigma^2}\Big)
-
\frac{1}{2\tau\sqrt{\sigma^2}}
\Exp\!\Big[\varepsilon
\Big(1-\sigma\!\Big(\tfrac{\mu+\sigma\varepsilon-4}{\tau}\Big)\Big)\Big]
=0,
\end{align*}
with $\varepsilon\sim\calN(0,1)$.
For small $\tau$, solutions with $\sigma^{2*}_{\rm RKL}<\tfrac12$ may occur;
for example, with $\tau=0.1$ one obtains
\[
(\mu^*_{\rm RKL},\sigma^{2*}_{\rm RKL})=(3.2782,\,0.4306),
\]
again implying infinite IS variance.
\end{itemize}
\paragraph{Numerical summary.}
Table~\ref{tab:motivating_example_details} summarizes the optimal proposals and
the corresponding IS variances for the four cases discussed above.
\begin{table}[h!]
\centering
\small
\caption{Motivating example: optimal proposals obtained by FKL and RKL
and the resulting IS variance for $\Prob(X>4)$.}
\label{tab:motivating_example_details}
\begin{tabular}{llcc}
\toprule
Model &  & Optimum & IS variance \\
\midrule
$\calN(\mu,1)$ & FKL & $\mu^*=4.2256$ & $4.5162\cdot 10^{-9}$ \\
$\calN(\mu,1)$ & RKL ($\tau=0.1$) & $\mu^*=4.2036$ & $4.5045\cdot 10^{-9}$ \\
$\calN(\mu,\sigma^2)$ & FKL & $(4.2256,\,0.0467)$ & $\infty$ \\
$\calN(\mu,\sigma^2)$ & RKL ($\tau=0.1$) & $(3.2782,\,0.4306)$ & $\infty$ \\
\bottomrule
\end{tabular}
\end{table}

\paragraph{Interpretation.}
The example shows that KL fitting can concentrate the proposal too aggressively
around the rare-event conditional distribution. Such concentration may improve
the KL objective, but it can destroy the tail overlap required for stable
importance sampling. In particular, when variance parameters are unconstrained,
both FKL and RKL can drive $\sigma^2$ below the critical threshold
$\tfrac12$, leading to infinite estimator variance.

This mechanism is not specific to the one-dimensional Gaussian example. In
higher dimensions and for Gaussian mixture proposals, individual components may
collapse onto rare-event modes, producing covariance matrices that are too small
in some directions. This is the covariance-collapse mechanism formalized in
Proposition~\ref{propp:kl_collapse}. 
The constrained proposal family rules out this particular infinite-variance mechanism by enforcing the covariance floor from Theorem~\ref{thm:gmm_finite_var},
namely
\[
\bfSigma_k \succ \tfrac12\,\bfSigma,
\]
where $\bfSigma$ is the covariance matrix of the original sampling distribution.

\paragraph{Empirical GMM illustration.}
We also test the related two-sided rare-event problem
\[
I=\Prob(|X|>4), \qquad X\sim\calN(0,1),
\]
using a Gaussian mixture proposal with $K=20$ components. This setting is shown
in Fig.~\ref{fig:g0}, and the corresponding quantitative results are reported
in Table~\ref{tab:res_k20_g0}. Consistent with the derivation above, the
unconstrained CFI-u proposal exhibits severe variance pathologies, whereas the
constrained CFI-c proposal remains stable and accurate.

\paragraph{Summary.}
This example illustrates two key points. First, minimizing reverse or forward KL
alone does not guarantee a finite-variance IS estimator. Second, the issue is
not specific to the one-dimensional Gaussian case: the same covariance-collapse
mechanism appears in Gaussian mixtures and higher-dimensional settings. This
motivates the covariance constraints imposed in the main body.

%% file: 2026_NeurIPS_flowce_appendix_completeresults.tex
\section{Complete numerical results}
\label{app:full_results}
This section reports the complete numerical results for all benchmark problems
and model variants. 
The main paper reports compact summaries to preserve the flow of the experimental
section; the tables below provide the full per-benchmark numerical results.
Compared to the main body, the tables below additionally include the unconstrained
baseline methods fRKL-u and fFKL-u, which are omitted from the primary comparison
due to their lack of robustness.

In all tables below, FG marks methods whose final proposal is a constrained GMM
covered by the finite-variance guarantee of Theorem~\ref{thm:gmm_finite_var}.

\paragraph{Reference values.}
For selected benchmarks previously studied in the literature, we additionally report
reference   values taken from the corresponding publications
(denoted as ``Ref.\ results${}^{(*)}$'' or ``CEVE${}^{(*)}$'' in the tables,
with sources specified in the captions).
These values are used solely as external calibration points and are not meant
to enable direct efficiency comparisons, as experimental protocols
(e.g., IS sample sizes, number of independent runs, or variance reporting conventions)
often differ across studies and are not always fully specified.
In contrast, CMC and FlowCE results evaluated under our protocol use a fixed IS
budget and 20 independent evaluation runs, allowing for a consistent assessment
of estimator stability and variability. As specified in
Section~\ref{sec:exp_setup}, only
the CMC-dependent variance-reduction ratio excludes runs with zero CMC
estimates; the corresponding FlowCE estimates, mean log-errors, and CoVs use
all 20 runs.

\paragraph{Results for $K=6$.}
For $K=6$, we report full results for all benchmarks:
\begin{itemize}
  \item $g_1, g_2, g_3$: Table~\ref{tab:res_k6_g1g2g3_full},
  \item $g_5, g_6, g_7$: Table~\ref{tab:res_k6_g5g6g7_full},
  \item $g_{10}$: Table~\ref{tab:res_k6_s10_full}.
\end{itemize}
These tables provide a comprehensive view of estimator performance, including
methods that occasionally fail numerically or produce degenerate finite-sample
estimates. Such observed failures are distinct from a theoretically established
infinite-variance estimator.

\paragraph{Results for $K=16$.}
For $K=16$, the appendix tables provide the detailed results underlying the
compact summaries in the main body and additionally include the unconstrained
baselines:
\begin{itemize}
  \item $g_1, g_2, g_3$: Table~\ref{tab:res_k16_g1g2g3},
  \item $g_5, g_6, g_7$: Table~\ref{tab:res_k16_g5g6g7},
  \item $g_{10}$: Table~\ref{tab:res_k16_s10}.
\end{itemize}

\paragraph{Finite-budget behavior of fFKL-c on $g_6$ and $g_7$.}
For each fFKL-c setting with $K\in\{6,16\}$, the learned final proposal uses
training seed 100; the reported metrics average 20 independent ordinary,
unnormalized IS evaluations with seeds $100,200,\ldots,2000$ and $R=10^5$
each, without pooling samples across runs.  All 20 evaluations completed with
positive event counts, but this does not imply small finite-sample variance.
In particular, the $K=16$
estimates range from $1.9560\times10^{-9}$ to $7.7778\times10^{-4}$ for $g_6$
and from $3.7277\times10^{-6}$ to $1.0757\times10^{-3}$ for $g_7$.

\paragraph{Ranking and highlighting.}
In contrast to the compact main-body comparison, which reports selected
baselines and CFI-based methods, the appendix tables rank \emph{all} methods
with reportable, nondegenerate numerical estimates, including fRKL-u and
fFKL-u.
As a result, the highlighted entries (best, second-best, and third-best rows)
may differ from those shown in the main body.
This presentation is intended to provide a complete quantitative record, while
the main paper focuses on the robustness--efficiency trade-off.

\input{tables/table_k16_g1_g2_g3_v2}
\input{tables/table_k16_g5_g6_g7_v2}
\input{tables/table_k16_s10_v2}

\input{tables/table_new_k6_g1_g2_g3_avg_full_v2}
\input{tables/table_new_k6_g5_g6_g7_avg_full_v2}
\input{tables/table_new_k6_s10_avg_full_v2}

\clearpage
\paragraph{Defensive importance sampling and sensitivity to $\lambda$.}
For the $g_1$ benchmark, we additionally evaluate the Hesterberg-style
defensive counterpart of fRKL, denoted fRKL-d.  We use a $K=16$ trainable GMM
and the full defensive proposal
\[
p_{\bftheta}^{\mathrm{def}}(\bfx)
=
\lambda f(\bfx)+(1-\lambda)p_{\mathrm{GMM},\bftheta}(\bfx)
\]
throughout adaptation and evaluation.  The fixed nominal component implies
$f(\bfx)/p_{\bftheta}^{\mathrm{def}}(\bfx)\leq 1/\lambda$.
Following Hesterberg~\cite[Sec.~6.2]{hesterberg1995weighted}, we evaluate
$\lambda\in\{0.1,0.2,0.3,0.4,0.5\}$.  For each value, one trained proposal is
held fixed and evaluated using 20 independent seeds; all reported defensive-IS
estimates are nonzero.

Table~\ref{tab:defensive_lambda_sweep} shows that the defensive baseline is not
highly sensitive to $\lambda$ over this range, and the conclusions of the main
comparison remain unchanged.  The fixed component bounds the importance
weights, but in this experiment it does not prevent concentration of the
adaptive GMM component; this is an empirical observation for $g_1$, not a
general limitation of defensive importance sampling.

\input{tables/tab_defensive_lambda_sweep}

%% file: tables/table_k16_g1_g2_g3_v2.tex
\begin{table*}[t!]
\centering
\caption{Estimating $\Prob(g_i(\bfX)>\gamma), i=1,2,3, \bfX\sim\mathcal{N}(\bfmu^{(A)},\bfSigma^{(A)})$, $K=16$, $R=10^5$,  $d=2$. \\
Ref. results${}^{(*)}$: results from \cite{gao2023rare} for $g_1$ and \cite{papaioannou2016sequential} for $g_2,g_3$.}
\label{tab:res_k16_g1g2g3}
\setlength{\tabcolsep}{2pt}
\rowcolors{3}{gray!11}{white}
\makebox[\textwidth][c]{%
\resizebox{1.15\textwidth}{!}{%
\begin{tabular}{c l c |
*{5}{r}
*{5}{r} |
*{5}{r}}
\toprule
 & & & \multicolumn{5}{c}{$g_1$ \cite{gao2023rare}} & \multicolumn{5}{c}{$g_2$ \cite{papaioannou2016sequential}} & \multicolumn{5}{c}{$g_3$ \cite{papaioannou2016sequential}} \\
 & \multicolumn{1}{r}{$I_{\rm true}=$} & & \multicolumn{5}{c}{4.7934e-06} & \multicolumn{5}{c}{2.22668e-03} & \multicolumn{5}{c}{4.2073e-03} \\
\midrule
 & \multicolumn{1}{r}{} & \multicolumn{1}{c|}{\textbf{FG}}
& \multicolumn{1}{c}{$\hat{I}_R$}
& \multicolumn{1}{c}{Var red.}
& \multicolumn{1}{c}{log-err}
& \multicolumn{1}{c}{CoV}
& \multicolumn{1}{c|}{$\frac{\textrm{log-err}}{\textrm{Var red.}}$}
& \multicolumn{1}{c}{$\hat{I}_R$}
& \multicolumn{1}{c}{Var red.}
& \multicolumn{1}{c}{log-err}
& \multicolumn{1}{c}{CoV}
& \multicolumn{1}{c|}{$\frac{\textrm{log-err}}{\textrm{Var red.}}$}
& \multicolumn{1}{c}{$\hat{I}_R$}
& \multicolumn{1}{c}{Var red.}
& \multicolumn{1}{c}{log-err}
& \multicolumn{1}{c}{CoV}
& \multicolumn{1}{c}{$\frac{\textrm{log-err}}{\textrm{Var red.}}$}\\
\midrule
 & \hspace{0.1em}CMC &
 & 4.50e-06 & 1 & 1.43e+00 & \emph{3.92e-03} & 1.43e+00
 & 2.25e-03 & 1 & 4.14e-02 & \emph{2.46e-03} & 4.14e-02
 & 4.15e-03 & 1 & 4.25e-02 & \textbf{\emph{3.31e-03}} & 4.25e-02 \\

\multicolumn{3}{l|}{\emph{Ablation}} & \multicolumn{5}{c|}{} & \multicolumn{5}{c|}{} & \multicolumn{5}{c}{} \\
 & \hspace{0.1em}CFI-u &
 & 0 & \xmark & \xmark & \xmark & \xmark
 & 0 & \xmark & \xmark & \xmark & \xmark
 & 3.90e-06 & \textbf{4.57e+04} & 7.02e+00 & 1.15e-01 & 1.53e-04 \\
 & \hspace{0.1em}fRKL-c & \textcolor{green!60!black}{\cmark}
 & 4.80e-06 & 2.95e+04 & 1.59e-02 & 1.82e-02 & 5.38e-07
 & 2.22e-03 & 6.12e+01 & 5.18e-03 & 8.61e-03 & 8.46e-05
 & 4.21e-03 & 4.89e+01 & 5.98e-03 & 6.91e-03 & 1.22e-04 \\
 & \hspace{0.1em}fFKL-c & \textcolor{green!60!black}{\cmark}
 & 4.80e-06 & 1.41e+05 & 6.65e-03 & 8.31e-03 & 4.71e-08
 & 2.22e-03 & 1.89e+02 & 5.52e-03 & 4.91e-03 & 2.93e-05
 & 4.21e-03 & 1.20e+02 & 4.90e-03 & 4.41e-03 & 4.09e-05 \\

\multicolumn{3}{l|}{\emph{Baseline}} & \multicolumn{5}{c|}{} & \multicolumn{5}{c|}{} & \multicolumn{5}{c}{} \\
 & \hspace{0.1em}AMS &
 & 5.47e-06 & 6.39e-01 & 1.31e-01 & 1.34e+00 & 2.06e-01
 & 2.24e-03 & 5.47e-01 & 6.75e-03 & 7.41e-02 & 1.23e-02
 & 4.23e-03 & 3.73e-01 & 4.29e-03 & 6.77e-02 & 1.15e-02 \\
 & \hspace{0.1em}Subset &
 & 3.41e-06 & 1.48e+00 & 3.41e-01 & 1.45e+00 & 2.31e-01
 & 2.06e-03 & 7.81e-01 & 7.79e-02 & 8.42e-02 & 9.98e-02
 & 3.80e-03 & 7.93e-01 & 1.02e-01 & 5.41e-02 & 1.28e-01 \\

\multicolumn{3}{l|}{\emph{Our: CFI-based}} & \multicolumn{5}{c|}{} & \multicolumn{5}{c|}{} & \multicolumn{5}{c}{} \\
 & \hspace{0.1em}CFI-c & \textcolor{green!60!black}{\cmark}
 & 4.82e-06 & 9.85e+04 & 7.58e-03 & 9.87e-03 & 7.69e-08
 & 2.23e-03 & 5.57e+01 & 6.93e-03 & 9.00e-03 & 1.24e-04
 & 4.21e-03 & 5.76e+01 & \emph{3.77e-03} & 6.37e-03 & 6.55e-05 \\
 & \hspace{0.1em}CFI-c+RKL-c & \textcolor{green!60!black}{\cmark}
 & 4.79e-06 & \emph{1.87e+05} & 6.26e-03 & 7.24e-03 & 3.35e-08
 & 2.22e-03 & 1.44e+02 & 5.77e-03 & 5.62e-03 & 4.00e-05
 & \textbf{4.20e-03} & 1.05e+02 & \textbf{1.87e-03} & 4.71e-03 & \textbf{\emph{1.78e-05}} \\
 & \hspace{0.1em}CFI-c+FKL-c & \textcolor{green!60!black}{\cmark}
 & \emph{4.80e-06} & 1.86e+05 & \emph{6.23e-03} & 7.24e-03 & \emph{3.34e-08}
 & \emph{2.22e-03} & \emph{1.93e+02} & \emph{3.39e-03} & \emph{4.85e-03} & \emph{1.76e-05}
 & \emph{4.21e-03} & \emph{1.23e+02} & \emph{3.25e-03} & 4.35e-03 & \emph{2.63e-05} \\
 & \hspace{0.1em}CFI-c+RKL-c+F &
 & \textbf{\emph{4.79e-06}} & \textbf{7.77e+05} & \textbf{\emph{3.55e-03}} & \textbf{\emph{3.95e-03}} & \textbf{\emph{4.57e-09}}
 & \textbf{\emph{2.22e-03}} & \textbf{1.29e+03} & \textbf{\emph{2.87e-03}} & \textbf{1.86e-03} & \textbf{2.22e-06}
 & 4.20e-03 & 7.12e+01 & 3.51e-03 & \emph{4.15e-03} & 4.93e-05 \\
 & \hspace{0.1em}CFI-c+FKL-c+F &
 & \textbf{4.79e-06} & \textbf{\emph{7.51e+05}} & \textbf{3.16e-03} & \textbf{3.92e-03} & \textbf{4.21e-09}
 & \textbf{2.22e-03} & \textbf{\emph{6.54e+02}} & \textbf{2.22e-03} & \textbf{\emph{2.46e-03}} & \textbf{\emph{3.39e-06}}
 & \textbf{\emph{4.20e-03}} & \textbf{\emph{2.07e+02}} & \textbf{\emph{2.80e-03}} & \textbf{3.31e-03} & \textbf{1.35e-05} \\

\multicolumn{3}{c|}{\hspace{-0.9cm}\emph{\#seeds used}}
 & \multicolumn{5}{c|}{4/20}
 & \multicolumn{5}{c|}{20/20}
 & \multicolumn{5}{c}{20/20} \\
\bottomrule
 & \hspace{0.1em}Ref. results${}^{(*)}$ &
 & -- & -- & 1.10e-01 & -- & --
 & 2.23e-03 & -- & -- & 1.10e-01 & --
 & 4.20e-03 & -- & -- & 6.00e-02 & -- \\
\bottomrule
\end{tabular}
}
}
\end{table*}

%% file: tables/table_k16_g5_g6_g7_v2.tex
\begin{table*}[t!]
\centering
\scriptsize
\caption{Estimating $\Prob(g_i(\bfX)>\gamma), i=5,6,7$, with $K=16$ and $R=10^5$.\\
Ref. results${}^{(*)}$: results from \cite{gao2023rare} for $g_5,g_7$ and \cite{papaioannou2016sequential} for $g_6$.}
\label{tab:res_k16_g5g6g7}
\setlength{\tabcolsep}{2pt}
\rowcolors{3}{gray!11}{white}
\makebox[\textwidth][c]{%
\resizebox{1.15\textwidth}{!}{%
\begin{tabular}{c l c |
*{5}{r}
*{5}{r} |
*{5}{r}}
\toprule
 & & & \multicolumn{5}{c}{$g_5$ \cite{gao2023rare}} & \multicolumn{5}{c}{$g_6$ \cite{papaioannou2016sequential}} & \multicolumn{5}{c}{$g_7$ \cite{gao2023rare}} \\
 & \multicolumn{1}{r}{$I_{\rm true}=$} & & \multicolumn{5}{c}{3.1500e-05} & \multicolumn{5}{c}{2.3262e-04} & \multicolumn{5}{c}{6.00e-05} \\
\midrule
 & \multicolumn{1}{r}{} & \multicolumn{1}{c|}{\textbf{FG}}
& \multicolumn{1}{c}{$\hat{I}_R$}
& \multicolumn{1}{c}{Var red.}
& \multicolumn{1}{c}{log-err}
& \multicolumn{1}{c}{CoV}
& \multicolumn{1}{c|}{$\frac{\textrm{log-err}}{\textrm{Var red.}}$}
& \multicolumn{1}{c}{$\hat{I}_R$}
& \multicolumn{1}{c}{Var red.}
& \multicolumn{1}{c}{log-err}
& \multicolumn{1}{c}{CoV}
& \multicolumn{1}{c|}{$\frac{\textrm{log-err}}{\textrm{Var red.}}$}
& \multicolumn{1}{c}{$\hat{I}_R$}
& \multicolumn{1}{c}{Var red.}
& \multicolumn{1}{c}{log-err}
& \multicolumn{1}{c}{CoV}
& \multicolumn{1}{c}{$\frac{\textrm{log-err}}{\textrm{Var red.}}$}\\
\midrule
 & \hspace{0.1em}CMC &
 & 3.30e-05 & 1 & 4.59e-01 & 1.17e-01 & 4.59e-01
 & 2.27e-04 & 1 & 1.67e-01 & 1.80e-01 & 1.67e-01
 & 5.00e-05 & 1 & 3.35e-01 & 1.65e-02 & 3.35e-01 \\

\multicolumn{3}{l|}{\emph{Ablation}} & \multicolumn{5}{c|}{} & \multicolumn{5}{c|}{} & \multicolumn{5}{c}{} \\
 & \hspace{0.1em}fRKL-c & \textcolor{green!60!black}{\cmark} & \emph{3.39e-05} & \textbf{\emph{2.97e+02}} & \emph{7.31e-02} & \textbf{\emph{3.19e-02}} & \textbf{\emph{2.46e-04}}
 & \textbf{2.32e-04} & \emph{1.14e+03} & \textbf{5.90e-03} & \textbf{\emph{6.08e-03}} & \emph{5.20e-06}
 & 5.69e-05 & \emph{1.96e+03} & 5.26e-02 & \emph{8.88e-03} & \emph{2.69e-05} \\
 & \hspace{0.1em}fFKL-c & \textcolor{green!60!black}{\cmark} & 3.39e-05 & \emph{2.82e+02} & 7.32e-02 & \emph{3.27e-02} & \emph{2.60e-04}
 & 3.90e-05 & \textbf{3.82e+08} & 8.22e+00 & 6.36e-01 & \textbf{2.15e-08}
 & 1.19e-04 & 4.66e+01 & 1.60e+00 & 5.96e-01 & 3.42e-02 \\
 & \hspace{0.1em}CFI-u &
 & 5.24e-07 & 1.96e+02 & 5.33e+00 & 4.54e-01 & 2.72e-02
 & 0 & \xmark & \xmark & \xmark & \xmark
 & 0 & \xmark & \xmark & \xmark & \xmark \\

\multicolumn{3}{l|}{\emph{Baseline}} & \multicolumn{5}{c|}{} & \multicolumn{5}{c|}{} & \multicolumn{5}{c}{} \\
 & \hspace{0.1em}AMS &
 & 2.75e-05 & 9.26e-01 & 1.34e-01 & 9.51e-01 & 1.45e-01
 & 2.30e-04 & 6.96e-02 & 9.20e-03 & 6.61e-01 & 1.32e-01
 & 6.48e-05 & 1.50e-01 & 7.73e-02 & 8.15e-01 & 5.14e-01 \\
 & \hspace{0.1em}Subset &
 & 3.71e-05 & 7.07e-01 & 1.63e-01 & 5.93e-01 & 2.31e-01
 & 2.14e-04 & 1.78e-01 & 8.41e-02 & 3.92e-01 & 4.73e-01
 & \textbf{6.22e-05} & 5.44e-01 & \textbf{3.57e-02} & 4.93e-01 & 6.56e-02 \\

\multicolumn{3}{l|}{\emph{Our: CFI-based}} & \multicolumn{5}{c|}{} & \multicolumn{5}{c|}{} & \multicolumn{5}{c}{} \\
 & \hspace{0.1em}CFI-c & \textcolor{green!60!black}{\cmark} & 3.27e-05 & 4.24e+01 & 8.07e-02 & 8.89e-02 & 1.90e-03
 & 2.34e-04 & 8.33e+00 & 5.09e-02 & 5.81e-02 & 6.11e-03
 & \textbf{\emph{5.72e-05}} & 2.28e+02 & \textbf{\emph{5.00e-02}} & 2.42e-02 & 2.19e-04 \\
 & \hspace{0.1em}CFI-c+RKL-c & \textcolor{green!60!black}{\cmark} & \textbf{\emph{3.39e-05}} & 1.49e+02 & \textbf{\emph{7.28e-02}} & 4.30e-02 & 4.88e-04
 & \emph{2.32e-04} & 1.12e+03 & \emph{7.02e-03} & \emph{6.12e-03} & 6.26e-06
 & 5.68e-05 & \textbf{\emph{3.35e+03}} & 5.44e-02 & \textbf{\emph{6.80e-03}} & \textbf{\emph{1.62e-05}} \\
 & \hspace{0.1em}CFI-c+FKL-c & \textcolor{green!60!black}{\cmark} & 3.40e-05 & 7.60e+00 & 1.48e-01 & 1.80e-01 & 1.94e-02
 & 2.34e-04 & 4.14e-01 & 1.77e-01 & 1.88e-01 & 4.29e-01
 & \emph{5.70e-05} & 3.26e+02 & \emph{5.09e-02} & 2.14e-02 & 1.56e-04 \\
 & \hspace{0.1em}CFI-c+RKL-c+F &
 & \textbf{3.36e-05} & \textbf{3.15e+02} & \textbf{6.45e-02} & \textbf{2.99e-02} & \textbf{2.05e-04}
 & \textbf{\emph{2.33e-04}} & \textbf{\emph{1.42e+03}} & \textbf{\emph{5.94e-03}} & \textbf{5.43e-03} & \textbf{\emph{4.19e-06}}
 & 5.68e-05 & \textbf{3.97e+03} & 5.49e-02 & \textbf{6.24e-03} & \textbf{1.38e-05} \\
 & \hspace{0.1em}CFI-c+FKL-c+F &
 & 3.22e-05 & 1.32e+01 & 9.17e-02 & 1.17e-01 & 6.96e-03
 & 2.28e-04 & 5.48e-01 & 1.79e-01 & 1.80e-01 & 3.26e-01
 & 5.65e-05 & 5.72e+02 & 6.11e-02 & 1.65e-02 & 1.07e-04 \\

\multicolumn{3}{c|}{\hspace{-0.9cm}\emph{\#seeds used}}
 & \multicolumn{5}{c|}{19/20}
 & \multicolumn{5}{c|}{20/20}
 & \multicolumn{5}{c}{20/20} \\
\bottomrule
 & \hspace{0.1em}Ref. results${}^{(*)}$ &
 & -- & -- & 3.80e-01 & -- & --
 & 2.33e-04 & -- & -- & 2.80e-01 & --
 & -- & -- & 6.20e-01 & -- & -- \\
\bottomrule
\end{tabular}%
}%
}
\end{table*}

%% file: tables/table_k16_s10_v2.tex
\begin{table*}[t!]
\centering
\scriptsize
\caption{Estimating $\Prob(g_{10}(\bfX)>\gamma)$ for different thresholds, with $K=16$ and $R=10^5$. \\
CEVE${}^{(*)}$: results from \cite{asmussen2011efficient}.}
\label{tab:res_k16_s10}
\setlength{\tabcolsep}{2pt}
\rowcolors{3}{gray!11}{white}
\makebox[\textwidth][c]{%
\resizebox{1.15\textwidth}{!}{%
\begin{tabular}{c l c |
*{5}{r}
*{5}{r} |
*{5}{r}}
\toprule
 & & & \multicolumn{5}{c}{$\gamma=20{,}000$} & \multicolumn{5}{c}{$\gamma=40{,}000$} & \multicolumn{5}{c}{$\gamma=500{,}000$} \\
 & \multicolumn{1}{r}{$I_{\rm true}=$} & & \multicolumn{5}{c}{1.03526e-03} & \multicolumn{5}{c}{4.683e-04} & \multicolumn{5}{c}{1.77400e-05} \\
\midrule
 & \multicolumn{1}{r}{} & \multicolumn{1}{c|}{\textbf{FG}}
& \multicolumn{1}{c}{$\hat{I}_R$}
& \multicolumn{1}{c}{Var red.}
& \multicolumn{1}{c}{log-err}
& \multicolumn{1}{c}{CoV}
& \multicolumn{1}{c|}{$\frac{\textrm{log-err}}{\textrm{Var red.}}$}
& \multicolumn{1}{c}{$\hat{I}_R$}
& \multicolumn{1}{c}{Var red.}
& \multicolumn{1}{c}{log-err}
& \multicolumn{1}{c}{CoV}
& \multicolumn{1}{c|}{$\frac{\textrm{log-err}}{\textrm{Var red.}}$}
& \multicolumn{1}{c}{$\hat{I}_R$}
& \multicolumn{1}{c}{Var red.}
& \multicolumn{1}{c}{log-err}
& \multicolumn{1}{c}{CoV}
& \multicolumn{1}{c}{$\frac{\textrm{log-err}}{\textrm{Var red.}}$}\\
\midrule
 & \hspace{0.1em}CMC &
 & 1.03e-03 & 1 & 7.22e-02 & \emph{6.22e-03} & 7.22e-02
 & 4.69e-04 & 1 & 1.12e-01 & \emph{6.22e-03} & 1.12e-01
 & 2.30e-05 & 1 & 4.28e-01 & \textbf{\emph{1.68e-02}} & 4.28e-01 \\

\multicolumn{3}{l|}{\emph{Ablation}} & \multicolumn{5}{c|}{} & \multicolumn{5}{c|}{} & \multicolumn{5}{c}{} \\
 & \hspace{0.1em}CFI-u &
 & 0 & \xmark & \xmark & \xmark & \xmark
 & 0 & \xmark & \xmark & \xmark & \xmark
 & 0 & \xmark & \xmark & \xmark & \xmark \\
 & \hspace{0.1em}fRKL-c & \textcolor{green!60!black}{\cmark} & 1.03e-03 & 2.76e+01 & 1.43e-02 & 1.88e-02 & 5.18e-04
 & 4.61e-04 & 4.89e+01 & 2.27e-02 & 2.12e-02 & 4.64e-04
 & 1.78e-05 & \textbf{\emph{1.25e+03}} & 2.94e-02 & 2.53e-02 & \emph{2.34e-05} \\
 & \hspace{0.1em}fFKL-c & \textcolor{green!60!black}{\cmark} & \textbf{1.02e-03} & \textbf{3.02e+02} & \textbf{1.09e-02} & \textbf{5.70e-03} & \textbf{3.62e-05}
 & \textbf{\emph{4.63e-04}} & \textbf{6.07e+02} & \textbf{\emph{1.24e-02}} & \textbf{6.01e-03} & \textbf{2.04e-05}
 & \textbf{1.79e-05} & \textbf{1.30e+04} & \textbf{1.19e-02} & \textbf{7.81e-03} & \textbf{9.15e-07} \\

\multicolumn{3}{l|}{\emph{Baseline}} & \multicolumn{5}{c|}{} & \multicolumn{5}{c|}{} & \multicolumn{5}{c}{} \\
 & \hspace{0.1em}AMS &
 & \emph{1.02e-03} & 1.72e-01 & \emph{1.17e-02} & 2.02e-01 & 6.82e-02
 & \textbf{4.71e-04} & 2.49e-01 & \textbf{5.48e-03} & 3.31e-01 & 2.20e-02
 & 1.70e-05 & 8.09e-01 & 4.15e-02 & 1.00e+00 & 5.12e-02 \\
 & \hspace{0.1em}Subset &
 & 9.72e-04 & 9.51e-01 & 6.29e-02 & 1.08e-01 & 6.62e-02
 & 4.58e-04 & 1.78e+00 & 2.33e-02 & 1.21e-01 & 1.31e-02
 & 1.48e-05 & 2.96e+00 & 1.79e-01 & 5.47e-01 & 6.06e-02 \\

\multicolumn{3}{l|}{\emph{Our: CFI-based}} & \multicolumn{5}{c|}{} & \multicolumn{5}{c|}{} & \multicolumn{5}{c}{} \\
 & \hspace{0.1em}CFI-c & \textcolor{green!60!black}{\cmark} & 1.02e-03 & 5.28e+01 & 1.26e-02 & 1.33e-02 & 2.39e-04
 & 4.65e-04 & 8.76e+01 & 1.35e-02 & 1.37e-02 & 1.54e-04
 & 1.79e-05 & 2.16e+02 & 3.04e-02 & 3.25e-02 & 1.41e-04 \\
 & \hspace{0.1em}CFI-c+RKL-c & \textcolor{green!60!black}{\cmark} & 1.02e-03 & 7.63e+00 & 2.87e-02 & 3.12e-02 & 3.77e-03
 & 4.66e-04 & 5.33e+01 & 1.39e-02 & 1.98e-02 & 2.61e-04
 & 1.86e-05 & 6.26e+01 & 6.11e-02 & 6.36e-02 & 9.76e-04 \\
 & \hspace{0.1em}CFI-c+FKL-c & \textcolor{green!60!black}{\cmark} & 1.02e-03 & \emph{1.66e+02} & 1.24e-02 & 7.46e-03 & 7.47e-05
 & \emph{4.63e-04} & \textbf{\emph{4.68e+02}} & 1.14e-02 & 6.80e-03 & \textbf{\emph{2.44e-05}}
 & \textbf{\emph{1.78e-05}} & 1.19e+03 & \textbf{\emph{1.31e-02}} & \emph{1.80e-02} & \textbf{\emph{1.10e-05}} \\
 & \hspace{0.1em}CFI-c+RKL-c+F &
 & 1.01e-03 & 1.79e+01 & 2.52e-02 & 2.26e-02 & 1.41e-03
 & 4.59e-04 & 1.39e+02 & 2.25e-02 & \emph{1.14e-02} & 1.62e-04
 & 1.76e-05 & 1.15e+02 & 5.15e-02 & 3.22e-02 & 4.48e-04 \\
 & \hspace{0.1em}CFI-c+FKL-c+F &
 & \textbf{\emph{1.02e-03}} & \textbf{\emph{1.90e+02}} & \textbf{\emph{1.16e-02}} & \textbf{\emph{6.22e-03}} & \textbf{\emph{6.10e-05}}
 & 4.61e-04 & \emph{4.02e+02} & 1.61e-02 & \textbf{\emph{6.22e-03}} & \emph{3.99e-05}
 & \emph{1.78e-05} & \emph{3.69e+03} & \emph{1.46e-02} & \textbf{1.68e-02} & \textbf{3.96e-06} \\

\multicolumn{3}{c|}{\hspace{-0.9cm}\emph{\#seeds used}}
 & \multicolumn{5}{c|}{20/20}
 & \multicolumn{5}{c|}{20/20}
 & \multicolumn{5}{c}{18/20} \\
\bottomrule
 & \hspace{0.1em}CEVE${}^{(*)}$ &
 & 1.02e-03 & -- & -- & 1.28e-01 & --
 & 4.63e-04 & -- & -- & 8.60e-02 & --
 & 1.79e-05 & -- & -- & 4.46e-02 & -- \\
\bottomrule
\end{tabular}%
}%
}
\end{table*}

%% file: tables/table_new_k6_g1_g2_g3_avg_full_v2.tex
\begin{table*}[t!]
\centering
\scriptsize
\caption{Estimating  $\Prob(g_i(\bfX)>\gamma), i=1,2,3, \bfX\sim\mathcal{N}(\bfmu^{(A)},\bfSigma^{(A)})$, $K=6$,
$R=10^5$,  $d=2$. \\
Ref. results${}^{(*)}$: results from \cite{gao2023rare} for $g_1$ and \cite{papaioannou2016sequential} for $g_2,g_3$.}
\label{tab:res_k6_g1g2g3_full}
\setlength{\tabcolsep}{2pt}
\rowcolors{3}{gray!11}{white}
\makebox[\textwidth][c]{%
\resizebox{1.15\textwidth}{!}{%
\begin{tabular}{c l c |
*{5}{S[scientific-notation=true,table-format=1.3e2,table-number-alignment=right,table-text-alignment=right]}
*{5}{S[scientific-notation=true,table-format=1.3e2,table-number-alignment=right,table-text-alignment=right]} |
*{5}{S[scientific-notation=true,table-format=1.3e2,table-number-alignment=right,table-text-alignment=right]}}
\toprule
 & & & \multicolumn{5}{c}{$g_1$ \cite{gao2023rare}} & \multicolumn{5}{c}{$g_2$ \cite{papaioannou2016sequential}} & \multicolumn{5}{c}{$g_3$ \cite{papaioannou2016sequential}} \\
 & \multicolumn{1}{r}{$I_{\rm true}=$} &  & \multicolumn{5}{c}{4.7934e-06} & \multicolumn{5}{c}{2.22668e-03} & \multicolumn{5}{c}{4.2073e-03} \\
\midrule
 & \multicolumn{1}{r}{} & \multicolumn{1}{c|}{\textbf{FG}}
& \multicolumn{1}{c}{$\hat{I}_R$}
& \multicolumn{1}{c}{Var red.}
& \multicolumn{1}{c}{log-err}
& \multicolumn{1}{c}{CoV}
& \multicolumn{1}{c|}{$\frac{\textrm{log-err}}{\textrm{Var red.}}$}
& \multicolumn{1}{c}{$\hat{I}_R$}
& \multicolumn{1}{c}{Var red.}
& \multicolumn{1}{c}{log-err}
& \multicolumn{1}{c}{CoV}
& \multicolumn{1}{c|}{$\frac{\textrm{log-err}}{\textrm{Var red.}}$}
& \multicolumn{1}{c}{$\hat{I}_R$}
& \multicolumn{1}{c}{Var red.}
& \multicolumn{1}{c}{log-err}
& \multicolumn{1}{c}{CoV}
& \multicolumn{1}{c}{$\frac{\textrm{log-err}}{\textrm{Var red.}}$}\\

\midrule

 & \hspace{0.1em}CMC  && \multicolumn{1}{r}{4.50e-06} & \multicolumn{1}{r}{1} & \multicolumn{1}{r}{1.43e+00} & \multicolumn{1}{r}{3.92e-03} & \multicolumn{1}{r|}{1.43e+00} & \multicolumn{1}{r}{2.25e-03} & \multicolumn{1}{r}{1} & \multicolumn{1}{r}{4.14e-02} & \multicolumn{1}{r}{2.46e-03} & \multicolumn{1}{r|}{4.14e-02} & \multicolumn{1}{r}{4.15e-03} & \multicolumn{1}{r}{1} & \multicolumn{1}{r}{4.25e-02} & \multicolumn{1}{r}{3.31e-03} & \multicolumn{1}{r}{4.25e-02} \\

  \multicolumn{3}{l|}{\emph{Baselines (no CFI)}} & \multicolumn{5}{c|}{} & \multicolumn{5}{c|}{} & \multicolumn{5}{c}{} \\

 & \hspace{0.1em}fRKL-u  && \multicolumn{1}{r}{4.74e-06} & \multicolumn{1}{r}{1.96e+03} & \multicolumn{1}{r}{5.56e-02} & \multicolumn{1}{r}{6.95e-02} & \multicolumn{1}{r|}{2.83e-05} & \multicolumn{1}{r}{2.22e-03} & \multicolumn{1}{r}{6.43e+01} & \multicolumn{1}{r}{7.68e-03} & \multicolumn{1}{r}{8.40e-03} & \multicolumn{1}{r|}{1.19e-04} & \multicolumn{1}{r}{4.21e-03} & \multicolumn{1}{r}{4.30e+01} & \multicolumn{1}{r}{5.73e-03} & \multicolumn{1}{r}{7.02e-03} & \multicolumn{1}{r}{1.33e-04} \\
 & \hspace{0.1em}fFKL-u  && \multicolumn{1}{r}{\emph{4.80e-06}} & \multicolumn{1}{r}{\emph{3.15e+05}} & \multicolumn{1}{r}{\emph{5.65e-03}} & \multicolumn{1}{r}{\emph{5.56e-03}} & \multicolumn{1}{r|}{\emph{1.79e-08}} & \multicolumn{1}{r}{2.21e-03} & \multicolumn{1}{r}{2.69e+02} & \multicolumn{1}{r}{8.05e-03} & \multicolumn{1}{r}{4.10e-03} & \multicolumn{1}{r|}{3.00e-05} & \multicolumn{1}{r}{\textbf{4.21e-03}} & \multicolumn{1}{r}{\textbf{5.07e+02}} & \multicolumn{1}{r}{\textbf{1.38e-03}} & \multicolumn{1}{r}{\textbf{2.14e-03}} & \multicolumn{1}{r}{\textbf{2.72e-06}} \\
 \multicolumn{3}{l|}{\emph{Ablation}} & \multicolumn{5}{c|}{} & \multicolumn{5}{c|}{} & \multicolumn{5}{c}{} \\
 & \hspace{0.1em}CFI-u  && \multicolumn{1}{r}{0} & \multicolumn{1}{r}{\xmark} & \multicolumn{1}{r}{\xmark} & \multicolumn{1}{r}{\xmark} & \multicolumn{1}{r|}{\xmark} & \multicolumn{1}{r}{0} & \multicolumn{1}{r}{\xmark} & \multicolumn{1}{r}{\xmark} & \multicolumn{1}{r}{\xmark} & \multicolumn{1}{r|}{\xmark} & \multicolumn{1}{r}{3.90e-06} & \multicolumn{1}{r}{4.57e+04} & \multicolumn{1}{r}{7.02e+00} & \multicolumn{1}{r}{1.15e-01} & \multicolumn{1}{r}{1.53e-04} \\
 & \hspace{0.1em}fRKL-c & \textcolor{green!60!black}{\cmark} & \multicolumn{1}{r}{4.80e-06} & \multicolumn{1}{r}{2.95e+04} & \multicolumn{1}{r}{1.59e-02} & \multicolumn{1}{r}{1.82e-02} & \multicolumn{1}{r|}{5.38e-07} & \multicolumn{1}{r}{2.22e-03} & \multicolumn{1}{r}{6.12e+01} & \multicolumn{1}{r}{5.18e-03} & \multicolumn{1}{r}{8.61e-03} & \multicolumn{1}{r|}{8.46e-05} & \multicolumn{1}{r}{4.21e-03} & \multicolumn{1}{r}{4.89e+01} & \multicolumn{1}{r}{5.98e-03} & \multicolumn{1}{r}{6.91e-03} & \multicolumn{1}{r}{1.22e-04} \\
 & \hspace{0.1em}fFKL-c & \textcolor{green!60!black}{\cmark} & \multicolumn{1}{r}{4.80e-06} & \multicolumn{1}{r}{1.41e+05} & \multicolumn{1}{r}{6.65e-03} & \multicolumn{1}{r}{8.31e-03} & \multicolumn{1}{r|}{4.71e-08} & \multicolumn{1}{r}{2.22e-03} & \multicolumn{1}{r}{1.89e+02} & \multicolumn{1}{r}{5.52e-03} & \multicolumn{1}{r}{4.91e-03} & \multicolumn{1}{r|}{2.93e-05} & \multicolumn{1}{r}{4.21e-03} & \multicolumn{1}{r}{1.20e+02} & \multicolumn{1}{r}{4.90e-03} & \multicolumn{1}{r}{4.41e-03} & \multicolumn{1}{r}{4.09e-05} \\

\multicolumn{3}{l|}{\emph{Our: CFI-based}} & \multicolumn{5}{c|}{} & \multicolumn{5}{c|}{} & \multicolumn{5}{c}{} \\
 & \hspace{0.1em}CFI-c & \textcolor{green!60!black}{\cmark} & \multicolumn{1}{r}{4.82e-06} & \multicolumn{1}{r}{9.85e+04} & \multicolumn{1}{r}{7.58e-03} & \multicolumn{1}{r}{9.87e-03} & \multicolumn{1}{r|}{7.69e-08} & \multicolumn{1}{r}{2.23e-03} & \multicolumn{1}{r}{5.57e+01} & \multicolumn{1}{r}{6.93e-03} & \multicolumn{1}{r}{9.00e-03} & \multicolumn{1}{r|}{1.24e-04} & \multicolumn{1}{r}{4.21e-03} & \multicolumn{1}{r}{5.76e+01} & \multicolumn{1}{r}{3.77e-03} & \multicolumn{1}{r}{6.37e-03} & \multicolumn{1}{r}{6.55e-05} \\
 & \hspace{0.1em}CFI-c+RKL-c & \textcolor{green!60!black}{\cmark} & \multicolumn{1}{r}{4.79e-06} & \multicolumn{1}{r}{1.87e+05} & \multicolumn{1}{r}{6.26e-03} & \multicolumn{1}{r}{7.24e-03} & \multicolumn{1}{r|}{3.35e-08} & \multicolumn{1}{r}{2.22e-03} & \multicolumn{1}{r}{1.44e+02} & \multicolumn{1}{r}{5.77e-03} & \multicolumn{1}{r}{5.62e-03} & \multicolumn{1}{r|}{4.00e-05} & \multicolumn{1}{r}{\emph{4.20e-03}} & \multicolumn{1}{r}{\emph{1.05e+02}} & \multicolumn{1}{r}{\emph{1.87e-03}} & \multicolumn{1}{r}{\emph{4.71e-03}} & \multicolumn{1}{r}{\emph{1.78e-05}} \\
 & \hspace{0.1em}CFI-c+FKL-c & \textcolor{green!60!black}{\cmark} & \multicolumn{1}{r}{4.80e-06} & \multicolumn{1}{r}{1.86e+05} & \multicolumn{1}{r}{6.23e-03} & \multicolumn{1}{r}{7.24e-03} & \multicolumn{1}{r|}{3.34e-08} & \multicolumn{1}{r}{\emph{2.22e-03}} & \multicolumn{1}{r}{\emph{1.93e+02}} & \multicolumn{1}{r}{\emph{3.39e-03}} & \multicolumn{1}{r}{\emph{4.85e-03}} & \multicolumn{1}{r|}{\emph{1.76e-05}} & \multicolumn{1}{r}{4.21e-03} & \multicolumn{1}{r}{1.23e+02} & \multicolumn{1}{r}{3.25e-03} & \multicolumn{1}{r}{4.35e-03} & \multicolumn{1}{r}{2.63e-05} \\
 & \hspace{0.1em}CFI-c+RKL-c+F  && \multicolumn{1}{r}{\textbf{\emph{4.79e-06}}} & \multicolumn{1}{r}{\textbf{\emph{7.77e+05}}} & \multicolumn{1}{r}{\textbf{\emph{3.55e-03}}} & \multicolumn{1}{r}{\textbf{\emph{3.95e-03}}} & \multicolumn{1}{r|}{\textbf{\emph{4.57e-09}}} & \multicolumn{1}{r}{\textbf{2.22e-03}} & \multicolumn{1}{r}{\textbf{1.29e+03}} & \multicolumn{1}{r}{\textbf{2.87e-03}} & \multicolumn{1}{r}{\textbf{1.86e-03}} & \multicolumn{1}{r|}{\textbf{2.22e-06}} & \multicolumn{1}{r}{4.20e-03} & \multicolumn{1}{r}{7.12e+01} & \multicolumn{1}{r}{3.51e-03} & \multicolumn{1}{r}{4.15e-03} & \multicolumn{1}{r}{4.93e-05} \\
 & \hspace{0.1em}CFI-c+FKL-c+F  && \multicolumn{1}{r}{\textbf{4.79e-06}} & \multicolumn{1}{r}{\textbf{7.51e+05}} & \multicolumn{1}{r}{\textbf{3.16e-03}} & \multicolumn{1}{r}{\textbf{3.92e-03}} & \multicolumn{1}{r|}{\textbf{4.21e-09}} & \multicolumn{1}{r}{\textbf{\emph{2.22e-03}}} & \multicolumn{1}{r}{\textbf{\emph{6.54e+02}}} & \multicolumn{1}{r}{\textbf{\emph{2.22e-03}}} & \multicolumn{1}{r}{\textbf{\emph{2.46e-03}}} & \multicolumn{1}{r|}{\textbf{\emph{3.39e-06}}} & \multicolumn{1}{r}{\textbf{\emph{4.20e-03}}} & \multicolumn{1}{r}{\textbf{\emph{2.07e+02}}} & \multicolumn{1}{r}{\textbf{\emph{2.80e-03}}} & \multicolumn{1}{r}{\textbf{\emph{3.31e-03}}} & \multicolumn{1}{r}{\textbf{\emph{1.35e-05}}} \\
\multicolumn{3}{c|}{\hspace{-0.2cm}\emph{\#seeds used}} & \multicolumn{5}{c|}{4/20} & \multicolumn{5}{c|}{20/20} & \multicolumn{5}{c}{20/20} \\

 \midrule
  & \hspace{0.1em}Ref. results${}^{(*)}$   &&
    \multicolumn{1}{r}{--} & \multicolumn{1}{r}{--}  & \multicolumn{1}{r}{1.10e-01} &
    \multicolumn{1}{r}{--} & \multicolumn{1}{r|}{--} &
    \multicolumn{1}{r}{2.23e-03}  & \multicolumn{1}{r}{--} & \multicolumn{1}{r}{--} & \multicolumn{1}{r}{1.10e-01} &\multicolumn{1}{r|}{--}&
    \multicolumn{1}{r}{4.20e-03} & \multicolumn{1}{r}{--} & \multicolumn{1}{r}{--} & \multicolumn{1}{r}{6.00e-02} & \multicolumn{1}{r}{--}\\
   \bottomrule
\end{tabular}%
}%
}
\end{table*}

%% file: tables/table_new_k6_g5_g6_g7_avg_full_v2.tex
\begin{table*}[t!]
\centering
\scriptsize
\caption{Estimating $\Prob(g_i(\bfX)>\gamma), i=5,6,7, \bfX\sim\mathcal{N}(\bfmu^{(A)},\bfSigma^{(A)})$, $K=6$,
$R=10^5$.\\
Ref. results${}^{(*)}$: results from \cite{gao2023rare} for $g_5,g_7$ and \cite{papaioannou2016sequential} for $g_6$.
}

\label{tab:res_k6_g5g6g7_full}
\setlength{\tabcolsep}{2pt}
\rowcolors{3}{gray!11}{white}
\makebox[\textwidth][c]{%
\resizebox{1.15\textwidth}{!}{%
\begin{tabular}{c l c |
*{5}{S[scientific-notation=true,table-format=1.3e2,table-number-alignment=right,table-text-alignment=right]}
*{5}{S[scientific-notation=true,table-format=1.3e2,table-number-alignment=right,table-text-alignment=right]} |
*{5}{S[scientific-notation=true,table-format=1.3e2,table-number-alignment=right,table-text-alignment=right]}}
\toprule
 & & & \multicolumn{5}{c}{$g_5$ \cite{gao2023rare}} & \multicolumn{5}{c}{$g_6$ \cite{papaioannou2016sequential}} & \multicolumn{5}{c}{$g_7$ (ResNet) \cite{gao2023rare}} \\
 & \multicolumn{1}{r}{dimension $d=$} & & \multicolumn{5}{c}{40} & \multicolumn{5}{c}{100} & \multicolumn{5}{c}{62} \\
 & \multicolumn{1}{r}{$I_{\rm true}=$} &  & \multicolumn{5}{c}{3.1500e-05} & \multicolumn{5}{c}{2.3262e-04} & \multicolumn{5}{c}{6.00e-05} \\
\midrule
 & \multicolumn{1}{r}{} & \multicolumn{1}{c|}{\textbf{FG}}
& \multicolumn{1}{c}{$\hat{I}_R$}
& \multicolumn{1}{c}{Var red.}
& \multicolumn{1}{c}{log-err}
& \multicolumn{1}{c}{CoV}
& \multicolumn{1}{c|}{$\frac{\textrm{log-err}}{\textrm{Var red.}}$}
& \multicolumn{1}{c}{$\hat{I}_R$}
& \multicolumn{1}{c}{Var red.}
& \multicolumn{1}{c}{log-err}
& \multicolumn{1}{c}{CoV}
& \multicolumn{1}{c|}{$\frac{\textrm{log-err}}{\textrm{Var red.}}$}
& \multicolumn{1}{c}{$\hat{I}_R$}
& \multicolumn{1}{c}{Var red.}
& \multicolumn{1}{c}{log-err}
& \multicolumn{1}{c}{CoV}
& \multicolumn{1}{c}{$\frac{\textrm{log-err}}{\textrm{Var red.}}$}\\

 \midrule

& \hspace{0.1em}CMC  && \multicolumn{1}{r}{3.30e-05} & \multicolumn{1}{r}{1} & \multicolumn{1}{r}{4.59e-01} & \multicolumn{1}{r}{1.17e-01} & \multicolumn{1}{r|}{4.59e-01} & \multicolumn{1}{r}{2.27e-04} & \multicolumn{1}{r}{1} & \multicolumn{1}{r}{1.67e-01} & \multicolumn{1}{r}{1.80e-01} & \multicolumn{1}{r|}{1.67e-01} & \multicolumn{1}{r}{5.00e-05} & \multicolumn{1}{r}{1} & \multicolumn{1}{r}{3.35e-01} & \multicolumn{1}{r}{1.65e-02} & \multicolumn{1}{r}{3.35e-01} \\

\multicolumn{3}{l|}{\emph{Baselines (no CFI)}} & \multicolumn{5}{c|}{} & \multicolumn{5}{c|}{} & \multicolumn{5}{c}{} \\

 & \hspace{0.1em}fRKL-u  && \multicolumn{1}{r}{3.63e-05} & \multicolumn{1}{r}{1.25e+00} & \multicolumn{1}{r}{2.31e-01} & \multicolumn{1}{r}{2.41e-01} & \multicolumn{1}{r|}{1.84e-01} & \multicolumn{1}{r}{\textbf{2.33e-04}} & \multicolumn{1}{r}{\textbf{\emph{5.87e+03}}} & \multicolumn{1}{r}{\textbf{2.12e-03}} & \multicolumn{1}{r}{\textbf{2.67e-03}} & \multicolumn{1}{r|}{\textbf{3.61e-07}} & \multicolumn{1}{r}{\emph{5.68e-05}} & \multicolumn{1}{r}{\emph{2.06e+03}} & \multicolumn{1}{r}{\emph{5.53e-02}} & \multicolumn{1}{r}{\emph{8.67e-03}} & \multicolumn{1}{r}{\emph{2.68e-05}} \\
 & \hspace{0.1em}fFKL-u  && \multicolumn{1}{r}{\textbf{3.39e-05}} & \multicolumn{1}{r}{\textbf{4.03e+03}} & \multicolumn{1}{r}{\textbf{7.37e-02}} & \multicolumn{1}{r}{\textbf{8.64e-03}} & \multicolumn{1}{r|}{\textbf{1.83e-05}} & \multicolumn{1}{r}{0} & \multicolumn{1}{r}{\xmark} & \multicolumn{1}{r}{\xmark} & \multicolumn{1}{r}{\xmark} & \multicolumn{1}{r|}{\xmark} & \multicolumn{1}{r}{0} & \multicolumn{1}{r}{\xmark} & \multicolumn{1}{r}{\xmark} & \multicolumn{1}{r}{\xmark} & \multicolumn{1}{r}{\xmark} \\

 \multicolumn{3}{l|}{\emph{Ablation}} & \multicolumn{5}{c|}{} & \multicolumn{5}{c|}{} & \multicolumn{5}{c}{} \\
 & \hspace{0.1em}CFI-u  && \multicolumn{1}{r}{5.24e-07} & \multicolumn{1}{r}{1.96e+02} & \multicolumn{1}{r}{5.33e+00} & \multicolumn{1}{r}{4.54e-01} & \multicolumn{1}{r|}{2.72e-02} & \multicolumn{1}{r}{0} & \multicolumn{1}{r}{\xmark} & \multicolumn{1}{r}{\xmark} & \multicolumn{1}{r}{\xmark} & \multicolumn{1}{r|}{\xmark} & \multicolumn{1}{r}{0} & \multicolumn{1}{r}{\xmark} & \multicolumn{1}{r}{\xmark} & \multicolumn{1}{r}{\xmark} & \multicolumn{1}{r}{\xmark} \\
 & \hspace{0.1em}fRKL-c & \textcolor{green!60!black}{\cmark} & \multicolumn{1}{r}{\emph{3.39e-05}} & \multicolumn{1}{r}{\emph{2.97e+02}} & \multicolumn{1}{r}{\emph{7.31e-02}} & \multicolumn{1}{r}{\emph{3.19e-02}} & \multicolumn{1}{r|}{\emph{2.46e-04}} & \multicolumn{1}{r}{\emph{2.32e-04}} & \multicolumn{1}{r}{1.14e+03} & \multicolumn{1}{r}{\emph{5.90e-03}} & \multicolumn{1}{r}{\emph{6.08e-03}} & \multicolumn{1}{r|}{\emph{5.20e-06}} & \multicolumn{1}{r}{5.69e-05} & \multicolumn{1}{r}{1.96e+03} & \multicolumn{1}{r}{5.26e-02} & \multicolumn{1}{r}{8.88e-03} & \multicolumn{1}{r}{2.69e-05} \\
 & \hspace{0.1em}fFKL-c & \textcolor{green!60!black}{\cmark} & \multicolumn{1}{r}{3.39e-05} & \multicolumn{1}{r}{2.82e+02} & \multicolumn{1}{r}{7.32e-02} & \multicolumn{1}{r}{3.27e-02} & \multicolumn{1}{r|}{2.60e-04} & \multicolumn{1}{r}{7.50e-06} & \multicolumn{1}{r}{\textbf{1.16e+04}} & \multicolumn{1}{r}{4.17e+00} & \multicolumn{1}{r}{5.76e-01} & \multicolumn{1}{r|}{3.59e-04} & \multicolumn{1}{r}{4.81e-05} & \multicolumn{1}{r}{1.70e+01} & \multicolumn{1}{r}{5.33e-01} & \multicolumn{1}{r}{2.54e-01} & \multicolumn{1}{r}{3.13e-02} \\

\multicolumn{3}{l|}{\emph{Our: CFI-based}} & \multicolumn{5}{c|}{} & \multicolumn{5}{c|}{} & \multicolumn{5}{c}{} \\
 & \hspace{0.1em}CFI-c & \textcolor{green!60!black}{\cmark} & \multicolumn{1}{r}{3.27e-05} & \multicolumn{1}{r}{4.24e+01} & \multicolumn{1}{r}{8.07e-02} & \multicolumn{1}{r}{8.89e-02} & \multicolumn{1}{r|}{1.90e-03} & \multicolumn{1}{r}{2.34e-04} & \multicolumn{1}{r}{8.33e+00} & \multicolumn{1}{r}{5.09e-02} & \multicolumn{1}{r}{5.81e-02} & \multicolumn{1}{r|}{6.11e-03} & \multicolumn{1}{r}{5.72e-05} & \multicolumn{1}{r}{2.28e+02} & \multicolumn{1}{r}{5.00e-02} & \multicolumn{1}{r}{2.42e-02} & \multicolumn{1}{r}{2.19e-04} \\
 & \hspace{0.1em}CFI-c+RKL-c & \textcolor{green!60!black}{\cmark} & \multicolumn{1}{r}{3.39e-05} & \multicolumn{1}{r}{1.49e+02} & \multicolumn{1}{r}{7.28e-02} & \multicolumn{1}{r}{4.30e-02} & \multicolumn{1}{r|}{4.88e-04} & \multicolumn{1}{r}{2.32e-04} & \multicolumn{1}{r}{1.12e+03} & \multicolumn{1}{r}{7.02e-03} & \multicolumn{1}{r}{6.12e-03} & \multicolumn{1}{r|}{6.26e-06} & \multicolumn{1}{r}{\textbf{\emph{5.68e-05}}} & \multicolumn{1}{r}{\textbf{\emph{3.35e+03}}} & \multicolumn{1}{r}{\textbf{\emph{5.44e-02}}} & \multicolumn{1}{r}{\textbf{\emph{6.80e-03}}} & \multicolumn{1}{r}{\textbf{\emph{1.62e-05}}} \\
 & \hspace{0.1em}CFI-c+FKL-c & \textcolor{green!60!black}{\cmark} & \multicolumn{1}{r}{3.40e-05} & \multicolumn{1}{r}{7.60e+00} & \multicolumn{1}{r}{1.48e-01} & \multicolumn{1}{r}{1.80e-01} & \multicolumn{1}{r|}{1.94e-02} & \multicolumn{1}{r}{2.34e-04} & \multicolumn{1}{r}{4.14e-01} & \multicolumn{1}{r}{1.77e-01} & \multicolumn{1}{r}{1.88e-01} & \multicolumn{1}{r|}{4.29e-01} & \multicolumn{1}{r}{5.70e-05} & \multicolumn{1}{r}{3.26e+02} & \multicolumn{1}{r}{5.09e-02} & \multicolumn{1}{r}{2.14e-02} & \multicolumn{1}{r}{1.56e-04} \\
 & \hspace{0.1em}CFI-c+RKL-c+F  && \multicolumn{1}{r}{\textbf{\emph{3.36e-05}}} & \multicolumn{1}{r}{\textbf{\emph{3.15e+02}}} & \multicolumn{1}{r}{\textbf{\emph{6.45e-02}}} & \multicolumn{1}{r}{\textbf{\emph{2.99e-02}}} & \multicolumn{1}{r|}{\textbf{\emph{2.05e-04}}} & \multicolumn{1}{r}{\textbf{\emph{2.33e-04}}} & \multicolumn{1}{r}{\emph{1.42e+03}} & \multicolumn{1}{r}{\textbf{\emph{5.94e-03}}} & \multicolumn{1}{r}{\textbf{\emph{5.43e-03}}} & \multicolumn{1}{r|}{\textbf{\emph{4.19e-06}}} & \multicolumn{1}{r}{\textbf{5.68e-05}} & \multicolumn{1}{r}{\textbf{3.97e+03}} & \multicolumn{1}{r}{\textbf{5.49e-02}} & \multicolumn{1}{r}{\textbf{6.24e-03}} & \multicolumn{1}{r}{\textbf{1.38e-05}} \\
 & \hspace{0.1em}CFI-c+FKL-c+F  && \multicolumn{1}{r}{3.22e-05} & \multicolumn{1}{r}{1.32e+01} & \multicolumn{1}{r}{9.17e-02} & \multicolumn{1}{r}{1.17e-01} & \multicolumn{1}{r|}{6.96e-03} & \multicolumn{1}{r}{2.28e-04} & \multicolumn{1}{r}{5.48e-01} & \multicolumn{1}{r}{1.79e-01} & \multicolumn{1}{r}{1.80e-01} & \multicolumn{1}{r|}{3.26e-01} & \multicolumn{1}{r}{5.65e-05} & \multicolumn{1}{r}{5.72e+02} & \multicolumn{1}{r}{6.11e-02} & \multicolumn{1}{r}{1.65e-02} & \multicolumn{1}{r}{1.07e-04} \\
\multicolumn{3}{l|}{\emph{\#seeds used}} & \multicolumn{5}{c|}{19/20} & \multicolumn{5}{c|}{20/20} & \multicolumn{5}{c}{20/20} \\

 \midrule

\bottomrule
  & \hspace{0.5em}Ref. results${}^{(*)}$   &&
  \multicolumn{1}{r}{--} & \multicolumn{1}{r}{--} & \multicolumn{1}{r}{3.80e-01} & \multicolumn{1}{r}{--} & \multicolumn{1}{r|}{--} &
  \multicolumn{1}{r}{2.33e-04}  & \multicolumn{1}{r}{--} & \multicolumn{1}{r}{--} & \multicolumn{1}{r}{2.80e-01} & \multicolumn{1}{r}{--} &
  \multicolumn{1}{r}{--} & \multicolumn{1}{r}{--} & \multicolumn{1}{r}{6.20e-01} & \multicolumn{1}{r}{--} & \multicolumn{1}{r}{--}\\
 \bottomrule
\end{tabular}%
}%
}
\end{table*}

%% file: tables/table_new_k6_s10_avg_full_v2.tex
\begin{table*}[t!]
\centering
\scriptsize
\caption{Estimating $\Prob(g_{10}(\bfX)>\gamma), \bfX\sim\mathcal{N}(\bfmu^{(B)},\bfSigma^{(B)})$, $K=6$,
$R=10^5$. \\ CEVE${}^{(*)}$: results from \cite{asmussen2011efficient}.
}
\label{tab:res_k6_s10_full}
\setlength{\tabcolsep}{2pt}
\rowcolors{3}{gray!11}{white}
\makebox[\textwidth][c]{%
\resizebox{1.15\textwidth}{!}{%
\begin{tabular}{c l c |
*{5}{S[scientific-notation=true,table-format=1.3e2,table-number-alignment=right,table-text-alignment=right]}
*{5}{S[scientific-notation=true,table-format=1.3e2,table-number-alignment=right,table-text-alignment=right]} |
*{5}{S[scientific-notation=true,table-format=1.3e2,table-number-alignment=right,table-text-alignment=right]}}
\toprule
 & \multicolumn{1}{r}{$\gamma=$} & & \multicolumn{5}{c}{20,000} & \multicolumn{5}{c}{40,000} & \multicolumn{5}{c}{500,000} \\
 & \multicolumn{1}{r}{$I_{\rm true}=$} &  & \multicolumn{5}{c}{1.03526e-03} & \multicolumn{5}{c}{4.683e-04} & \multicolumn{5}{c}{1.77400e-05} \\

 \midrule
 & \multicolumn{1}{r}{} & \multicolumn{1}{c|}{\textbf{FG}}
& \multicolumn{1}{c}{$\hat{I}_R$}
& \multicolumn{1}{c}{Var red.}
& \multicolumn{1}{c}{log-err}
& \multicolumn{1}{c}{CoV}
& \multicolumn{1}{c|}{$\frac{\textrm{log-err}}{\textrm{Var red.}}$}
& \multicolumn{1}{c}{$\hat{I}_R$}
& \multicolumn{1}{c}{Var red.}
& \multicolumn{1}{c}{log-err}
& \multicolumn{1}{c}{CoV}
& \multicolumn{1}{c|}{$\frac{\textrm{log-err}}{\textrm{Var red.}}$}
& \multicolumn{1}{c}{$\hat{I}_R$}
& \multicolumn{1}{c}{Var red.}
& \multicolumn{1}{c}{log-err}
& \multicolumn{1}{c}{CoV}
& \multicolumn{1}{c}{$\frac{\textrm{log-err}}{\textrm{Var red.}}$}\\

  \midrule
 & \hspace{0.1em}CMC  && \multicolumn{1}{r}{1.03e-03} & \multicolumn{1}{r}{1} & \multicolumn{1}{r}{7.22e-02} & \multicolumn{1}{r}{6.97e-03} & \multicolumn{1}{r|}{7.22e-02} & \multicolumn{1}{r}{4.69e-04} & \multicolumn{1}{r}{1} & \multicolumn{1}{r}{1.12e-01} & \multicolumn{1}{r}{7.98e-03} & \multicolumn{1}{r|}{1.12e-01} & \multicolumn{1}{r}{2.30e-05} & \multicolumn{1}{r}{1} & \multicolumn{1}{r}{4.28e-01} & \multicolumn{1}{r}{1.36e-02} & \multicolumn{1}{r}{4.28e-01} \\

 \multicolumn{3}{l|}{\emph{Ablation}} & \multicolumn{5}{c|}{} & \multicolumn{5}{c|}{} & \multicolumn{5}{c}{} \\

 & \hspace{0.1em}CFI-u  && \multicolumn{1}{r}{4.37e-05} & \multicolumn{1}{r}{1.60e+01} & \multicolumn{1}{r}{3.26e+00} & \multicolumn{1}{r}{3.65e-01} & \multicolumn{1}{r|}{2.04e-01} & \multicolumn{1}{r}{3.40e-06} & \multicolumn{1}{r}{1.01e+03} & \multicolumn{1}{r}{5.10e+00} & \multicolumn{1}{r}{3.30e-01} & \multicolumn{1}{r|}{5.05e-03} & \multicolumn{1}{r}{0} & \multicolumn{1}{r}{\xmark} & \multicolumn{1}{r}{\xmark} & \multicolumn{1}{r}{\xmark} & \multicolumn{1}{r}{\xmark} \\
\multicolumn{3}{l|}{\emph{Our: CFI-based}} & \multicolumn{5}{c|}{} & \multicolumn{5}{c|}{} & \multicolumn{5}{c}{} \\
 & \hspace{0.1em}fRKL-c & \textcolor{green!60!black}{\cmark} & \multicolumn{1}{r}{1.02e-03} & \multicolumn{1}{r}{2.48e+01} & \multicolumn{1}{r}{1.66e-02} & \multicolumn{1}{r}{2.00e-02} & \multicolumn{1}{r|}{6.70e-04} & \multicolumn{1}{r}{4.59e-04} & \multicolumn{1}{r}{5.17e+01} & \multicolumn{1}{r}{2.46e-02} & \multicolumn{1}{r}{2.07e-02} & \multicolumn{1}{r|}{4.76e-04} & \multicolumn{1}{r}{1.78e-05} & \multicolumn{1}{r}{7.70e+02} & \multicolumn{1}{r}{2.97e-02} & \multicolumn{1}{r}{3.24e-02} & \multicolumn{1}{r}{3.86e-05} \\
 & \hspace{0.1em}fFKL-c & \textcolor{green!60!black}{\cmark} & \multicolumn{1}{r}{\textbf{1.02e-03}} & \multicolumn{1}{r}{\textbf{2.29e+02}} & \multicolumn{1}{r}{\textbf{1.04e-02}} & \multicolumn{1}{r}{\textbf{6.55e-03}} & \multicolumn{1}{r|}{\textbf{4.53e-05}} & \multicolumn{1}{r}{\textbf{\emph{4.62e-04}}} & \multicolumn{1}{r}{\textbf{\emph{4.20e+02}}} & \multicolumn{1}{r}{\textbf{\emph{1.36e-02}}} & \multicolumn{1}{r}{\textbf{\emph{7.23e-03}}} & \multicolumn{1}{r|}{\textbf{\emph{3.24e-05}}} & \multicolumn{1}{r}{\textbf{1.79e-05}} & \multicolumn{1}{r}{\textbf{7.71e+03}} & \multicolumn{1}{r}{\textbf{9.87e-03}} & \multicolumn{1}{r}{\textbf{1.01e-02}} & \multicolumn{1}{r}{\textbf{1.28e-06}} \\

 & \hspace{0.1em}CFI-c & \textcolor{green!60!black}{\cmark} & \multicolumn{1}{r}{1.02e-03} & \multicolumn{1}{r}{3.51e+01} & \multicolumn{1}{r}{2.06e-02} & \multicolumn{1}{r}{1.59e-02} & \multicolumn{1}{r|}{5.85e-04} & \multicolumn{1}{r}{4.62e-04} & \multicolumn{1}{r}{8.88e+01} & \multicolumn{1}{r}{1.69e-02} & \multicolumn{1}{r}{1.55e-02} & \multicolumn{1}{r|}{1.90e-04} & \multicolumn{1}{r}{\textbf{\emph{1.80e-05}}} & \multicolumn{1}{r}{\textbf{\emph{3.09e+03}}} & \multicolumn{1}{r}{\textbf{\emph{1.62e-02}}} & \multicolumn{1}{r}{\textbf{\emph{1.56e-02}}} & \multicolumn{1}{r}{\textbf{\emph{5.24e-06}}} \\
 & \hspace{0.1em}CFI-c+RKL-c & \textcolor{green!60!black}{\cmark} & \multicolumn{1}{r}{9.98e-04} & \multicolumn{1}{r}{1.09e+00} & \multicolumn{1}{r}{7.10e-02} & \multicolumn{1}{r}{4.50e-02} & \multicolumn{1}{r|}{6.53e-02} & \multicolumn{1}{r}{4.67e-04} & \multicolumn{1}{r}{9.51e+00} & \multicolumn{1}{r}{3.31e-02} & \multicolumn{1}{r}{3.88e-02} & \multicolumn{1}{r|}{3.48e-03} & \multicolumn{1}{r}{1.80e-05} & \multicolumn{1}{r}{9.39e+02} & \multicolumn{1}{r}{2.65e-02} & \multicolumn{1}{r}{2.75e-02} & \multicolumn{1}{r}{2.83e-05} \\
 & \hspace{0.1em}CFI-c+FKL-c & \textcolor{green!60!black}{\cmark} & \multicolumn{1}{r}{\textbf{\emph{1.02e-03}}} & \multicolumn{1}{r}{\textbf{\emph{1.52e+02}}} & \multicolumn{1}{r}{\textbf{\emph{1.31e-02}}} & \multicolumn{1}{r}{\textbf{\emph{8.03e-03}}} & \multicolumn{1}{r|}{\textbf{\emph{8.64e-05}}} & \multicolumn{1}{r}{\textbf{4.63e-04}} & \multicolumn{1}{r}{\textbf{4.08e+02}} & \multicolumn{1}{r}{\textbf{1.12e-02}} & \multicolumn{1}{r}{\textbf{7.30e-03}} & \multicolumn{1}{r|}{\textbf{2.75e-05}} & \multicolumn{1}{r}{1.81e-05} & \multicolumn{1}{r}{6.96e+02} & \multicolumn{1}{r}{2.03e-02} & \multicolumn{1}{r}{1.83e-02} & \multicolumn{1}{r}{2.92e-05} \\
 & \hspace{0.1em}CFI-c+RKL-c+F  && \multicolumn{1}{r}{9.90e-04} & \multicolumn{1}{r}{2.70e+01} & \multicolumn{1}{r}{4.49e-02} & \multicolumn{1}{r}{1.90e-02} & \multicolumn{1}{r|}{1.66e-03} & \multicolumn{1}{r}{4.56e-04} & \multicolumn{1}{r}{3.84e+00} & \multicolumn{1}{r}{6.28e-02} & \multicolumn{1}{r}{3.81e-02} & \multicolumn{1}{r|}{1.63e-02} & \multicolumn{1}{r}{\emph{1.80e-05}} & \multicolumn{1}{r}{\emph{2.81e+03}} & \multicolumn{1}{r}{\emph{1.47e-02}} & \multicolumn{1}{r}{\emph{1.25e-02}} & \multicolumn{1}{r}{\emph{5.24e-06}} \\
 & \hspace{0.1em}CFI-c+FKL-c+F  && \multicolumn{1}{r}{\emph{1.02e-03}} & \multicolumn{1}{r}{\emph{1.61e+02}} & \multicolumn{1}{r}{\emph{1.40e-02}} & \multicolumn{1}{r}{\emph{6.97e-03}} & \multicolumn{1}{r|}{\emph{8.74e-05}} & \multicolumn{1}{r}{\emph{4.63e-04}} & \multicolumn{1}{r}{\emph{2.47e+02}} & \multicolumn{1}{r}{\emph{1.22e-02}} & \multicolumn{1}{r}{\emph{7.98e-03}} & \multicolumn{1}{r|}{\emph{4.92e-05}} & \multicolumn{1}{r}{1.79e-05} & \multicolumn{1}{r}{2.35e+03} & \multicolumn{1}{r}{1.37e-02} & \multicolumn{1}{r}{1.36e-02} & \multicolumn{1}{r}{5.83e-06} \\
\multicolumn{3}{c|}{\hspace{-0.2cm}\emph{\#seeds used}} & \multicolumn{5}{c|}{20/20} & \multicolumn{5}{c|}{20/20} & \multicolumn{5}{c}{18/20} \\

 \bottomrule
  & \hspace{0.5em}CEVE${}^{(*)}$   &&
  \multicolumn{1}{r}{1.02e-03} & \multicolumn{1}{r}{--} & \multicolumn{1}{r}{--} & \multicolumn{1}{r}{1.28e-01} & \multicolumn{1}{r|}{--} &
  \multicolumn{1}{r}{4.63e-04}  & \multicolumn{1}{r}{--} & \multicolumn{1}{r}{--} & \multicolumn{1}{r}{8.60e-02} & \multicolumn{1}{r}{--}  &
  \multicolumn{1}{r}{1.79e-05} & \multicolumn{1}{r}{--} & \multicolumn{1}{r}{--} & \multicolumn{1}{r}{4.46e-02}& \multicolumn{1}{r}{--}  \\
  \bottomrule
\end{tabular}%
}%
}
\end{table*}

%% file: tables/tab_defensive_lambda_sweep.tex
\begin{table}[H]
\centering
\small
\caption{Sensitivity of the Hesterberg-style defensive baseline on $g_1$
with $K=16$, $R=10^5$, and 20 independent evaluation runs.}
\label{tab:defensive_lambda_sweep}
\setlength{\tabcolsep}{6pt}
\begin{tabular}{lrrrr}
\toprule
Method & Estimate & Var. reduction & log-error & CoV \\
\midrule
fRKL-d, $\lambda=0.1$ & 4.773e-06 & 1.73e+03 & 5.06e-02 & 7.53e-02 \\
fRKL-d, $\lambda=0.2$ & 4.810e-06 & 1.71e+03 & 6.52e-02 & 7.55e-02 \\
fRKL-d, $\lambda=0.3$ & 4.866e-06 & 1.07e+03 & 8.54e-02 & 9.33e-02 \\
fRKL-d, $\lambda=0.4$ & 4.828e-06 & 1.03e+03 & 6.24e-02 & 9.49e-02 \\
fRKL-d, $\lambda=0.5$ & 4.955e-06 & 8.29e+02 & 8.66e-02 & 9.76e-02 \\
\bottomrule
\end{tabular}
\end{table}

%% file: 2026_NeurIPS_flowce_appendix_times.tex
\subsection{Indicative running times}
\label{app:running_times}

We report indicative training times for $g_1$ ($d=2$), $g_6$ ($d=100$), and
$g_7$ ($d=62$, ResNet). The measurements use our shared PyTorch
implementation, but wall-clock time depends on implementation details such as
batching, parallelization, and hardware. We did not optimize the code for
wall-clock time and used largely shared, conservative epoch counts across
benchmarks.

\begin{table}[h!]
\centering
\small
\caption{Indicative training times in seconds. Panel (a) reports total training
time; panel (b) reports separately measured operations on $g_7$.}
\label{tab:runtime_comparison}
\begin{tabular}{lrr}
\toprule
\multicolumn{3}{c}{(a) Total training time} \\
\midrule
Method & $g_1$ ($d=2$) & $g_6$ ($d=100$) \\
\midrule
fFKL-c      & $\sim 3000$ & $\sim 1570$ \\
fRKL-c      & $\sim 55$   & $\sim 340$ \\
CFI-c       & $\sim 5$    & $\sim 8$ \\
CFI-c+FKL-c & $\sim 60$   & $\sim 234$ \\
\bottomrule
\end{tabular}

\vspace{0.7em}

\begin{tabular}{lr}
\toprule
\multicolumn{2}{c}{(b) Measured operations on $g_7$ ($d=62$, ResNet)} \\
\midrule
Operation & Time \\
\midrule
Full fRKL-c training          & $\sim 130$ \\
CFI-c                         & $\sim 76$ \\
FKL-c refinement after CFI-c & $\sim 13$ \\
RKL-c refinement after CFI-c & $\sim 52$ \\
Flow after CFI-c+FKL-c       & $\sim 38$ \\
Flow after CFI-c+RKL-c       & $\sim 42$ \\
\bottomrule
\end{tabular}
\end{table}

CFI is inexpensive because each intermediate level requires only an unweighted
GMM fit to the current elite sample. Full adaptive KL/CE baselines instead
repeat proposal optimization, including many gradient-based epochs, at every
level. For example, on $g_6$, CFI-c+FKL-c takes approximately $234$ seconds,
compared with $1570$ seconds for fFKL-c, while CFI-c alone takes approximately
$8$ seconds.

On $g_7$, flow refinement after CFI-c+RKL-c adds approximately $42$ seconds
and reduces CoV from $6.80\times10^{-3}$ to $6.24\times10^{-3}$. After
CFI-c+FKL-c, it adds approximately $38$ seconds and reduces CoV from
$2.14\times10^{-2}$ to $1.65\times10^{-2}$. These gains are useful but
moderate relative to the extra training time, which motivates treating flow
refinement as optional. The final KL and flow stages could likely use fewer
epochs than the conservative shared settings used here.

%% file: 2026_NeurIPS_flowce_appendix_scalability.tex
\section{Scalability: higher-dimensional ResNet experiment}
\label{app:resnet174}

To further assess scalability, we consider a higher-dimensional variant of the
ResNet-based rare-event benchmark $g_7$. The mechanism discussed in the main
text becomes especially relevant in higher dimensions: in rare-event regimes,
the conditional distribution $f(\x\mid g(\x)>\gamma)$ may concentrate near
lower-dimensional regions or along a small number of dominant directions. This
can lead to increasingly ill-conditioned covariance structures, making
unconstrained KL-based optimization more prone to covariance collapse and
degenerate proposals.

We therefore repeat the ResNet experiment with a finer grouping of network
weights, increasing the perturbation dimension from $d=62$ to $d=174$. The setup
is otherwise analogous to the $g_7$ benchmark described in
Appendix~\ref{app:benchmark_settings}. Results are averaged over 20 independent
seeds and use a GMM proposal with $K=16$ components.

\begin{table}[h!]
\centering
\small
\caption{Scalability experiment on the high-dimensional ResNet benchmark
($d=174$, $K=16$). Results are averaged over 20 independent seeds.}
\label{tab:resnet174_scalability}
\begin{tabular}{lccc}
\toprule
Method & log-err & CoV & log-err/var\_red \\
\midrule
CMC             & 3.89e-01 & 1.45e-01 & 3.89e-01 \\
CFI-c           & 1.42e-01 & 1.34e-01 & 3.48e-02 \\
CFI-c+RKL-c     & 1.57e-02 & 6.79e-03 & 5.37e-06 \\
CFI-c+FKL-c     & 1.76e-01 & 1.45e-01 & 9.89e-02 \\
\bottomrule
\end{tabular}
\end{table}

All CFI-based methods outperform the CMC baseline on this higher-dimensional
problem. In particular, CFI-c+RKL-c achieves the strongest performance, reducing
the composite metric log-err/var\_red by approximately four orders of magnitude
relative to CMC. These results are consistent with the lower-dimensional
ResNet experiment and indicate that the proposed constrained, coverage-first
pipeline remains effective as the perturbation dimension increases.

%% file: 2026_NeurIPS_flowce_appendix_sensitivity.tex
\section{Hyperparameters and training details}
\label{app:hyperparameters}

All experiments are run using a largely shared set of hyperparameters.
Thanks to the automatic selection of the smoothing parameter $\tau$ and
$\beta$-annealing for the covariance constraint, no manual tuning is required across benchmarks, and the same configuration is used throughout unless stated otherwise.

\paragraph{Common hyperparameters.}
For all GMM-based stages (CFI, constrained FKL, constrained RKL), we use:
\begin{itemize}
  \item initial sample size for model initialization: $10{,}000$,
  \item surrogate batch size: $M=25{,}000$,
  \item number of gradient steps: $500$,
  \item learning rate: $5\times10^{-3}$,
  \item smoothing factor: $\alpha=0.75$,
  \item maximum number of outer iterations: $100$,
  \item elite fraction (when applicable): $\rho=0.2$.
\end{itemize}

\paragraph{Sensitivity to the smoothing parameter $\tau$.}
The reverse-KL objectives are sensitive to the selection of smoothing parameter $\tau$. This sensitivity is illustrated in Fig.~\ref{fig:tau_ablation} for the $g_1$ benchmark with $K=6$ and $\bfSigma=\bfSigma^{(A)}$, where variance reduction relative to CMC is reported as a function of~$\tau$.

\begin{figure}[t]
  \centering
  \includegraphics[width=0.48\textwidth]{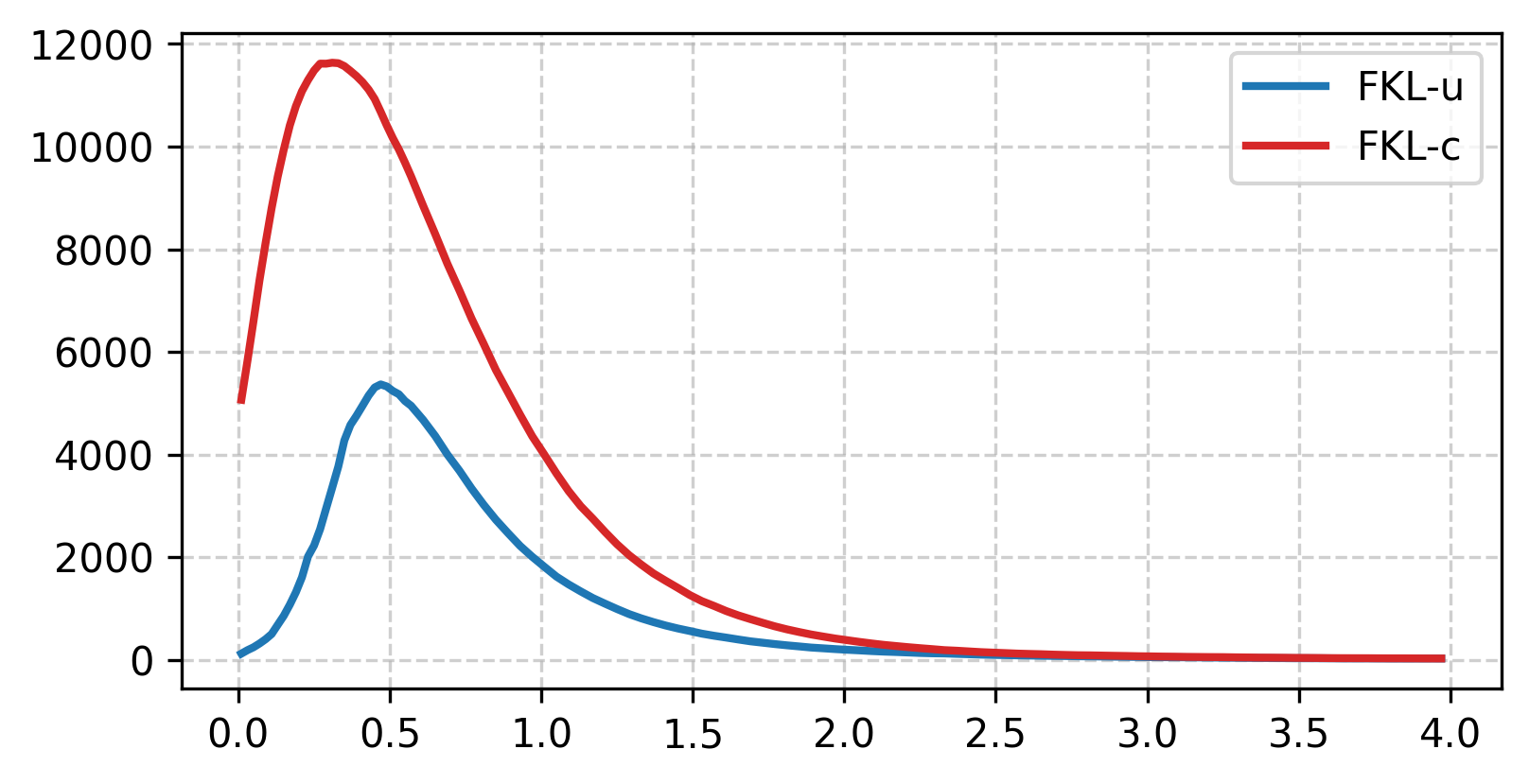}
  \caption{Variance reduction ($y$-axis) relative to CMC, as a function of the smoothing parameter $\tau$ ($x$-axis)
  for RKL-c (red) and RKL-u (blue).}
  \label{fig:tau_ablation}
\end{figure}

  As seen in the figure, good performance is attained only within a relatively narrow range of~$\tau$, while overly small or large values lead to a noticeable degradation. Our adaptive procedure stabilizes early training and automatically selects values of~$\tau$ within the effective range  without any problem-specific tuning. In all experiments reported in this paper, a single choice of~$c=0.6$ and bounds $[\tau_{\min},\tau_{\max}]=[0.02,0.8]$ was used for all benchmarks.

In addition, $\beta$-annealing gradually increases $\beta$ to avoid premature
variance collapse and recovers the covariance floor required for the guarantee
in the final constrained-GMM proposal.

\paragraph{Flow-specific hyperparameters.}
For the RealNVP refinement stage, we use the same training protocol across all
benchmarks, with the following settings:
\begin{itemize}
  \item initial sample size: $25{,}000$,
  \item surrogate batch size: $M=25{,}000$,
  \item number of gradient steps: $1{,}000$,
  \item learning rate: $10^{-5}$.
\end{itemize}

The only differences across benchmarks concern the architecture of the coupling
networks.
For low-dimensional problems ($g_1$, $g_2$, $g_3$, with $d=2$), we use $48$
coupling layers with hidden dimension $1024$.
For all higher-dimensional benchmarks, we use $32$ coupling layers with hidden dimension $32$. All other flow hyperparameters are kept fixed.

\paragraph{Parameter smoothing.}
After each refitting step in CFI, we smooth the newly fitted parameters
with their previous values. Let primes denote the newly fitted,
pre-smoothed quantities and let $\alpha\in(0,1]$ be the weight assigned
to the new fit.

Means and mixture weights are updated as
\[
\bfmu_k^{(t)}
=
\alpha\bfmu_k'
+
(1-\alpha)\bfmu_k^{(t-1)},
\qquad
\pi_k^{(t)}
=
\alpha\pi_k'
+
(1-\alpha)\pi_k^{(t-1)}.
\]
The mixture logits are then set to
$\ell_k^{(t)}=\log(\pi_k^{(t)}+\varepsilon)$,
with a small $\varepsilon>0$ for numerical stability.

For unconstrained GMMs, we smooth the covariance matrices and recompute
their Cholesky factors:
\[
\bfSigma_k^{(t)}
=
\alpha\bfSigma_k'
+
(1-\alpha)\bfSigma_k^{(t-1)},
\qquad
\bfL_k^{(t)}
=
\operatorname{Cholesky}(\bfSigma_k^{(t)}).
\]

For constrained GMMs, we smooth only the free covariance component.
Writing
$\bfS_k'=\bfL_k'\bfL_k'^\top$ and
$\bfS_k^{(t-1)}
=\bfL_k^{(t-1)}(\bfL_k^{(t-1)})^\top$,
we set
\[
\bfS_k^{(t)}
=
\alpha\bfS_k'
+
(1-\alpha)\bfS_k^{(t-1)},
\qquad
\bfL_k^{(t)}
=
\operatorname{Cholesky}(\bfS_k^{(t)}),
\]
and form
\[
\bfSigma_k^{(t)}
=
\bfL_k^{(t)}(\bfL_k^{(t)})^\top
+
\beta\bfSigma.
\]
This preserves the covariance floor at the current value of $\beta$.

\section{Hyperparameter sensitivity analysis}
\label{app:sensitivity}

This section studies the sensitivity of the proposed method to the main CFI
hyperparameters. The parameter $\rho$ controls the fraction of samples used in
the level update and therefore affects the exploration--exploitation trade-off:
smaller values use a more selective elite set, while larger values produce more
conservative level updates. The smoothing parameter $\alpha$ regularizes the
parameter updates: smaller values lead to slower, more stable updates, while
larger values make the proposal adapt more aggressively.

We evaluate sensitivity on the $g_5$ benchmark ($d=40$) using CFI-c+RKL-c. We
report the coefficient of variation (CoV) and the composite metric
$\text{log-err}/\text{var\_red}$ used in the main experiments.

\paragraph{Sensitivity to $\rho$.}
Table~\ref{tab:rho_sensitivity} varies $\rho$ while keeping $\alpha=0.75$ fixed.

\begin{table}[h]
\centering
\small
\caption{Sensitivity to $\rho$ for CFI-c+RKL-c on $g_5$ ($d=40$), with fixed $\alpha=0.75$.}
\label{tab:rho_sensitivity}
\begin{tabular}{ccc}
\toprule
$\rho$ & CoV & $\text{log-err}/\text{var\_red}$ \\
\midrule
0.05 & $4.28\cdot 10^{-2}$ & $6.40\cdot 10^{-4}$ \\
0.10 & $4.02\cdot 10^{-2}$ & $4.23\cdot 10^{-4}$ \\
0.25 & $4.30\cdot 10^{-2}$ & $4.88\cdot 10^{-4}$ \\
0.30 & $3.76\cdot 10^{-2}$ & $3.02\cdot 10^{-4}$ \\
0.50 & $3.81\cdot 10^{-2}$ & $4.10\cdot 10^{-4}$ \\
\bottomrule
\end{tabular}
\end{table}

The performance varies only mildly across a broad range of values. We use
$\rho=0.25$ as a representative default value: it is well within the stable
region and avoids relying on the best value from this particular ablation.

\paragraph{Sensitivity to $\alpha$.}
Table~\ref{tab:alpha_sensitivity} varies $\alpha$ while keeping $\rho=0.25$
fixed.

\begin{table}[h]
\centering
\small
\caption{Sensitivity to $\alpha$ for CFI-c+RKL-c on $g_5$ ($d=40$), with fixed $\rho=0.25$.}
\label{tab:alpha_sensitivity}
\begin{tabular}{ccc}
\toprule
$\alpha$ & CoV & $\text{log-err}/\text{var\_red}$ \\
\midrule
0.25 & $4.77\cdot 10^{-2}$ & $4.44\cdot 10^{-4}$ \\
0.50 & $4.62\cdot 10^{-2}$ & $4.90\cdot 10^{-4}$ \\
0.75 & $4.30\cdot 10^{-2}$ & $4.88\cdot 10^{-4}$ \\
0.90 & $3.86\cdot 10^{-2}$ & $3.95\cdot 10^{-4}$ \\
\bottomrule
\end{tabular}
\end{table}

Again, the results are stable across moderate values of $\alpha$. 
The value $\alpha=0.75$ used in our experiments lies in this stable region and provides a
balanced trade-off between update stability and adaptation speed. This choice is
also consistent with standard cross-entropy practice, where smoothing is commonly
used to stabilize parameter updates, with typical values around $\alpha=0.7$ and
recommended ranges such as $0.4\leq \alpha \leq 0.9$
reported in~\cite{rubinstein1999cross,rubinstein2004cross}.

\paragraph{Role of the sample size $M$.}
The surrogate sample size $M$ plays a different role from $\rho$ and $\alpha$.
It is primarily a simulation budget rather than a stabilization hyperparameter:
increasing $M$ reduces Monte Carlo variability but increases computational cost,
without changing the structure of the algorithm. For this reason, and to ensure
fair comparison across methods, we keep $M$ fixed throughout the experiments.

\paragraph{Summary.}
Overall, the ablation indicates that the method is not highly sensitive to the
precise choice of $\rho$ or $\alpha$. This supports the use of a shared
hyperparameter configuration across benchmarks.

%% file: 2026_NeurIPS_flowce_appendix_alternative.tex
\section{Alternative divergence measures}
\label{app:divergence_measures}

While alternative discrepancies such as Wasserstein distance and kernel Stein
discrepancy (KSD) have been proposed for comparing distributions, their use in
rare-event importance sampling is less direct. In our setting, the optimal
target distribution is
\[
h^{\rm opt}(\x) \propto \mathbf{1}(g(\x)>\gamma)f(\x),
\]
which is not directly sampleable. KL-based objectives naturally lead to
tractable estimators such as
\[
\Exp\!\left[
\mathbf{1}(g(\bfX)>\gamma)\log p_{\bftheta}(\bfX)
\right],
\]
which can be approximated using samples from $f$ or from an importance proposal.

In contrast, Wasserstein-based approaches require solving an optimal transport
problem between $p_{\bftheta}$ and an empirical approximation of
$h^{\rm opt}$, leading to nested optimization procedures such as Sinkhorn
iterations. These can be computationally demanding and statistically unstable in
rare-event regimes, where the number of effective samples from the target region
is small. Similarly, KSD-based methods require access to the score function
$\nabla \log h^{\rm opt}(\x)$, which is not directly available because
$h^{\rm opt}$ contains a hard indicator. Smoothing the indicator is possible, but
introduces an additional temperature parameter and corresponding sensitivity.

For these reasons, KL-based formulations remain the most natural choice in our
setting: they connect directly to cross-entropy and variational objectives,
admit simple Monte Carlo estimators, and are compatible with the constrained
proposal families used in this work.

\section{Alternative rare-event sampling methods}
\label{app:alternative_sampling_methods}

A number of alternative approaches have been proposed for rare-event estimation.
Large-deviation-based methods exploit problem-specific structure to construct
asymptotically optimal proposals, but typically rely on strong assumptions and
are not easily applicable in general black-box settings.

Sequential Monte Carlo (SMC) and related splitting methods construct a sequence
of intermediate distributions or level sets to gradually reach the rare-event
region. These methods can be effective and are natural baselines for rare-event
simulation, but they often require careful design of intermediate levels,
resampling, and rejuvenation steps. Neural SMC variants further increase
proposal flexibility, at the cost of additional modeling and optimization
complexity.

Adjoint-based sampling methods provide another route to rare-event estimation by
using gradient or sensitivity information to guide sampling toward important
regions. Such methods can be powerful when accurate derivatives and suitable
problem structure are available, but they are less directly applicable in
general black-box settings and may require substantially more algorithmic
machinery.

In contrast, our approach focuses on a lightweight importance-sampling framework
with explicit variance control and minimal structural assumptions. By combining
coverage-first initialization with constrained KL refinement, we obtain a robust
and computationally efficient pipeline applicable across a wide range of
rare-event problems.

\section{CFI vs gradient ascent: detailed comparison}
\label{sec:grad_appendix}
A natural question is whether the coverage stage of CFI can be replaced by gradient-based optimization of $g(\x)$, for example via multiple gradient ascent restarts. We show that this approach suffers from two qualitatively different failure modes:
\begin{enumerate}[label=(\roman*)]
\item \textbf{Local failure in multimodal functions}, where some modes are systematically missed;
\item \textbf{Global failure for unbounded (``runaway'') functions}, where trajectories diverge to infinity.
\end{enumerate}
\textbf{(i) Multimodal functions.} 
We consider a two-dimensional function with three modes located at
$\bfmu_1 = (4,0)$, $\bfmu_2 = (0,4)$, and $\bfmu_3 = (3.3, 3.3)$.
The function is defined as
\[
g(\x) = \max\{h_1(\x), h_2(\x), h_3(\x)\},
\]
where
\[
h_1(\x) = \exp\left(-\frac{(x_1-4)^2 + x_2^2}{8}\right), \quad
h_2(\x) = \exp\left(-\frac{x_1^2 + (x_2-4)^2}{8}\right),
\]
and
\[
h_3(\x) = 2.8 \exp\left(-\frac{(x_1-3.3)^2 + (x_2-3.3)^2}{1.6}\right)
        - 0.9 \exp\left(-\frac{x_1^2}{0.7}\right)
        - 0.9 \exp\left(-\frac{x_2^2}{0.7}\right).
\]
This construction ensures that, in a neighborhood of the origin, the gradient points toward either $\bfmu_1$ or $\bfmu_2$, making $\bfmu_3$ difficult to reach.
Figure~\ref{fig:multimodal_grad} illustrates the geometry of this example. The gradient field directs trajectories toward $\bfmu_1$ and $\bfmu_2$, while $\bfmu_3$ has a very small basin of attraction.
\begin{figure}[t]
\centering
\begin{minipage}{0.48\linewidth}
\centering
\includegraphics[width=\linewidth]{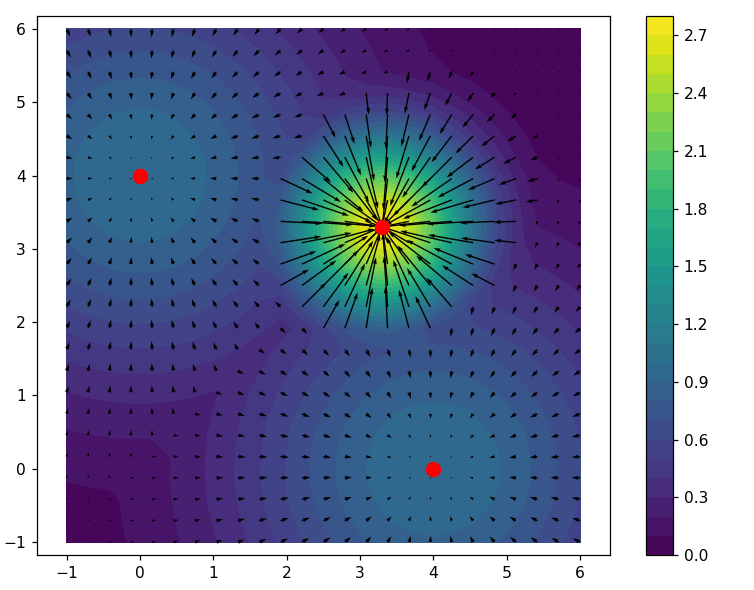}
\end{minipage}
\hfill
\begin{minipage}{0.48\linewidth}
\centering
\includegraphics[width=\linewidth]{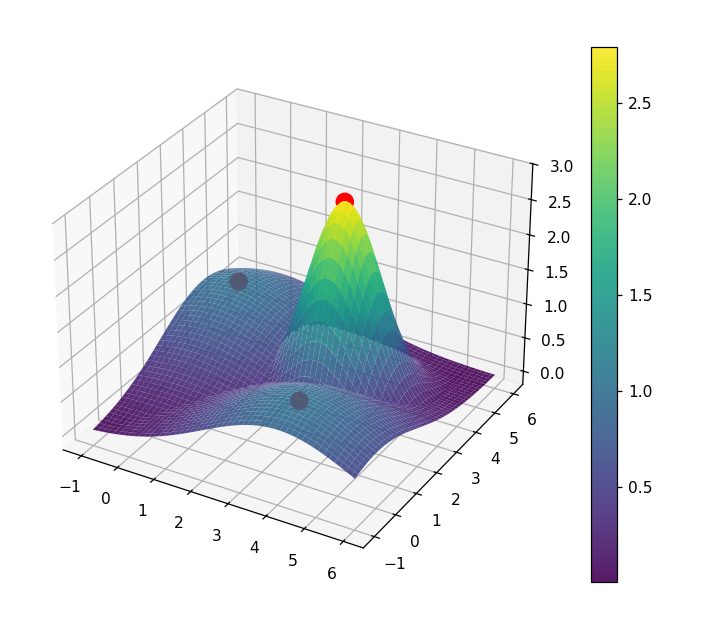}
\end{minipage}
\caption{
Multimodal example: contour plot with gradient field (left) and surface plot (right).
Gradient ascent trajectories initialized near the origin are attracted to $\bfmu_1$ and $\bfmu_2$, while $\bfmu_3$ has a very small basin of attraction.
}
\label{fig:multimodal_grad}
\end{figure}
We perform 1000 gradient ascent runs initialized from $\bfX \sim \mathcal{N}(\mathbf{0}, \mathbf{I})$ and assign each final point to the nearest mode. We compare this with CFI-c using Gaussian mixture proposals. For gradient ascent, the entries below count final points among the 1000 restarts; for CFI-c, they count fitted mixture components assigned to each mode.
\begin{center}
\begin{tabular}{lccc}
\toprule
Method & $\bfmu_1$ & $\bfmu_2$ & $\bfmu_3$ \\
\midrule
Gradient ascent (1000 restarts) & 499 & 501 & 0 \\
CFI-c (GMM6)  & 1 & 1 & 4 \\
CFI-c (GMM16) & 4 & 3 & 9 \\
CFI-c (GMM18) & 8 & 5 & 5 \\
\bottomrule
\end{tabular}
\end{center}
Gradient ascent never discovers $\bfmu_3$, while CFI-c identifies all three modes in each reported configuration. This highlights a fundamental limitation of gradient ascent: even with multiple restarts, it produces a collection of local optima rather than a proposal distribution covering the relevant regions.

\textbf{(ii) Unbounded (``runaway'') functions (global failure).}
We now consider the function
\[
g(x_1,x_2)=\max(x_1-4,\,-x_1x_2-4,\;x_1x_2-4).
\]
\begin{wrapfigure}{l}{0.35\linewidth}
\centering
\includegraphics[width=\linewidth]{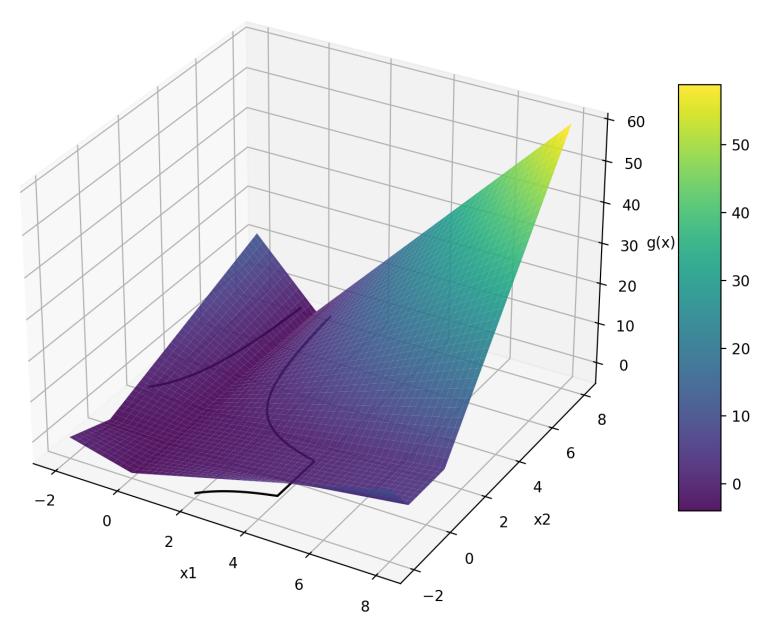}
\caption{
Runaway example: the function grows unbounded in certain directions, causing gradient ascent trajectories to diverge instead of approaching relevant level sets.
}
\label{fig:runaway_surface}
\vspace{-7ex}
\end{wrapfigure}
This function grows unbounded in multiple directions, making direct maximization ill-posed. Gradient ascent trajectories initialized from $\bfX \sim \mathcal{N}(\mathbf{0}, \mathbf{I})$ diverge rapidly, with $\lVert\x_t\rVert \to \infty$, rather than converging to regions near the level set $g(\x) \approx \gamma$.

Figure~\ref{fig:runaway_surface} illustrates the geometry of this example. The absence of a well-defined maximizer prevents gradient ascent from producing a meaningful proposal distribution or notion of coverage.

In contrast, CFI-c targets successive level sets rather than a maximizer, so unboundedness of $g$ alone does not make its objective ill-posed.

\textbf{Summary.}
These examples show that gradient ascent may fail either locally (missing modes) or globally (diverging), whereas CFI-c provides a stable distributional mechanism for covering the rare-event region.

%% file: 2026_NeurIPS_flowce_appendix_safeice.tex
\section{Comparison with Safe-ICE}
\label{app:safeice}

We additionally compare our method with Safe-ICE, a recent safe
cross-entropy-based importance sampling method for rare-event simulation
\cite{gao2025safeice}. We used the publicly available implementation from
\cite{safeice_github}. Since this implementation was computationally expensive
in our setting, with some low-dimensional runs already requiring several hours
for $K_0=16$ initial components, we restrict the comparison to representative
benchmarks rather than running Safe-ICE across the full experimental suite.

The comparison is motivated by the close connection between Safe-ICE and
cross-entropy-based adaptive importance sampling, as well as by related
dimension-reduction-based CE approaches for rare-event simulation
\cite{uribe2021cross}. We consider three cases: a low-dimensional benchmark
($g_2$, $d=2$), a moderate-dimensional version of $g_5$ with $d=4$, and the
higher-dimensional $g_5$ benchmark with $d=40$.

For the two-dimensional benchmark $g_2$, both methods produce accurate estimates.
The true value is $2.226\cdot 10^{-3}$; Safe-ICE gives
$2.18\cdot 10^{-3}$, while our method gives $2.225\cdot 10^{-3}$, with lower
computational cost in our implementation. Thus, in low dimension, both methods
are accurate.

For the moderate-dimensional case $g_5$ with $d=4$, Safe-ICE also gives a
reasonable estimate, $2.78\cdot 10^{-5}$, compared with the reference value
$3.15\cdot 10^{-5}$. This indicates that Safe-ICE remains effective for
moderately complex settings, although the error is slightly larger than for our
method.

The behavior differs sharply in the higher-dimensional benchmark $g_5$ with
$d=40$. In this case, our method remains stable and accurate, giving
$3.40\cdot 10^{-5}$ compared with the reference value $3.15\cdot 10^{-5}$.
In contrast, Safe-ICE collapses and returns a degenerate estimate,
$9.14\cdot 10^{-24}$. This behavior is consistent with the high-dimensional
concentration and covariance-collapse effects discussed in the main text:
cross-entropy-based updates can become unstable when the learned proposal
concentrates too strongly around a small part of the rare-event region.

\paragraph{Interpretation.}
These results suggest that Safe-ICE is effective in low-dimensional or
moderately complex settings, but its stability may degrade in higher dimensions.
By contrast, our constrained-GMM variants combine
coverage-first initialization with covariance constraints that provide the
finite-variance guarantee of Theorem~\ref{thm:gmm_finite_var}; the robustness of
the flow-refined variants is an empirical finding. From a methodological perspective, the difference is
structural: Safe-ICE follows a CE/ICE-style proposal adaptation based on
cross-entropy updates, whereas our constrained CFI-based pipeline explicitly
separates coverage from fitting and enforces covariance floors that prevent the
variance-collapse mechanism identified in Proposition~\ref{propp:kl_collapse}
and Theorem~\ref{thm:gmm_finite_var}.

%% file: 2026_NeurIPS_flowce_appendix_benchmark.tex
\section{Benchmark function definitions}
\label{app:benchmark_settings}

Here we  provide  explicit definitions of the benchmark functions used
in the experimental evaluation.
All benchmarks are standard in the rare--event simulation literature and were
adopted from previous studies, as summarized in the main text.

For $\bfx=(x_1,\ldots,x_d)^\top$ and $\mathbf{c}=(3.8,\ldots,3.8)^\top$, the
functions are defined as follows:
\[
\begin{aligned}
g_1(\bfx)
&= 1 - \min\!\left\{\|\bfx+\mathbf{c}\|^2,\ \|\bfx-\mathbf{c}\|^2\right\},
\\[2pt]
g_2(\bfx)
&= \min\Big\{
0.1(x_1-x_2)^2 - \tfrac{x_1+x_2}{\sqrt{2}} + 3,\;
0.1(x_1-x_2)^2 + \tfrac{x_1+x_2}{\sqrt{2}} + 3, x_1 - x_2 + \tfrac{7}{\sqrt{2}},\;
x_2 - x_1 + \tfrac{7}{\sqrt{2}}
\Big\},
\\[4pt]
g_3(\bfx)
&= -\!\left(
\sum_{i\ \mathrm{odd}}^{d-1} 0.1\,(x_i-x_{i+1})^2
- \tfrac{1}{\sqrt{d}}\sum_{i=1}^d x_i + 2.5
\right),
\\[6pt]
g_5(\bfx)
&= \sum_{i=1}^{m}
\Big[
(x_{4i-3}+10x_{4i-2})^2
+5(x_{4i-1}-x_{4i})^2
+(x_{4i-2}-2x_{4i-1})^4
+10(x_{4i-3}-x_{4i})^4
\Big],
\\[6pt]
g_6(\bfx)
&= -\frac{1}{\sqrt{d}}\sum_{i=1}^d x_i,
\quad
g_7(\bfx)
= \mathrm{ResNet}(\bfx),
\quad
g_{10}(\bfx)
= \sum_{i=1}^{10} \exp(x_i).
\end{aligned}
\]

The functions $g_1$, $g_2$, and $g_3$ are used in the two--dimensional setting
($d=2$) and represent, respectively, a two--mode configuration (``Leaf''),
a radial benchmark with piecewise boundary, and an anisotropic problem with a
linear tilt.
Functions $g_5$ and $g_6$ correspond to higher--dimensional test problems commonly
used in rare--event estimation, while $g_{10}$ models the tail probability of a
sum of correlated lognormal random variables.

\paragraph{ResNet-based benchmark ($g_7$).}

We adopt the ResNet-based rare-event formulation introduced in~\cite{gao2023rare}, which studies rare but significant performance degradation of deep neural networks under stochastic parameter perturbations. Let $\bfw= \{w_i\}_{i=1}^{d}$ denote the parameters of a trained torchvision ResNet-18 model with \texttt{IMAGENET1K\_V1} weights~\cite{torchvision_resnet18_weights}, grouped in a layerwise manner according to the network architecture, where $d=62$
corresponds to the number of layerwise parameter groups.
A perturbation vector $\bfX \in \mathbb{R}^{d}$ is drawn from a standard Gaussian distribution
(i.e., $\bfX \sim \calN(\bf{0}, \bf{I})$) and induces multiplicative noise applied independently to each parameter group.

In practice, the perturbed parameters are constructed as
\begin{center}
\begin{verbatim}
for i, param in enumerate(model.parameters()):
    param.data = param.data * (1 + x[i] * noise_std)
\end{verbatim}
\end{center}
which corresponds to the   componentwise definition $w_i(\bfx) = w_i \bigl( 1 + \kappa x_i \bigr), i = 1, \dots, d,$
where $\kappa = 0.2$ controls the perturbation magnitude.
Let $\bfw$ denote  the nominal (unperturbed) network parameters, while $\bfw(\bfx)$ denotes their multiplicatively perturbed version corresponding to the realization~$\bfx$ (thus  $\bfw\equiv\bfw(\mathbf{0})$).
Let $L(\bfw)$ denote the average cross-entropy loss evaluated on a fixed validation set of 24 images from
the \texttt{imagenetv2-matched-frequency-format} ImageNetV2 dataset~\cite{recht2019imagenetv2}. The accepted $g_7$ evaluation uses the fixed prepared
\texttt{imagenetv2-matched-frequency-format-val\_first24.pt} tensor and
torchvision ResNet-18 \texttt{IMAGENET1K\_V1} weights. They are external,
user-supplied local inputs and are neither bundled nor hosted by the release.
The evaluator checks the documented fixed filename, size, and SHA-256
identities and runs offline without download or cache fallback. The legacy
preparation script does not exactly reproduce the accepted tensor. The function is defined as
\begin{equation}
g_7(\bfx) = L(\bfw(\bfx)) - L(\bfw).
\label{eqn:g7_function}
\end{equation}
Rare events correspond to $g_7(\bfx) > \gamma$, where $\gamma = 5.15715$ is chosen via extensive CMC calculations such that $\mathbb{P}(g_7(\bfX) > \gamma) \approx 6 \times 10^{-5}$, matching the reference value reported in~\cite{gao2023rare}.

Evaluating $g_7(\bfx)$ is computationally expensive, as each function call requires inference with a perturbed network. To enable efficient evaluation while preserving the structure of the rare-event problem, a first-order approximation of the loss around the nominal parameters $\bfw$ is employed. A first-order Taylor expansion of (\ref{eqn:g7_function}) yields
\[
g_7(\bfx) \approx  \kappa \sum_{i=1}^{d} x_i \left\langle \nabla_{w_i} L(\bfw), w_i \right\rangle,
\]
where $w_i$ denotes the $i$-th layerwise parameter group  and
$\langle \cdot, \cdot \rangle$ denotes the Euclidean inner product. This linearized model enables rapid evaluation of $g_7(\bfx)$ for a large number of perturbation samples and facilitates the application of importance-sampling-based rare-event estimation methods.

\paragraph{Summary of functions.}
References and experimental settings (thresholds $\gamma$, dimensions, and input
distributions) follow the original works cited in the main text,
they are summarized in Table~\ref{tab:benchmark_summary}.
Input distribution  $A$ denotes $\calN(\bfmu^{(A)},\bfSigma^{(A)})$ and
$B$ denotes $\calN(\bfmu^{g_{10}},\bfSigma^{g_{10}})$.

\begin{table}[t]
\centering
\caption{Summary of benchmark settings used in the experiments.}
\label{tab:benchmark_summary}
\begin{tabular}{lcccc}
\toprule
Benchmark & Dimension $d$ & Input & Threshold $\gamma$
          & Reference $\Prob(g(\bfX)>\gamma)$ \\
\midrule
$g_1$    & 2   & A & 0         & 4.7934e-6  \\
$g_2$    & 2   & A & 0         & 2.22668e-3 \\
$g_3$    & 2   & A & 0         & 4.2073e-3  \\
$g_5$    & 40  & A & -403.6636 & 3.1500e-5  \\
$g_6$    & 100 & A & 3.5       & 2.3262e-4  \\
$g_7$    & 62  & A & 5.15715   & 6.00e-5    \\
\midrule
$g_{10}$ & 10  & B & 20k       & 1.03526e-3 \\
$g_{10}$ & 10  & B & 40k       & 4.683e-4   \\
$g_{10}$ & 10  & B & 500k      & 1.77400e-5 \\
\bottomrule
\end{tabular}
\end{table}

\section{Normalizing flow architecture (RealNVP)}
\label{app:realnvp_architecture}

The normalizing flow used in the final refinement stage follows the RealNVP
formulation introduced by~\cite{DinhSB17}.
The flow defines an invertible
transformation
$T_{\bftheta}:\mathbb{R}^d\to\mathbb{R}^d$, which induces the density
\[
p^{\rm flow}_{\bftheta}(\bfX)
=
p^{\rm KL}\!\big(T_{\bftheta}(\bfX)\big)
\left|\det J_{\bftheta}(\bfX)\right|,
\]
where $p^{\rm KL}$ denotes the base distribution obtained after KL fitting
(i.e., the constrained GMM), and $J_{\bftheta}$ is the Jacobian of
$T_{\bftheta}$.
The transformation $T_{\bftheta}$ is constructed as a composition of $K$ affine
coupling layers,
\[
T_{\bftheta} = T_K \circ \cdots \circ T_1.
\]
In each coupling layer, the input vector is partitioned into two subvectors
$(\mathbf{x}_a,\mathbf{x}_b)$, and the transformation is defined as
\[
  \begin{aligned}
  \mathbf{x}_a &\mapsto \mathbf{x}_a, \\
  \mathbf{x}_b &\mapsto \mathbf{x}_b \odot \exp\left( s(\mathbf{x}_a) \right) + t(\mathbf{x}_a),
  \end{aligned}
  \]
where $s(\cdot)$ and $t(\cdot)$ are neural networks.
The partition alternates across successive layers, ensuring that all input
dimensions are transformed and that the overall mapping is sufficiently
expressive.
This structure yields a triangular Jacobian, making the determinant
$\det J_{\bftheta}(\bfX)$ computationally tractable.

Throughout training, the base distribution $p^{\rm KL}$ remains fixed, and the
flow acts purely as a refinement layer on top of the KL-fitted proposal.

%% file: 2026_NeurIPS_flowce_appendix_other.tex
\begin{figure*}[t!]
\centering
\begin{tabular}{ccc}
\includegraphics[width=120pt]{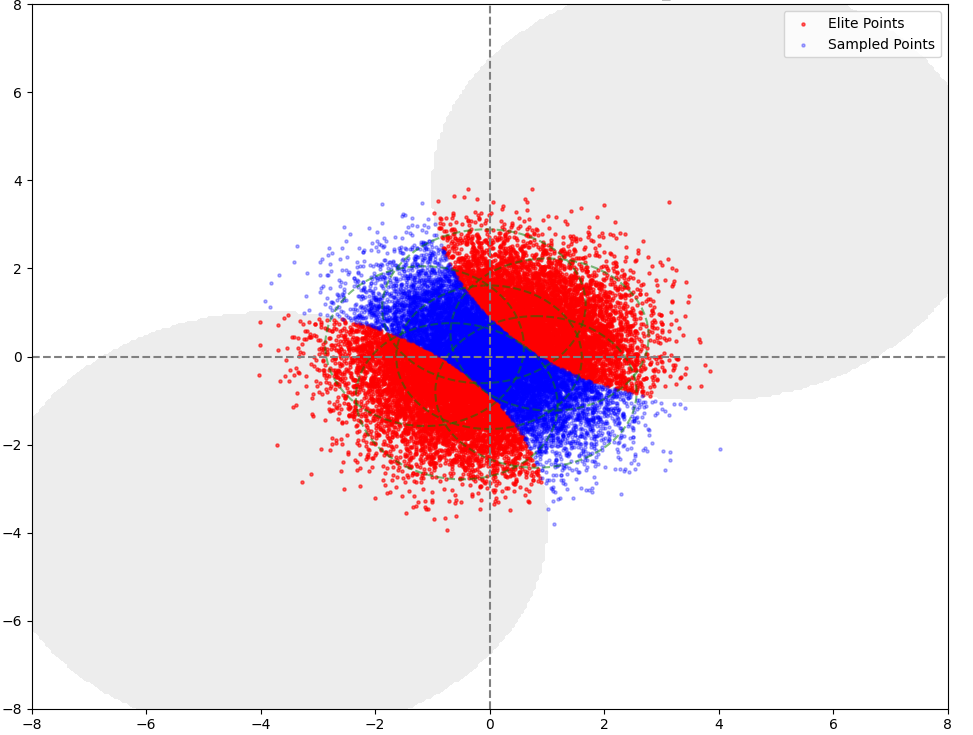} &
\includegraphics[width=120pt]{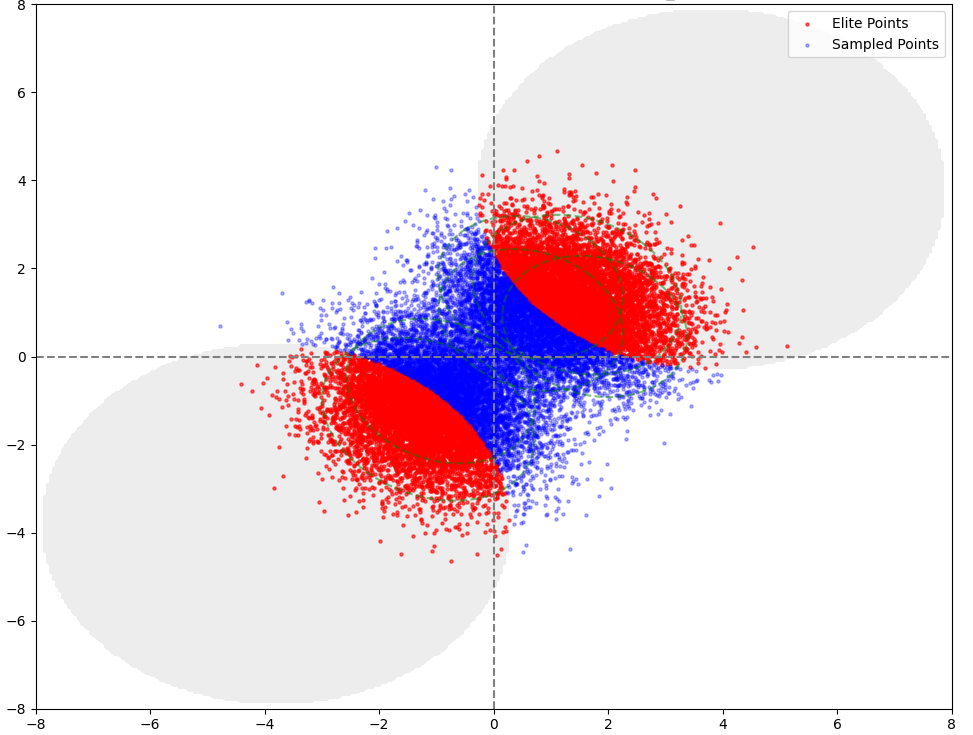} &
\includegraphics[width=120pt]{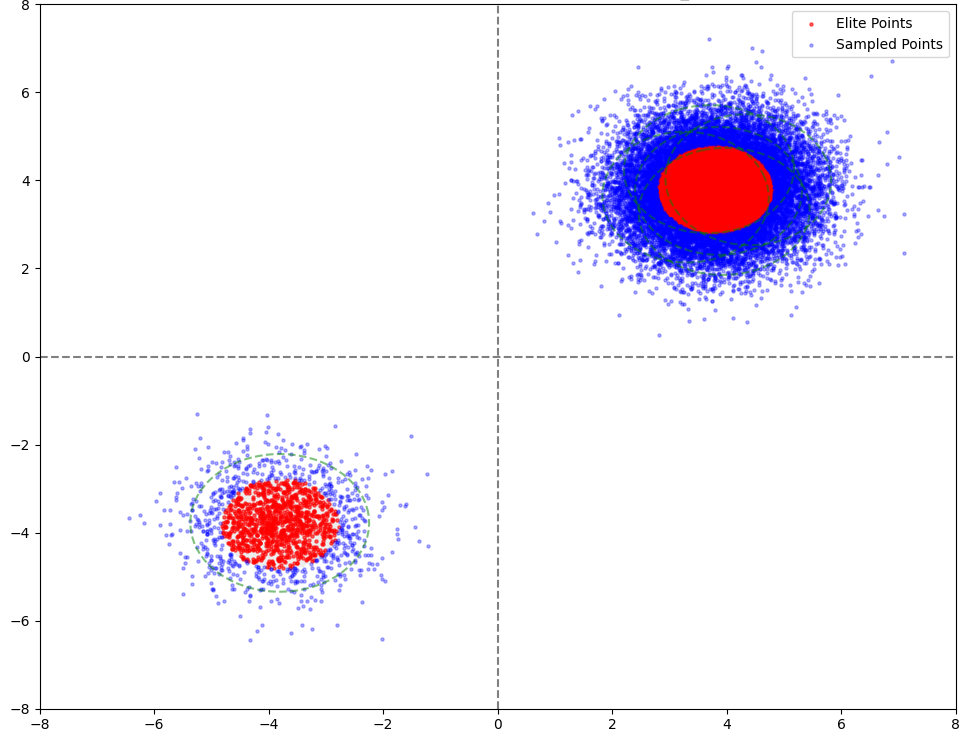} \\
iter=0, $\gamma_0=-22.15$ &
iter=2, $\gamma_2=-15.71$ &
iter=97, $\gamma_{97}=-0.06$ \\[5pt]
\includegraphics[width=120pt]{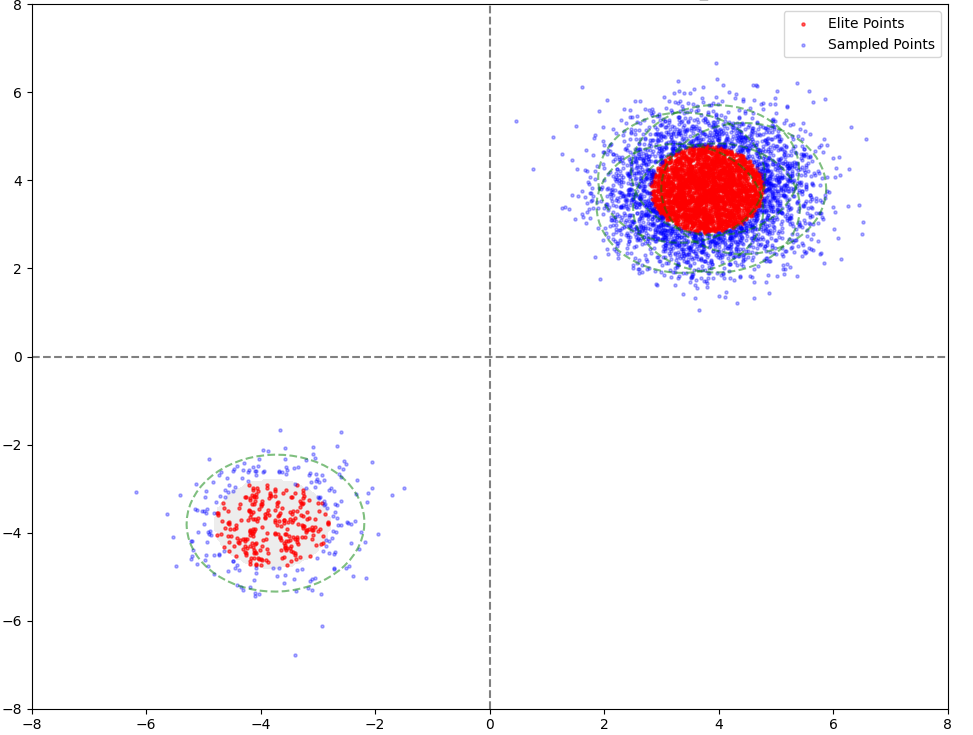} &
\includegraphics[width=120pt]{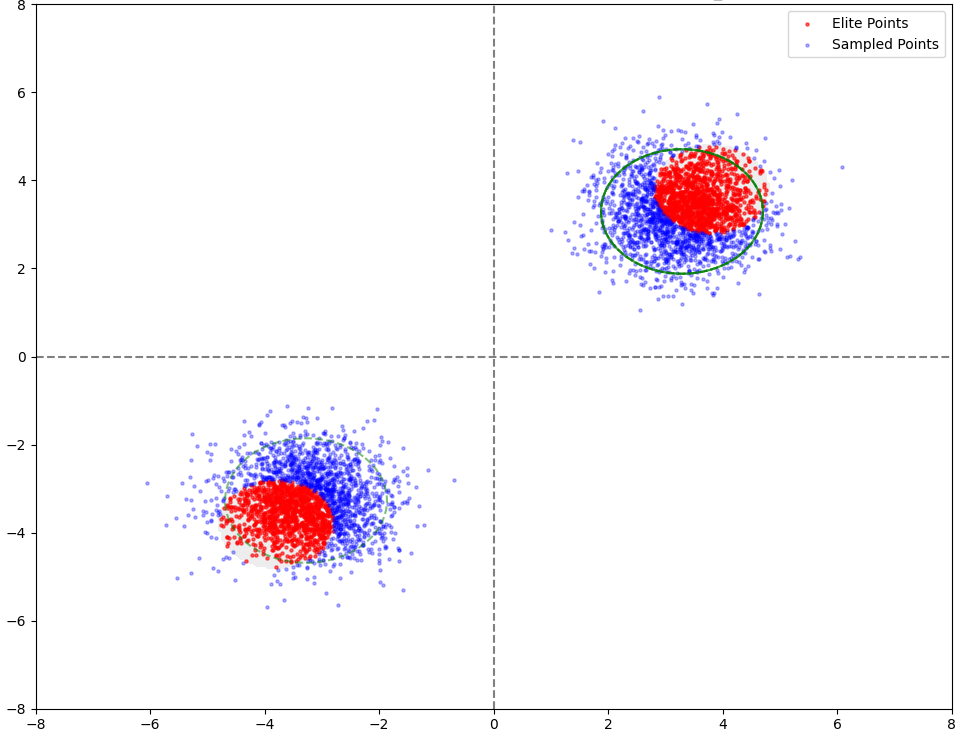} &
\includegraphics[width=120pt]{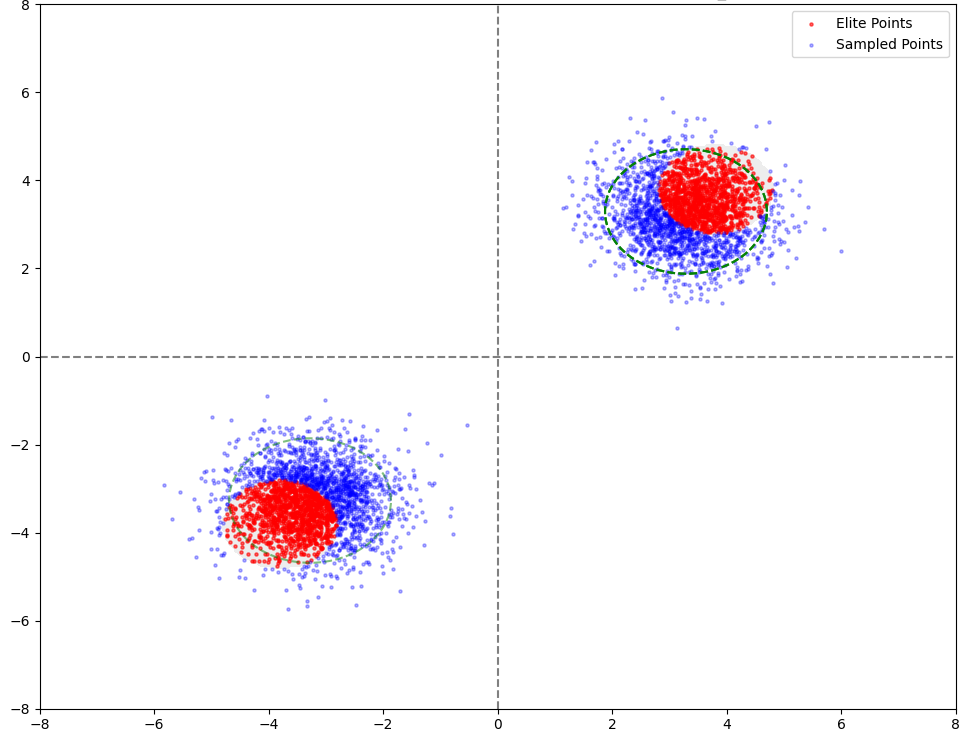} \\
iter=0, $\gamma_0=0.0$ &
iter=200, $\gamma_{200}=0.0$ &
iter=400, $\gamma_{400}=0.0$ \\[5pt]
\includegraphics[width=120pt]{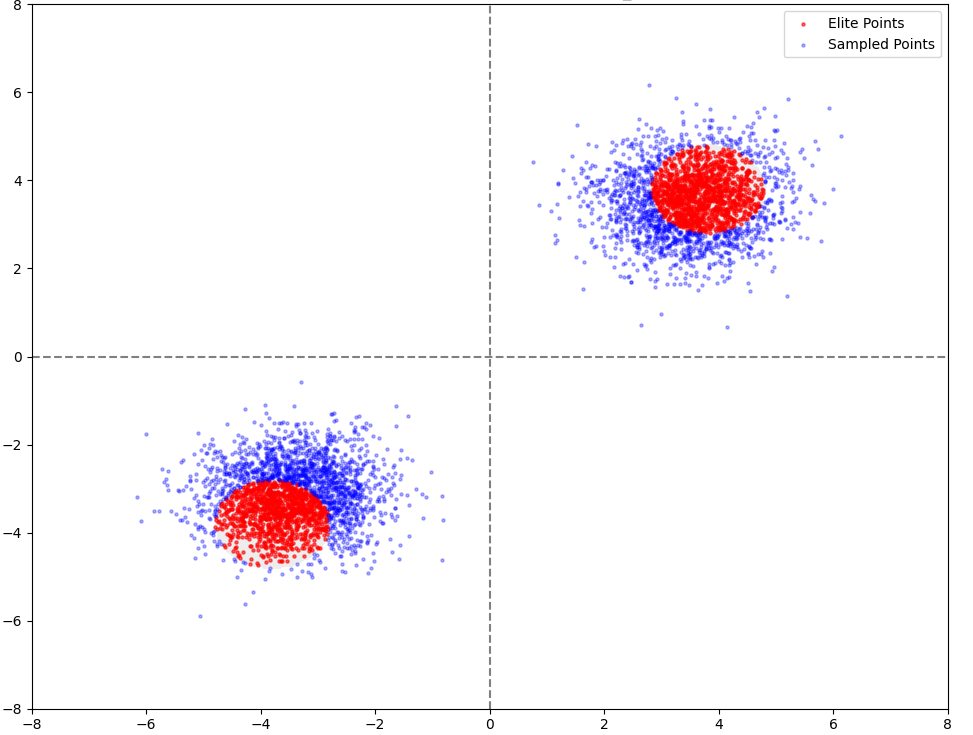} &
\includegraphics[width=120pt]{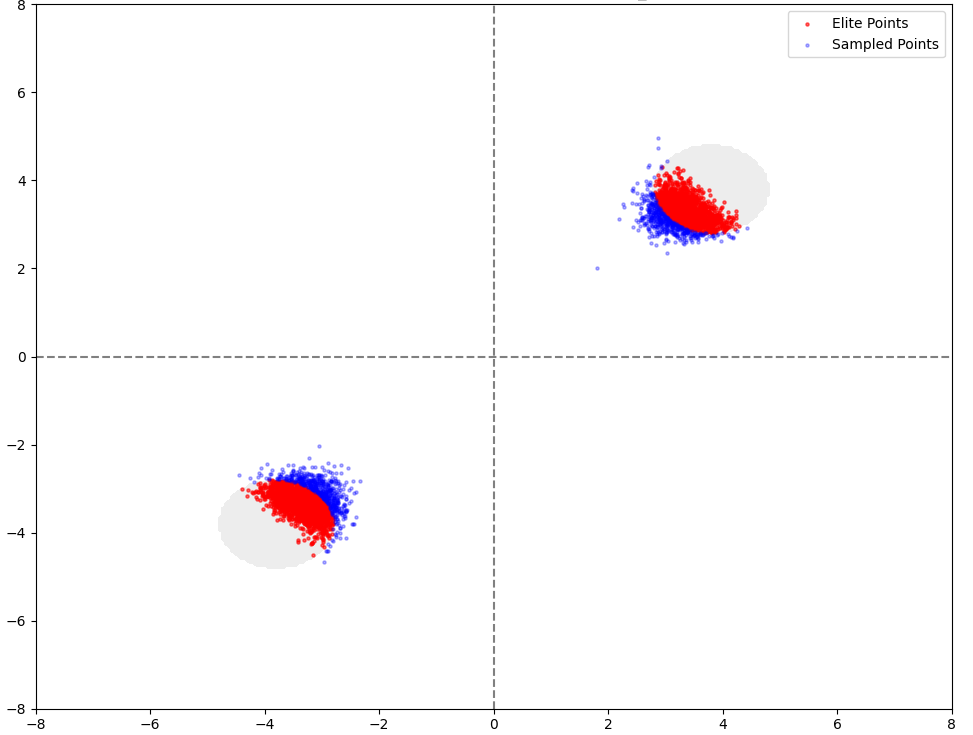} &
\includegraphics[width=120pt]{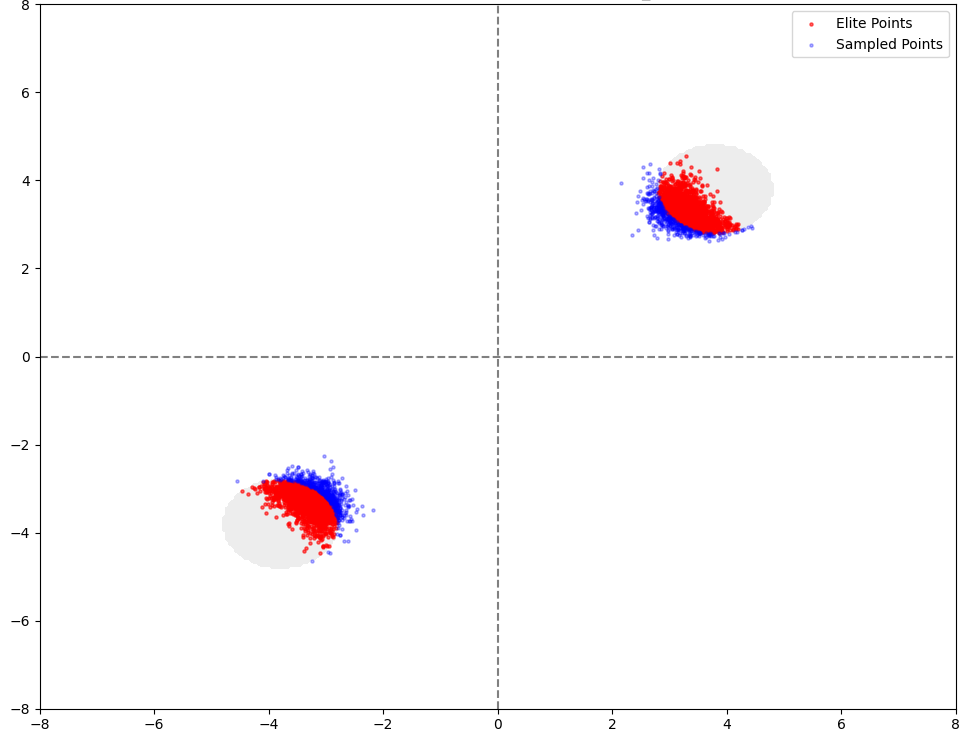} \\
iter=0, $\gamma_0=0.0$ &
iter=200, $\gamma_{200}=0.0$ &
iter=400, $\gamma_{400}=0.0$
\end{tabular}
\caption{
Evolution of the proposal distribution during training for estimating
$\Prob(g_1(\bfX)>0)$ with $\bfX\sim\calN(\bfmu^{(A)},\bfSigma^{g_{10}})$ and $K=6$.
Top to bottom: CFI-c, FKL-c initialized from CFI-c, and the subsequent
flow-based refinement (+Flow).
}
\label{fig:g1}
\end{figure*}
\begin{figure}[t!]
\centering
\begin{tabular}{ccc}
\includegraphics[width=120pt]{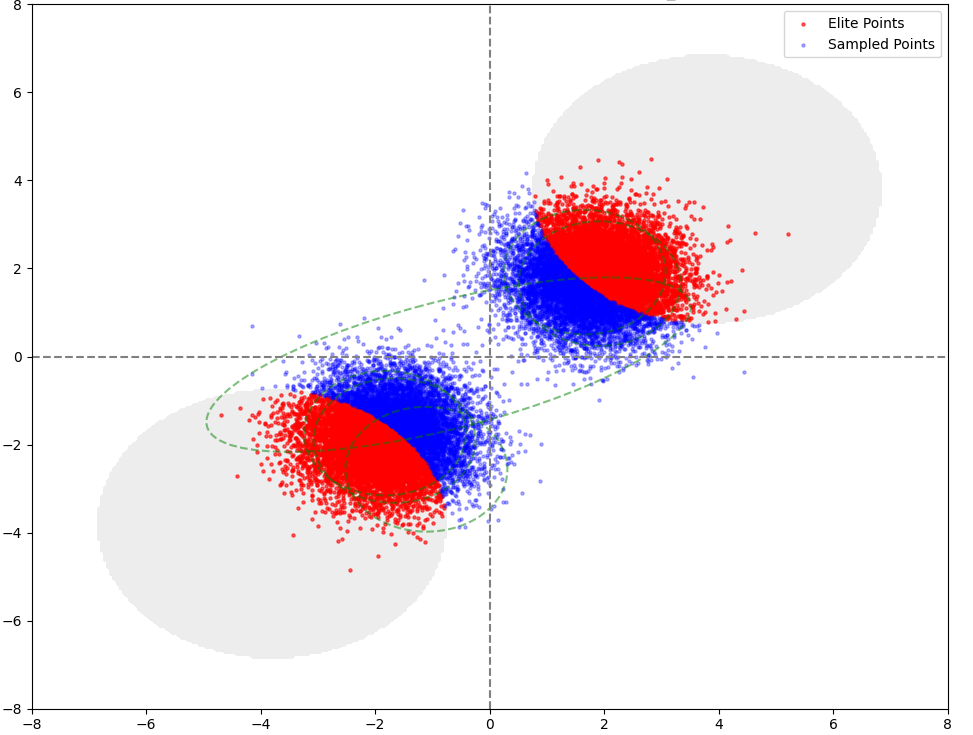} &
\includegraphics[width=120pt]{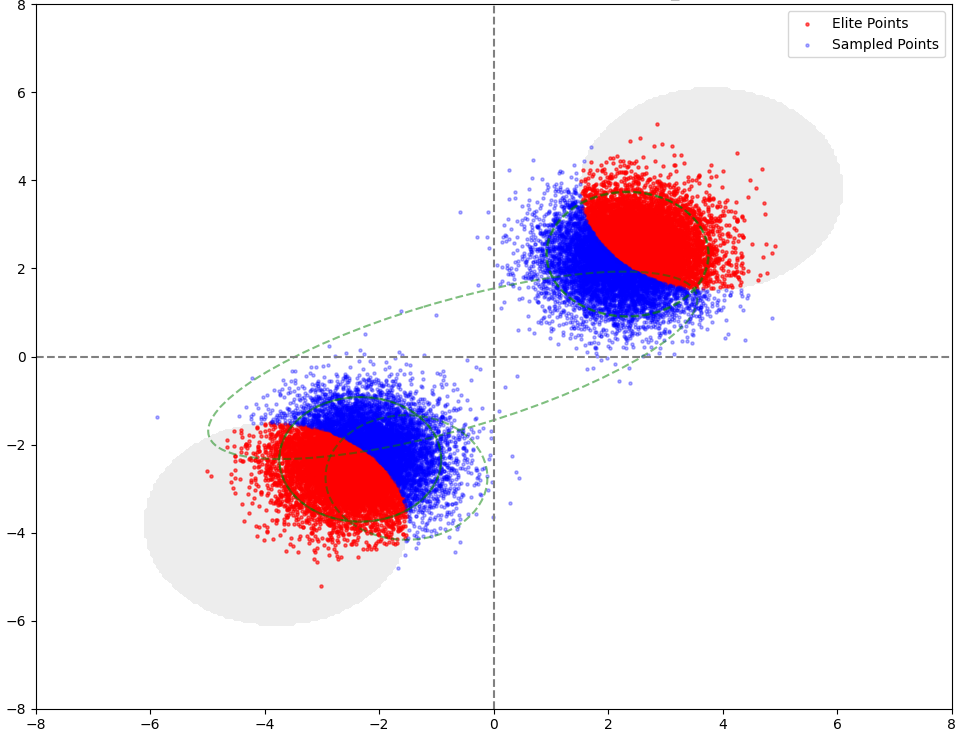} &
\includegraphics[width=120pt]{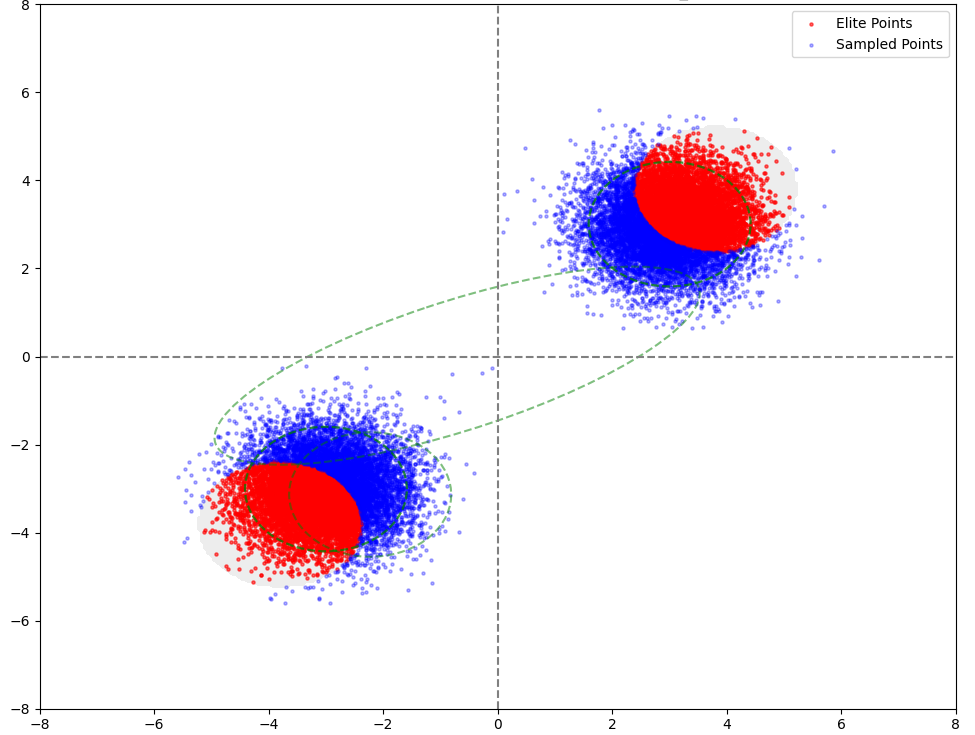} \\
iter=5, $\gamma_5=-8.26$ &
iter=9, $\gamma_9=-4.26$ &
iter=29, $\gamma_{29}=-1.03$
\end{tabular}
\caption{
Selected snapshots from fFKL-c training for estimating
$\Prob(g_1(\bfX)>0)$ with $\bfX\sim\calN(\mathbf{0},I_2)$.
}
\label{fig:g1-face}
\end{figure}

\section{Additional visualizations for $g_1$}
\label{app:graph_g1}

This section provides additional graphical illustrations for the task of
estimating $\Prob(g_1(\bfX)>0)$ with $\bfX\sim\calN(\bfmu^{(A)},\bfSigma^{g_{10}})$
in the two--dimensional setting.
The visualizations complement the numerical results reported in the main body
and illustrate the qualitative behavior of the proposed pipeline across
iterations, model classes, and refinement strategies.

Figure~\ref{fig:g1} shows representative snapshots of the learned proposal
distributions for $K=6$ mixture components.
From top to bottom, the rows correspond to:
(i) the constrained coverage-first initialization (CFI-c),
(ii) constrained forward-KL refinement (FKL-c) initialized from CFI-c, and
(iii) the subsequent flow-based refinement (+Flow) built on top of FKL-c.
Each row displays selected iterations or epochs from the corresponding training
procedure.

\paragraph{Comparison with fully adaptive CE.}
For comparison, Fig.~\ref{fig:g1-face} shows selected snapshots from the fully
adaptive constrained CE method (fFKL-c), in which the proposal is re-optimized
at each intermediate level using importance-weighted likelihood.
On $g_6$ and $g_7$, all fFKL-c evaluations produce positive finite estimates.
The method nevertheless typically requires more iterations and remains
sensitive to hyperparameter choices and finite-sample variability; the $K=16$
estimates in particular vary substantially across evaluation seeds. Thus,
a finite numerical output should not be interpreted as uniformly accurate or
low-variance performance. These observations motivate the two-stage
pipeline based on CFI-c followed by constrained KL refinement.